\documentclass{article}

\PassOptionsToPackage{semicolon}{natbib}
\usepackage[final]{neurips_2022}

\usepackage[utf8]{inputenc} % allow utf-8 input
\usepackage[T1]{fontenc}    % use 8-bit T1 fonts

\usepackage{titletoc}
\usepackage[page, header]{appendix} % MAKE SURE THIS IS LOADED BEFORE hyperref PACKAGE!

\usepackage[colorlinks=true, linkcolor=blue, citecolor=blue,urlcolor=black]{hyperref}       % hyperlinks
\usepackage{url}            % simple URL typesetting
\usepackage{booktabs}       % professional-quality tables
\usepackage{amsfonts}       % blackboard math symbols
\usepackage{nicefrac}       % compact symbols for 1/2, etc.
\usepackage{microtype}      % microtypography
\usepackage{xcolor}         % colors
\usepackage{makecell}

\title{Follow-the-Perturbed-Leader for Adversarial Markov Decision Processes with Bandit Feedback}

\author{%
  Yan Dai \\
  Tsinghua University \\
  \texttt{yan-dai20@mails.tsinghua.edu.cn} \\
   \AND
   Haipeng Luo \\
   University of Southern California \\
   \texttt{haipengl@usc.edu} \\
   \And
   Liyu Chen \\
   University of Southern California \\
   \texttt{liyuc@usc.edu} \\
}

\usepackage{amsmath}
\usepackage{amsthm}
\usepackage{amssymb}
\usepackage{mathtools}
\allowdisplaybreaks

\usepackage{algorithm}
\usepackage[noend]{algpseudocode}
\makeatletter
\newcounter{HALG@line}
\renewcommand{\theHALG@line}{\thealgorithm.\arabic{ALG@line}}
\makeatother

\usepackage{cleveref}
\usepackage{footnote}
\usepackage{multirow}
\usepackage{enumitem}

\usepackage{pifont}
\newcommand{\cmark}{\ding{51}}%
\newcommand{\xmark}{\ding{55}}%

\newtheorem{theorem}{Theorem}
\newtheorem{lemma}[theorem]{Lemma}

\newtheorem{corollary}[theorem]{Corollary}
\newtheorem{definition}[theorem]{Definition}
\newtheorem{remark}[theorem]{Remark}
\newtheorem{assumption}[theorem]{Assumption}

\newcommand{\E}{\operatornamewithlimits{\mathbb{E}}}
\newcommand{\D}{~\mathrm{d}}

\newcommand{\argmin}{\operatornamewithlimits{\mathrm{argmin}}}
\renewcommand{\O}{\operatorname{\mathcal O}}
\newcommand{\Otil}{\operatorname{\tilde{\mathcal O}}}
\newcommand{\mS}{\mathcal{S}}
\newcommand{\mA}{\mathcal{A}}
\newcommand{\delay}{\mathfrak{D}}

\Crefformat{equation}{Eq.~(#2#1#3)}
\renewcommand{\tilde}{\widetilde}
\renewcommand{\hat}{\widehat}
\renewcommand{\bar}{\overline}
\usepackage{comment}
\usepackage{bbm}

\makeatletter
\newcommand\footnoteref[1]{\protected@xdef\@thefnmark{\ref{#1}}\@footnotemark}
\makeatother

\begin{document}
%\normalsize
%\setlength{\textfloatsep}{0.1cm}% Remove \textfloatsep

\maketitle

\begin{abstract}
We consider regret minimization for Adversarial Markov Decision Processes (AMDPs), where the loss functions are changing over time and adversarially chosen, and the learner only observes the losses for the visited state-action pairs (i.e., bandit feedback).
While there has been a surge of studies on this problem using Online-Mirror-Descent (OMD) methods, very little is known about the Follow-the-Perturbed-Leader (FTPL) methods, which are usually computationally more efficient and also easier to implement since it only requires solving an offline planning problem.
Motivated by this, we take a closer look at FTPL for learning AMDPs, starting from the standard episodic finite-horizon setting.
We find some unique and intriguing difficulties in the analysis and propose a workaround to eventually show that FTPL is also able to achieve near-optimal regret bounds in this case.
More importantly, we then find two significant applications:
First, the analysis of FTPL turns out to be readily generalizable to delayed bandit feedback with order-optimal regret, while OMD methods exhibit extra difficulties~\citep{jin2022near}. 
Second, using FTPL, we also develop the first no-regret algorithm for learning communicating AMDPs in the infinite-horizon setting with bandit feedback and stochastic transitions.
Our algorithm is efficient assuming access to an offline planning oracle, while even for the easier full-information setting, the only existing algorithm~\citep{chandrasekaran2021online} is computationally inefficient. 
%    For AMDPs, there are mainly two types of algorithms: Follow-the-Perturbed-Leader (FTPL) based and Online-Mirror-Descent (OMD) based, while the former one is more efficient and easier to implement, it is harder to analyze.
%    In the full-information, episodic settings, FTPL-based algorithms have recently been shown to have comparable regret guarantees to those OMD-based algorithms \citep{wang2020refined}. However, no $\mathcal O(\sqrt K)$ regret bounds has been achieved by FTPL in bandit feedback yet.
%    In this paper, we design several FTPL-based algorithms for AMDPs with bandit feedback and yield near-optimal regret guarantees.
%    Specifically, in episodic settings, our algorithm is capable of $\tilde{\mathcal O}(H^{\nicefrac 32}\sqrt{SAK})$ regret with known transition and $\tilde{\mathcal O}(H^2S\sqrt{AK})$ with an unknown transition. Compared to the OMD-based state-of-the-arts \citep{zimin2013online,jin2020learning}, they are sub-optimal only in terms of $H$, while enjoying the easy-to-implement feature.
%    We then considered the delayed-feedback AMDPs, yielding $\tilde{\mathcal O}(H^2S\sqrt{AK}+H^{\nicefrac 32}\sqrt{SAD})$ regret, which over-performs the \textsc{Delayed UOB-FTRL} algorithm \citep{jin2022near}.
%    At last, we designed an \textit{oracle-efficient} algorithm for infinite-horizon, communicating AMDPs, achieving $\tilde{\mathcal O}(T^{5/6})$ regret, which is, as we are aware of, the first ``no-regret'' result for infinite-horizon AMDPs with bandit feedback and only communicating (instead of ergodic) assumptions.
\end{abstract}

\section{Introduction}
%!TEX root=main.tex
Markov Decision Processes (MDPs) have long been used to model problems in reinforcement learning, where the agent takes sequential actions in an environment, leading to transitions among different states and observations on loss (or reward equivalently) signals. 
While the classical MDP model assumes a fixed loss function,
there has been increasing interest in studying regret minimization under non-stationary or even adversarial loss functions via the Adversarial MDP (AMDP) model, starting from the work of~\citet{even2009online}.
%The goal of the agent is still to minimize her cumulative loss in $K$ episodes as the usual MDP formulation.

Similar to other regret minimization problems,
there are typically two categories of algorithms for AMDPs: those based on the Follow-the-Perturbed-Leader (FTPL) framework~\citep{even2009online,neu2010online,neu2012adversarial,chandrasekaran2021online} and those based on the Online-Mirror-Descent (OMD) or the closely related Follow-the-Regularized-Leader (FTRL) framework~\citep{zimin2013online,rosenberg2019onlinea,rosenberg2019onlineb,jin2020learning,jin2021best,jin2022near}.
%\footnote{The OMD and FTRL frameworks are very similar to each other \citep[\S 28]{lattimore2020bandit}, so we regard them as the same framework throughout this paper.}
FTPL methods are usually computationally more efficient and easier to implement as it only requires solving an offline optimization problem (a.k.a. a \textit{planning} problem in the MDP literature). In contrast, OMD/FTRL methods require solving convex optimization problems over a complicated occupancy measure space.

Despite its computational advantages and ease in implementation, FTPL methods are much less studied (especially for learning AMDPs) since they are harder to analyze, less versatile, and are believed to suffer worse regret compared to OMD/FTRL methods.
A recent work by~\citet{wang2020refined} disputes the last common belief and shows that, 
for episodic AMDPs with full-information feedback, FTPL  also enjoys near-optimal regret, similarly to OMD/FTRL.
Nevertheless, little is known about FTPL for learning AMDPs with the more challenging bandit feedback --- to our knowledge, the only FTPL algorithm for this case is by~\citet{neu2010online}. However, that algorithm is analyzed under a strong assumption that every state is reachable by any policy with at least a constant probability $\alpha>0$. Such an exploratory assumption is too strong to be used in realistic applications.

\begin{table}[t]
\begin{minipage}{\textwidth}
% \vspace{-0.5cm}
    \caption{An overview of the proposed algorithms/results and comparisons with related works.}
    \label{tab:related work}
    \begin{savenotes}
    \renewcommand{\arraystretch}{1.5}
    \resizebox{\textwidth}{!}{%
    \begin{tabular}{|c|c|c|c|c|c|c|}\hline
    Setting%\footnote{Episodic stands for episodic finite-horizon AMDPs, and Infinite stands for infinite-horizon AMDPs.} 
    & Transition & Feedback & Algorithm & Regret\footnote{Here, $S$ and $A$ are the number of states and actions respectively, $K$ is the number of episodes, $T$ is the total number of steps, $\delay$ is the total amount of delay, $\tau$ is the mixing time of an ergodic MDP, and $D$ is the diameter of a communicating MDP.
    Several related works use different notations from ours, and their regret bounds have been converted based on our notations. For infinite-horizon AMDPs, the extra assumptions are listed after the regret bounds, with ``Ergodic'' standing for ergodic MDPs, ``Deterministic'' standing for MDPs with deterministic transitions, and ``Commu'' standing for communicating MDPs (the weakest assumption).} & Method & Effi.\footnote{This column indicates the algorithm's efficiency: \cmark~means polynomial (in all parameters) time complexity, \xmark~means %$\Omega(\lvert \Pi\rvert)=$
    $\Omega(A^S)$ time complexity, and \cmark!~means efficient assuming access to a planning oracle (that returns the best policy given all the MDP's parameters).
    Note that FTPL-based algorithms are usually easier to implement compared to OMD/FTRL-based ones (both treated as OMD-based in this table as they are quite similar).} \\\hline

    % \multirow{4}{*}{Episodic} & \multirow{2}{*}{Known} & \multirow{2}{*}{Full-info} & \multirow{4}{*}{---} & \citet{zimin2013online} & $\Otil(H\sqrt T)$ & OMD & \cmark \\\cline{4-7}
    %  & & &  & \citet{even2009online} & $\Otil(H^2\sqrt T)$ & FTPL & \cmark \\\cline{3-3}\cline{4-7}
    \multirow{8}{*}{\makecell{Episodic \\ $H$-horizon \\ AMDPs}} & \multirow{2}{*}{Known} & \multirow{4}{*}{Bandit} & \citet{zimin2013online} & $\Otil(H\sqrt{SAK})$ & OMD & \cmark \\\cline{4-7}
     & & & \textbf{This work} (\Cref{thm:regret of episodic AMDPs known}) & $\Otil(H^{\nicefrac 32}\sqrt{SAK})$ & FTPL & \cmark \\ \cline{2-2}\cline{4-7}
     
    %  & \multirow{4}{*}{Unknown} & \multirow{2}{*}{Full-info} & & \citet{rosenberg2019onlinea} & $\Otil(HS\sqrt{AT})$ & OMD & ? \\\cline{4-7}
    %  & & & & \citet{neu2012adversarial} & $\Otil(H^2S\sqrt{AT})$ \textsuperscript{\ref{footnote:enhanced by wang and dong}} & FTPL & \cmark \\ \cline{3-3}\cline{4-7}
     &  \multirow{6}{*}{Unknown} & & \citet{jin2020learning} & $\Otil(H^2S\sqrt{AK})$ & OMD & \cmark \\\cline{4-7}
     & & & \textbf{This work} (\Cref{thm:regret of episodic AMDPs unknown}) & $\Otil(H^2S\sqrt{AK})$ & FTPL & \cmark \\ \cline{3-7}

    %\multirow{4}{*}{Delayed} & \multirow{4}{*}{Unknown} 
    & & \multirow{4}{*}{\makecell{Bandit \& \\ Delayed}} & Delayed \textsc{Hedge} \citep{jin2022near} & $\Otil(H^2S\sqrt{AK}+H^{\nicefrac 32}\sqrt{S\delay})$ & OMD & \xmark \\\cline{4-7}
     & & & Delayed \textsc{UOB-FTRL} \citep{jin2022near} & $\Otil(H^2S\sqrt{AK}+H^{\nicefrac 32}SA\sqrt{\delay})$ & OMD & \cmark \\\cline{4-7}
     & & & \makecell{Delayed  \textsc{UOB-REPS} %w. delay-adapted \\ estimators 
      \citep{jin2022near}} & $\Otil(H^2S\sqrt{AK}+H^{\nicefrac 54}(SA)^{\nicefrac 14}\sqrt{\delay})$ & OMD & \cmark \\\cline{4-7}
     & & & \textbf{This work} (\Cref{thm:regret of delayed AMDPs}) & $\Otil(H^2S\sqrt{AK}+H^{\nicefrac 32}\sqrt{SA\delay})$ & FTPL & \cmark \\\cline{1-7}

     \multirow{8}{*}{\makecell{Infinite- \\ horizon \\ AMDPs}} & \multirow{8}{*}{Known} & \multirow{3}{*}{Full-info} & \citet{even2009online} & $\Otil(\tau^2\sqrt{T})$ (Ergodic) & OMD & \cmark \\\cline{4-7}
     & &  & \citet{chandrasekaran2021online} & $\Otil(S^4\sqrt T)$ (Deterministic) & FTPL & \cmark \\\cline{4-7}
     & & & \citet{chandrasekaran2021online} & $\Otil(D^2\sqrt {ST})$ (Commu) & FTPL & \xmark \\\cline{3-7}
    %  & & & Commu & \textbf{This work} & $\Otil(T^{5/6})$ & FTPL & \cmark ! \\\cline{3-7}
     & & \multirow{5}{*}{Bandit} & \citet{neu2014online} & $\Otil(\sqrt {\tau^3 AT})$ (Ergodic) & OMD & \cmark \\\cline{4-7}
     & &  & \citet{dekel2013better} & $\Otil(S^3 A T^{\nicefrac 23})$ (Deterministic) & OMD & \cmark \\\cline{4-7}
     & & & \textbf{This work} (\Cref{thm:regret of infinite-horizon AMDPs}) & $\Otil(A^{\nicefrac 12}(SD)^{\nicefrac 23}T^{\nicefrac 56})$ (Commu) & FTPL & \cmark ! \\\cline{4-7}
     & & & \textbf{This work} (\Cref{thm:regret of Hedge}) & $\Otil(A^{\nicefrac 13}(SDT)^{\nicefrac 23})$ (Commu) & OMD & \xmark \\\cline{4-7}
     & & & \citet{dekel2014bandits} & $\Omega(S^{\nicefrac 13}T^{\nicefrac 23})$ (if only Commu) & --- & --- \\\cline{1-7}
    \end{tabular}}
    \end{savenotes}
\end{minipage}
\end{table}

Motivated by this fact, we take a closer look at FTPL for learning AMDPs under bandit feedback, aiming at showing strong regret guarantees while enjoying its computational advantages.
We start with the standard episodic finite-horizon setting and indeed find some intriguing difficulties compared to OMD/FTRL.
After addressing these difficulties, we then show critical applications of FTPL methods to two more challenging setups: episodic AMDPs with \textit{delayed} bandit feedback and \textit{infinite-horizon} AMDPs with only communicating assumptions, with the latter result advancing the state-of-the-art.
More specifically,  our contributions are (see also \Cref{tab:related work} for a summary):
% \subsection{Our Contribution}
%In this paper, we focus on designing FTPL-based algorithms for AMDPs with bandit feedback, as it is much harder to implement an OMD-based algorithm due to the complicated convex optimization problem. We consider the following three settings: episodic AMDPs, AMDPs with delayed feedback \citep{lancewicki2020learning,jin2022near} and infinite-horizon AMDPs with only communicating assumptions \citep{chandrasekaran2021online} (without involving the strong ergodic assumption as \citet{even2009online,neu2014online}).\footnote{There are actually four fundamental subclasses of MDPs: ergodic, unichain, communicating and weakly communicating \citep{bartlett2009regal}, but the first two are quite similar and strong, while the last two are similar and weak. For simplicity, we simply divide them into two categories: ergodic and communicating.}
%To summarize, our main contributions are:
% \setlist{nolistsep}
\begin{enumerate}[leftmargin=*]
    \item We start with the heavily studied episodic setting with $K$ episodes, $H$ steps in each episode, $S$ states, and $A$ actions. Our first intriguing observation is that: since the loss of each policy is linear in a non-binary vector (i.e., the occupancy measure), existing analysis for the stability term of FTPL fails, even though it works for the binary case (e.g.,~\citet{neu2016importance}). Our next important observation is that there exists a simple fix to this issue that only leads to an extra $H$ factor. This eventually leads to $\Otil (H^{\nicefrac 32}\sqrt{SAK})$ regret when the transition is known (\Cref{alg:finite-horizon bandit feedback known transition}, \Cref{thm:regret of episodic AMDPs known}), which is only $\sqrt{H}$ factor larger than the near-optimal regret achieved by OMD~\citep{zimin2013online}, and $\Otil(H^2S\sqrt{AK})$ regret when the transition is unknown (\Cref{alg:appendix finite-horizon bandit feedback unknown transition}, \Cref{thm:regret of episodic AMDPs unknown}), matching the state-of-the-art again achieved by OMD~\citep{jin2020learning}.
    See \Cref{sec:episodic} for details.
%    For episodic AMDPs with bandit feedback, our algorithms enjoy an $\Otil (H^{\nicefrac 32}\sqrt{SAK} )$ regret with a known transition (\Cref{alg:finite-horizon bandit feedback known transition}, \Cref{thm:regret of episodic AMDPs known}) and $\Otil (H^2S\sqrt{AK} )$ regret with an unknown transition (\Cref{alg:appendix finite-horizon bandit feedback unknown transition},\Cref{thm:regret of episodic AMDPs unknown}).
%    Compared to the OMD-based algorithms, our known-transition regret is only sub-optimal with gap $\Otil(\sqrt H)$ \citep{zimin2013online} and the unknown-transition regret matches the state-of-the-art algorithm \citep{jin2020learning}.

    \item We next find that compared to OMD, the analysis of FTPL is much easier to be generalized to the delayed feedback setting where losses for episode $k$ are observed only at the end of episode $k+d_k$ for some $d_k \geq 0$~\citep{lancewicki2020learning, jin2022near}. Indeed, these two prior works demonstrate the difficulty of analyzing OMD with delay feedback, with~\citet{lancewicki2020learning} only achieving $\Otil((K+\delay)^{\nicefrac 23})$ regret (where $\delay=\sum_k d_k$ is the total amount of delay; dependence on other parameters is omitted) and~\citet{jin2022near} improving it to $\Otil(\sqrt{K+\delay})$ via either an inefficient algorithm or an efficient OMD-based algorithm with more involved analysis and/or new delayed-adapted loss estimators.
    FTPL, on the other hand, achieves $\Otil(\sqrt{K+\delay})$ regret by a simple extension of the analysis (\Cref{thm:regret of delayed AMDPs}). The dependence on $S$ and $A$ is also better than the OMD method of~\citep{jin2022near} with the same kind of standard loss estimators (though worse than their best result with the delayed-adapted estimators; see \Cref{tab:related work} and \Cref{sec:delay} for details).
    
%    For episodic AMDPs with delayed bandit feedback and an unknown transition, our algorithm (\Cref{alg:delayed}) suffers no more than $\Otil(H^2S\sqrt{AK}+H^{\nicefrac 32}\sqrt{SAK})$ regret (\Cref{thm:regret of delayed AMDPs}), which outperforms OMD-based \textsc{Delayed UOB-FTRL} algorithm by \citet{jin2022near}, though they also provided another OMD-based algorithm that gives a better regret guarantee.
    \item While our results above do not improve the best existing ones, our final application of FTPL provides the \textit{first} result for learning infinite-horizon communicating AMDPs with bandit feedback and known stochastic transitions.
    Specifically, our algorithm achieves $\Otil(A^{\nicefrac 12}(SD)^{\nicefrac 23}T^{\nicefrac 56})$ regret (\Cref{alg:appendix infinite-horizon bandit feedback known transition}, \Cref{thm:regret of infinite-horizon AMDPs}), where $D$ is the diameter of the MDP and $T$ is the total number of steps. It is efficient assuming access to an offline planning oracle (that returns the best stationary policy given a fixed transition function and a sequence of loss functions for each step).
    Previous results either only handle deterministic transitions~\citep{dekel2013better} or full-information loss feedback~\citep{chandrasekaran2021online}.
    Moreover, the FTPL algorithm of~\citet{chandrasekaran2021online} for stochastic transitions is inefficient even given the same planning oracle (since it explicitly adds independent noise to \textit{every} policy).    
  %  \footnote{There is currently no polynomial time algorithm for best stationary identification even with stochastic losses (the best we can conclude is that this is P-hard) \citep{papadimitriou1987complexity,mundhenk2000complexity}.}
%    To our knowledge, this is the only upper bound result for infinite-horizon AMDPs with bandit feedback and only communicating assumptions, other than the work by \citet{dekel2013better} that only applies to deterministic AMDPs (ADMDPs).
%    Compared to the recent work \citet{chandrasekaran2021online} that only works with full-information feedback, our algorithm is oracle-efficient while theirs is inefficient even provided such an oracle.
    For completeness, we also provide an inefficient algorithm (\Cref{alg:appendix hedge}) that achieves $\Otil(A^{\nicefrac 12}(SDT)^{\nicefrac 23})$ regret in our bandit setting, matching the $\Omega(T^{\nicefrac 23})$ lower bound of~\citet{dekel2014bandits} in terms of $T$. %for communicating AMDPs with bandit feedback.
    See \Cref{sec:infinite} for details.
\end{enumerate}

% \subsection{Our Contribution}
\subsection{Related Work}
%!TEX root=main.tex

\textbf{Follow-the-Perturbed-Leader:} FTPL is first proposed by~\citet{hannan1957approximation} and later popularized by~\citet{kalai2005efficient}.
It has proven to be extremely powerful for structured online learning problems (such as online shortest path) since its implementation is as easy as solving the corresponding offline optimization problem (such as finding the shortest path of a given graph).
Over the years, FTPL has been extended to problems with semi-bandit feedback~\citep{neu2015first, neu2016importance}, contextual information~\citep{syrgkanis2016efficient}, non-linear losses~\citep{dudik2020oracle}, smoothed adversaries~\citep{block2022smoothed, haghtalab2022oracle}, and others.
However, FTPL for learning AMDPs under bandit feedback is poorly understood, which motivates this work.
As we successfully show, improving our understanding of FTPL is indeed beneficial since it at least leads to new results for the infinite-horizon setting (in addition to its computational advantages for other settings).
Below, we briefly review the literature of AMDPs for the three settings we consider.

\textbf{Episodic Finite-Horizon AMDPs:}
Earlier works on this topic focus on the easier known transition case.
In particular, the OMD-based \textsc{O-REPS} algorithm by~\citet{zimin2013online} achieves $\Otil(H\sqrt K)$ regret with full-information feedback and $\Otil(H\sqrt{SAK})$ regret with bandit feedback, both optimal up to logarithmic factors. On the other hand, FTPL is recently shown to achieve $\Otil(H^2\sqrt K)$ regret with full-information feedback~\citep{wang2020refined}.
As mentioned, the only FTPL algorithm for bandit feedback is by~\citet{neu2010online}, which guarantees $\Otil(H^2\sqrt{AK}/\alpha)$ regret assuming that all states are reachable by any policy with a probability of at least $\alpha$.
In contrast, our FTPL algorithm removes this requirement and achieves $\Otil (H^{\nicefrac 32}\sqrt{SAK})$ regret, which is only $\sqrt{H}$ away from optimal.

When the transition is unknown, with full-information feedback, the OMD-based algorithm \textsc{UC-O-REPS}~\citep{rosenberg2019onlinea} achieves $\Otil(H^2S\sqrt{AK})$ regret, while the FTPL-based \textsc{FPOP}~\citep{neu2012adversarial} is %originally shown to achieve $\tilde{\mathcal O}(HSA\sqrt T)$ regret, but again later improved 
shown to achieve $\Otil(H^2S\sqrt{AK})$ regret as well~\citep{wang2020refined}. With bandit feedback, the OMD-based algorithm \textsc{UOB-REPS}~\citep{jin2020learning} also achieves the same $\Otil(H^2S\sqrt{AK})$ regret. 
At the same time, our algorithm enjoys the same guarantee and is the first FTPL algorithm for bandit feedback and unknown transition.
However, the current best lower bound for this problem is $\Omega(H^{\nicefrac 32}\sqrt{SAK})$~\citep{jin2018q}, so there is still an $\O(\sqrt{HS})$ gap.

Besides OMD and FTPL, there is, in fact, another category of algorithms for learning AMDPs: policy optimization~\citep{shani2020optimistic, luo2021policy}, which performs OMD in each \textit{state} and is also efficient. However, the regret bounds are worse by at least an $H$ factor~\citep{luo2021policy}.

\textbf{Delayed Feedback:} 
The most related works are~\citet{lancewicki2020learning} and \citet{jin2022near}, and we refer the reader to the references therein for the literature on delayed feedback for different problems.
Importantly, \citet{jin2022near} point out the unique difficulty when analyzing OMD/FTRL for AMDPs with delayed feedback.
Circumventing this difficulty one way or another, they develop three algorithms: the first one, Delayed \textsc{Hedge}, is inefficient; the second one, Delayed \textsc{UOB-FTRL}, achieves worse regret ($\sqrt{SA}$ larger for the delay-related term) compared to ours; and the third one makes use of a delay-adapted estimator and achieves the best bound (see \Cref{tab:related work}).
We emphasize again that our FTPL analysis is much simpler and a direct extension of the non-delayed case.
The current best lower bound for this problem is $\Omega(H^{\nicefrac 32}\sqrt{SAK}+H\sqrt{\delay})$~\citep{lancewicki2020learning}.

%For MDPs, the early works \citep{katsikopoulos2003markov,walsh2009learning} consider delays in observing the current state and the recent works \citep{lancewicki2020learning,howson2021delayed,jin2022near} consider delays in observing the losses, which is our case. For multi-armed bandits (MABs), delayed feedback has been extensively studied in both stochastic settings \citep{agarwal2011distributed,zhou2019learning,lancewicki2021stochastic} and adversarial settings \citep{cesa2018nonstochastic,zimmert2020optimal}. However, as pointed out by \citet{lancewicki2020learning,jin2022near}, delays in MDPs introduce new challenges compared to MABs.

\textbf{Infinite-Horizon AMDPs:} 
Learning AMDPs becomes significantly more difficult in the infinite horizon setting. As far as we know, all works in this line (including ours) assume a known transition function.
Earlier works focus on the simpler case with a strong \textit{ergodic} assumption~\citep{even2009online, neu2014online}.
For the more general \textit{communicating} assumptions, a recent work~\citep{chandrasekaran2021online} considers full-information feedback and develops an efficient FTPL  algorithm for deterministic transitions with $\Otil(S^4\sqrt T)$ regret and another inefficient FTPL algorithm for stochastic transitions with $\Otil(D^2\sqrt {ST})$ regret.
Under bandit feedback, prior works only study deterministic transitions~\citep{arora2012deterministic,dekel2013better}, with~\citet{dekel2013better} achieving $\Otil(S^3A T^{\nicefrac 23})$ regret, matching the lower bound~\citep{dekel2014bandits} for the $T$-dependency.
Our results are the first for bandit feedback and stochastic transitions.
Note that since bandit feedback is only more general, our oracle-efficient algorithm can also be applied to the full-information setting, while the only existing algorithm~\citep{chandrasekaran2021online} is computationally inefficient.

%As we are aware of, all works in this line assumed a known transition. There are two typical assumptions about the MDP structure: ergodicity and communication. Ergodicity is a common assumption in early works, such as the FTPL-based algorithm by \citet{even2009online} (with $\Otil(\sqrt T)$ regret for full-information feedback) and the one by \citet{neu2014online} (with $\Otil(\sqrt T)$ regret for bandit feedback). However, this assumption is too restrictive. Informally, it means any mistake made in logarithmic steps before will be forgiven \citep{even2009online}.

%In contrast, communication (\Cref{def:communicating MDP}) is a much weaker assumption, as it only requires every two states to be eventually (in expectation) reachable from each other. The recent work \citep{chandrasekaran2021online} considered communicating AMDPs with full-information feedback, and provided an efficient FTPL-based algorithm for deterministic AMDPs with $\Otil(\sqrt T)$ regret and an inefficient algorithm for stochastic ones with $\Otil(\sqrt T)$ regret. For bandit feedback, to our knowledge, the only existing works focused on deterministic AMDPs \citep{arora2012deterministic,dekel2013better}. The latter one gave an $\Otil(T^{2/3})$ guarantee, matching the lower bound for ADMDPs \citep{dekel2014bandits}.

\section{Preliminaries}
%!TEX root=main.tex
\textbf{General Notations:} We use $[N]$ to denote the set $\{1,2,\ldots,N\}$.
For a (finite) set $X$, we use $\triangle(X)\triangleq \{x\in \mathbb R_{\ge 0}^{\lvert X\rvert}\mid \sum_{i=1}^{\lvert X\rvert}x_i=1\}$ to denote the probability simplex over the set $X$.
We use $\Otil(\cdot)$ to hide all terms logarithmic in $H,S,A,K$ and $T$.
$\text{Laplace}(\eta)$ denotes the Laplace (also known as double-exponential) distribution with center $0$ and parameter $\eta$, whose probability density is $f(x)=\frac{\eta}{2}\exp(-\eta \lvert x\rvert)$, $\forall x\in \mathbb R$.
For an event $\mathcal E$, let $\mathbbm 1[\mathcal E]$ be its indicator.
In episodic settings, let $\{\mathcal F_k\}_{k=0}^K$ be the natural filtration such that $\mathcal F_k$ contains the history of episodes $1, \ldots, k$. With a slight abuse of notation, in the infinite-horizon setting, we also use $\{\mathcal F_t\}_{t=0}^T$ to denote the natural filtration.

\textbf{Episodic Adversarial Markov Decision Process:}
An episodic Adversarial Markov Decision Process (AMDP) is defined by a tuple $\mathcal M=(\mathcal S,\mathcal A,\mathbb P,\ell,K,H,s^1)$, where
$\mathcal S$ is the state space,
$\mathcal A$ is the action space,
$\mathbb P\colon [H]\times \mathcal S\times \mathcal A\to \triangle(\mathcal S)$ is the  transition function, %\footnote{We assume the transition does not change with $h$, i.e., it is \textit{time-homogeneous}. This is for a fair comparison with infinite-horizon MDPs, where the transition is always time-homogeneous.}
$\ell\colon [K]\times [H]\times \mathcal S\times \mathcal A\to [0,1]$ is the loss function unknown to the agent but fixed before the game (i.e., we are assuming an \textit{oblivious} adversary),\footnote{Note that the loss function can vary arbitrarily for different $(k,h)$-pairs, instead of being stochastic.}
$K$ is the number of episodes,
$H$ is the horizon length, and
$s^1\in \mathcal S$ is the initial state.
Denote by $S=\lvert \mathcal S\rvert<\infty$ and $A=\lvert \mathcal A\rvert<\infty$, the number of states and actions, respectively.
% For a fair comparison with infinite-horizon AMDPs, we define $T=HK$ as the total number of interactions between the agent and the environment.

The agent interacts with the environment for $K$ episodes.
For the $k$-th one ($k\le K$), she starts from the initial state $s^1$ and sequentially interacts with the environment for $H$ steps. At the $h$-th step (where $h\in [H]$), the agent observes state $s_k^h\in \mathcal S$, chooses an action $a_k^h\in \mathcal A$, observes and suffers the loss $\ell_k^h(s_k^h,a_k^h)$ (bandit feedback),\footnote{On the other hand, in the easier full-information setting, the entire $\ell_k^h$ is revealed.} and then transits to state $s_k^{h+1}$ according to the probability distribution $\mathbb P^h(\cdot \mid s_k^h,a_k^h)$. After $H$ steps, the episode ends and the agent proceeds to episode $k+1$.

A (deterministic) \textit{policy} of the agent is defined by $\pi=\{\pi^h\colon \mathcal S\to \mathcal A\}_{h\in [H]}$.
Denote the set of all deterministic policies by $\Pi$.
%The expected loss incurred by using policy $\pi\in \Pi$ starting from a state $s\in \mathcal S$ under some loss function $\ell$ is defined as
%$
%    V^\pi(s;\ell)\triangleq \E \left [\sum_{h=1}^H \ell^h(s^h,\pi^h(s^h))\middle \vert s^{h+1}\sim \mathbb P_h(\cdot\mid s^h,\pi^h(s^h)),s^1=s\right ].
%$
The expected loss incurred by policy $\pi\in \Pi$ for an episode with loss function $\hat\ell$ is denoted by
$
    V(\pi;\hat\ell)\triangleq \E \left [\sum_{h=1}^H \hat\ell^h(s^h,\pi^h(s^h))\middle \vert s^{h+1}\sim \mathbb P^h(\cdot\mid s^h,\pi^h(s^h)), \forall h < H\right ]
$.
Suppose the agent uses policies $\pi_1,\pi_2,\ldots,\pi_K$ for episodes $1,2,\ldots,K$, respectively. The total expected loss of the agent is then $\E\left [\sum_{k=1}^K V(\pi_k;\ell_k)\right ]$, where the expectation is taken with respect to the agent's private randomness.
The baseline is the best deterministic policy in hindsight, defined by $\pi^\ast\in \argmin_{\pi\in \Pi}\sum_{k=1}^K V(\pi;\ell_k)$.
The goal of the agent is to minimize her \textit{regret} over $K$ episodes, which is the difference between her total loss and that of $\pi^\ast$, formally defined as
\begin{equation*}
    \mathcal R_K\triangleq \E\left [\sum_{k=1}^K V(\pi_k;\ell_k)\right ]-\sum_{k=1}^K V(\pi^\ast;\ell_k).
\end{equation*}

\begin{comment}
\textbf{Feedback Models:} There are two types of feedback models, \textit{full-information} and \textit{bandit}.
In the former model, the agent will see the whole loss functions $\ell_k^h\colon \mathcal S\times \mathcal A\to [0,1],\forall h\in [H]$ at the end of $k$-th episode, whereas in the latter model, only $\ell_k^h(s_k^h,a_k^h),\forall h\in [H]$ is observable to the agent at the end of $k$-th episode.
\end{comment}

\textbf{Episodic AMDPs with Delayed Feedback:} This setup is exactly the same as the episodic AMDPs, except that the feedback $\{\ell_k^h(s_k^h,a_k^h)\}_{h=1}^H$ for episode $k$ is only available after $d_k$ episodes, i.e., at the end of the $(k+d_k)$-th episode.
Define $\delay=\sum_{k=1}^K d_k$ to be the total feedback delay, assumed to be known to the agent as this assumption can be easily relaxed via a {doubling trick} \citep{thune2019nonstochastic}.\footnote{As in \citet{jin2022near}, we only consider delayed loss feedback, but not delayed trajectory feedback, since the latter only affects the transition estimation and can be handled similarly to~\citet{lancewicki2020learning}.}

\textbf{Infinite-Horizon AMDPs:}
Similar to episodic AMDPs, infinite-horizon AMDPs is defined by a tuple $\mathcal M=(\mathcal S,\mathcal A,\mathbb P,\ell,T,s^1)$.
Here, starting from the initial state $s^1 \in \mathcal S$, the agent interacts with the environment for $T$ total steps without any reset, under the transition model $\mathbb P\colon \mathcal S\times \mathcal A\to \triangle(\mathcal S)$ (which does not vary over time)
and loss functions $\ell\colon [T] \times \mathcal S\times \mathcal A\to [0,1]$.
More specifically, 
at time $t\in [T]$, the agent observes state $s^t\in \mathcal S$, chooses an action $a^t\in \mathcal A$, observes and suffers loss $\ell^t(s^t,a^t)$, and then transits to $s^{t+1}\sim \mathbb P(\cdot \mid s^t,a^t)$.
Her goal is also to minimize the regret, defined as
\begin{equation}\label{eq:regret_def_infinite_horizon}
\begin{split}
    \mathcal R_T\triangleq \E\left [\sum_{t=1}^T \ell^t(s^t,a^t)\middle \vert s^{t+1}\sim \mathbb P(\cdot \mid s^t,a^t) \right ]-\min_{\pi\in \Pi}\E\left [\sum_{t=1}^T \ell^t(s^t,\pi(s^t))\middle \vert s^{t+1}\sim \mathbb P(\cdot \mid s^t,\pi(s^t))\right ],
\end{split}
\end{equation}
where $\Pi$ is now the set of all deterministic policies mapping from $\mathcal S$ to $\mathcal A$.
As pointed out by \citet{bartlett2009regal}, without any extra assumptions, sublinear regret is impossible for this problem due to the lack of resets.
Earlier works make a strong \textit{ergodic} assumption such that, intuitively, any mistake will be forgiven after logarithmic steps \citep{even2009online}.
Here, we instead focus on the much weaker \textit{communicating} assumption as in~\citet{chandrasekaran2021online}:
%For infinite-horizon AMDPs, we have to make some assumptions in order to make it learnable. Among the two typical assumptions that $\mathcal M$ is \textit{ergodic} or \textit{communicating}, the former one is stronger in the sense that any mistake will be forgiven in logarithmic steps \citep{even2009online}. In this paper, we focus on the latter assumption, which is much weaker and is formally defined by:
\begin{definition}[Communicating MDP]\label{def:communicating MDP}
We call an MDP $\mathcal M$ communicating if it has a finite \textit{diameter} $D \triangleq\max_{s\ne s'}\min_{\pi\in \Pi} \E[T(s'\mid \mathcal M,\pi,s)]$ where 
$T(s'\mid \mathcal M,\pi,s)$ is the (random) time step when state $s'$ is first reached by policy $\pi$ starting from state $s$. %We say $\mathcal M$ is \textit{communicating}, if the diameter $D(\mathcal M)<\infty$, where $D(\mathcal M)$ is defined as:
%\begin{equation*}
%    D(\mathcal M)=\max_{s\ne s'}\min_\pi \E[T(s'\mid \mathcal M,\pi,s)].
%\end{equation*}
\end{definition}

Just like~\citet{chandrasekaran2021online}, for technical reasons, we also need the following mild assumption saying that there exists a special state for the agent to ``park'' there without moving.
\begin{assumption}\label{assump:staying state}
There exist state $s^\ast \in \mathcal S$ and action $a^\ast \in \mathcal A$ such that $\mathbb P(s^\ast\mid s^\ast,a^\ast)=1$.
\end{assumption}

%Note that for infinite-horizon AMDPs with bandit feedback, communicating cases are statistically harder than ergodic ones as there exists an $\Omega(T^{\nicefrac 23})$ lower bound \citep{dekel2014bandits} for communicating cases, but an $\Otil(\sqrt T)$ upper bound \citep{neu2014online} for ergodic cases.

\section{FTPL for Episodic AMDPs}\label{sec:episodic}
In this section, we consider the basic (non-delayed) episodic setting. To best illustrate the unique difficulty we meet when analyzing FTPL and the way we address it, we first discuss the known-transition case (i.e., $\{\mathbb P^h\}_{h=1}^H$ is known to the agent), and then move on to unknown transitions.

\subsection{Known Transition}\label{sec:known transition}
%!TEX root=main.tex

Our algorithm follows the standard FTPL framework (see \Cref{alg:finite-horizon bandit feedback known transition}).
Ahead of time (as the adversary is oblivious), we sample a perturbation vector $z\colon [H]\times \mathcal S \times \mathcal A \rightarrow \mathbb{R}$ so that $z^h(s,a)$ is an independent sample from $\text{Laplace}(\eta)$ for some parameter $\eta$.
At the beginning of episode $k$, given the loss estimators $\hat \ell_1, \ldots, \hat \ell_{k-1}$ from previous episodes (whose construction will be specified later),
we simply play the policy that minimizes the cumulative perturbed estimated loss (break tie arbitrarily):
\begin{equation*}
    \pi_k=\argmin_{\pi\in \Pi}\left (V(\pi;z)+\sum_{k'=1}^{k-1} V(\pi;\hat \ell_{k'})\right )=\argmin_{\pi\in\Pi}V\left (\pi; \hat \ell_{0:k-1} \right ),
\end{equation*}
where we use $\hat \ell_{l:r}$ (where $0\le l\le r\le K$) as a shorthand notation for $\sum_{k'=l}^{r}\hat \ell_{k'}$ and $\hat\ell_0$ as an alias for $z$ for notational convenience.
This optimization over $\pi\in \Pi$ is a simple planning problem and can be solved by dynamic programming efficiently.

Upon seeing $s_k^h$, $a_k^h$, and $\ell_k^h(s_k^h,a_k^h)$, we construct the loss estimator $\hat \ell_k^h$ using the Geometric Re-sampling technique \citep{neu2016importance}.
The idea is to repeat the sampling procedure (\Cref{line:GR1} to \ref{line:GR2}) until the same pair $(s_k^h, a_k^h)$ is visited again at step $h$ or this has been repeated $L$ times for some parameter $L$.
Let the total number of trials be $M_k^h$, then the estimator is defined as $\hat \ell_k^h(s,a)=M_k^h\cdot \ell_k^h(s_k^h,a_k^h)\cdot \mathbbm 1[(s_k^h,a_k^h)=(s,a)]$ (\Cref{line:definition of M_k^h}).
Note that the sampling procedure can be done freely without interacting with the environment as the transition is known. 
The rational behind this estimator is that as long as $L$ is reasonably large, $M_k^h$ is a good approximation of the inverse probability of visiting $(s_k^h, a_k^h)$ (which is hard to calculate directly for FTPL), making $\hat\ell_k^h$ a good (and efficient) approximation of the standard importance weighted estimator~\citep{zimin2013online}.

%The reason why we do not use the importance-weighting estimator which is common in OMD-based algorithms (see, e.g., \citet{jin2020learning}) is that, we cannot directly calculate the probability of visiting a state-action pair $(s,a)$:
%\begin{equation*}
%    \sum_{\pi \in \Pi}\Pr\nolimits_z\{\pi_k=\pi\}\cdot \Pr\{s_k^h=s,\pi(s_k^h)=a\mid  s_{k}^{h+1}\sim \mathbb P_h(s_k^h,\pi(s_k^h)),s_k^1=s^1\}
%\end{equation*}
%
%as it is inefficient to calculate $\Pr_z\{\pi_k=\pi\}$ for every $\pi$.

\begin{algorithm}[t]
\caption{FTPL for Episodic AMDPs with Bandit Feedback and Known Transition}
\label{alg:finite-horizon bandit feedback known transition}
\begin{algorithmic}[1]
\Require{Laplace distribution parameter $\eta$. Geometric Re-sampling parameter $L$.}
\State  Sample perturbation $\hat\ell_0=z$ such that $z^h(s,a)$ is an independent sample of $\text{Laplace}(\eta)$.
\For{$k=1,2,\ldots,K$}
\State Calculate $\pi_k=\argmin_{\pi \in \Pi}V(\pi;\hat \ell_{0:k-1})$ (via dynamic programming). \label{line:find_the_leader}
\For{$h=1,2,\ldots,H$}
\State Observe $s_k^h$, play $a_k^h=\pi_k(s_k^h)$, suffer and observe loss $\ell_k^h(s_k^h,a_k^h)$.
\State Calculate loss estimator $\hat \ell_k^h$ via Geometric Re-sampling \citep{neu2016importance}:
\For{$M_k^h=1,2,\ldots,L$}
    \State Sample a fresh perturbation $\tilde z$ in the same way as $z$. \label{line:GR1}
    \State Calculate $\pi_k'=\argmin_{\pi \in \Pi}V(\pi;\hat \ell_{1:k-1}+\tilde z)$. 
    \State Simulate $\pi_k'$ for $h$ steps starting from $s^1$ and following transitions $\mathbb P^1, \ldots, \mathbb P^h$.   \label{line:GR2}
    \If{$(s_k^h,a_k^h)$ is visited at step $h$ or $M_k^h=L$}
        \State{Set $\hat \ell_k^h(s,a)=M_k^h\cdot \ell_k^h(s_k^h,a_k^h)\cdot \mathbbm 1[(s_k^h,a_k^h)=(s,a)]$ and break. \label{line:definition of M_k^h}}
    \EndIf
\EndFor
\EndFor
\EndFor
\end{algorithmic}
\end{algorithm}

%We have the following regret guarantee for \Cref{alg:finite-horizon bandit feedback known transition}, whose proof is sketched in \Cref{sec:sketched analysis} and formally presented in \Cref{sec:appendix episodic known}.
%\begin{theorem}[Regret of FTPL for Episodic AMDPs with Bandit Feedback and Known Transitions]\label{thm:regret of episodic AMDPs known}
%For an episodic AMDP with bandit feedback and known transitions, \Cref{alg:finite-horizon bandit feedback known transition} with $\eta^{-1}=\sqrt{HSAK}$ and $L=\sqrt{SAK/H}$ enjoys a regret guarantee of
%\begin{equation*}
%    \mathcal R_T\le \Otil\left (H^{\nicefrac 32}\sqrt{SAK}\right ).
%\end{equation*}
%\end{theorem}

%\subsection{Sketched Analysis for Known-Transition Cases}\label{sec:sketched analysis}
%!TEX root=main.tex

\textbf{Analysis Sketch:}
While our algorithm follows the standard FTPL framework, we find some intriguing difficulty in the analysis that is unique to MDPs and undiscovered before.
To illustrate this difficulty, let us first describe an overview of the analysis.
First, since the loss estimators are almost unbiased (as shown by~\citet{neu2016importance}), we only need to focus on the regret with respect to the estimated losses, that is,
$\E\big [\sum_{k=1}^K V(\pi_k;\hat\ell_k)-\sum_{k=1}^K V(\pi^\ast;\hat\ell_k)\big ]$.
Adding and subtracting $\E\big [\sum_{k=1}^K V(\pi_{k+1};\hat\ell_k)\big]$ (the loss of an imaginary ``leader'' that looks one episode ahead),
our next goal is to bound the so-called \textit{stability term} $\E\big [\sum_{k=1}^K V(\pi_k;\hat\ell_k)-\sum_{k=1}^K V(\pi_{k+1};\hat\ell_k)\big ]$ (the rest, usually referred as the \textit{error term}, can be bounded by the
standard ``be-the-leader'' lemma).

For the stability term, fix an episode $k$ and define $p_k(\pi)$ as the probability of selecting $\pi$ as $\pi_k$ w.r.t. the randomness of the perturbation $z$.
Further introduce the notion of \textit{occupancy measures} \citep{altman1999constrained,neu2012adversarial}: each policy $\pi\in \Pi$ induces $H$ occupancy measures $\mu_\pi^h\in  \triangle(\mathcal S\times \mathcal A)$, $\forall h\in [H]$, where $\mu_{\pi}^h(s,a)$ denotes the probability of visiting $(s,a)$ at step $h$ if one executes policy $\pi$ starting from the initial state $s^1$.
With these notations, each summand for the stability term becomes:
\[
\E\left[ V(\pi_k;\hat\ell_k)- V(\pi_{k+1};\hat\ell_k) \right]
= \E\left[ \sum_{\pi\in \Pi}(p_k(\pi)-p_{k+1}(\pi)) \left\langle \mu_\pi,\hat \ell_k \right\rangle \right],
\]
where $\big\langle \mu_\pi,\hat \ell_k \big\rangle \triangleq \sum_{h=1}^H \big\langle \mu_\pi^h,\hat \ell_k^h \big\rangle$.
This stability term is exactly in the same form as that in Lemma~8 of~\citet{neu2016importance} or Lemma~10 of~\citet{syrgkanis2016efficient} for (contextual) semi-bandit problems, except that in their contexts, $\mu_\pi$ is a \textit{binary} vector.
This seemingly slight difference turns out to be important!
Specifically, in these two prior works, they both show (using our notations): 
\begin{equation}\label{eq:ideal_stability1}
p_{k+1}(\pi) \geq p_k(\pi)\exp\left(-\eta   \left\langle \mu_\pi,\hat \ell_k \right\rangle \right),
\end{equation}
which, together with the fact $\exp(-x) \geq 1-x$, implies
\begin{equation}\label{eq:ideal_stability2}
\E\left[ V(\pi_k;\hat\ell_k)- V(\pi_{k+1};\hat\ell_k) \right] 
\leq \eta\E\left[ \sum_{\pi\in \Pi}p_k(\pi)\left\langle \mu_\pi,\hat \ell_k \right\rangle^2 \right].
\end{equation}
Readers familiar with the online learning literature would have recognized the last expression, since it is also the standard stability term achieved by (inefficiently) running the classical \textsc{Hedge} algorithm~\citep{freund1997decision} over all policies (see e.g. Theorem~7.3 of~\citet{bubeck2011introduction}).
Indeed, this term is small enough and can be shown to be of order $\O(\eta HSA)$ in our context after plugging in the definition of the loss estimators, which would then basically complete the proof.

However, not only do we realize that the proof of \Cref{eq:ideal_stability1} heavily rely on the binary nature of $\mu_\pi$, we in fact also find a counterexample where \Cref{eq:ideal_stability2} is simply \textit{incorrect} when $\mu_\pi$ is non-binary
(see \Cref{sec:comparism with Syrgkanis16} for the counterexample).
We find this fact intriguing, because \Cref{eq:ideal_stability2} holds for the aforementioned inefficient \textsc{Hedge} algorithm regardless whether $\mu_\pi$ is binary or not.

Further examining the proof of~\citet{neu2016importance} and~\citet{syrgkanis2016efficient}, however, one can prove the following weaker version of \Cref{eq:ideal_stability1} and \Cref{eq:ideal_stability2} (namely \Cref{eq:actual_stability1} and \Cref{eq:actual_stability2} respectively).
\begin{lemma}[Single-Step Stability]\label{lem:single-step stability}
For all $k\in [K]$ and $\pi \in \Pi$, we have
\begin{equation}\label{eq:actual_stability1}
    %p_{k+1}(\pi)\le p_k(\pi)\exp\left (\eta \sum_{h=1}^H\lVert \hat \ell_k^h\rVert_1\right ),\quad 
    p_{k+1}(\pi)\ge p_k(\pi)\exp\left (-\eta \sum_{h=1}^H\lVert \hat \ell_k^h\rVert_1\right ),
\end{equation}
and thus 
\begin{equation}\label{eq:actual_stability2}
\E\left[ V(\pi_k;\hat\ell_k)- V(\pi_{k+1};\hat\ell_k) \right]
\leq \eta \E\left[\left(\sum_{h=1}^H\lVert \hat \ell_k^h\rVert_1 \right) \sum_{\pi \in \Pi} p_k(\pi)\left\langle \mu_\pi,\hat \ell_k \right\rangle \right].
\end{equation}
\end{lemma}
Fortunately, while \Cref{eq:actual_stability2} looks seemingly much larger than the classic bound \Cref{eq:ideal_stability2}, it is in fact at most larger by an $H$ factor, that is, the right-hand side of \Cref{eq:actual_stability2} can be shown be of order $\O(\eta H^2SA)$ (see \Cref{lem:appendix known stability term} in the appendix).
Putting everything together, this allows us to prove the following regret guarantee for \Cref{alg:finite-horizon bandit feedback known transition}, which is $\sqrt{H}$ larger than the optimal bound~\citep{zimin2013online} due to the weakened stability bound. One may refer to \Cref{sec:appendix episodic known} for the formal proof.
\begin{theorem}%[Regret of FTPL for Episodic AMDPs with Bandit Feedback and Known Transitions]
\label{thm:regret of episodic AMDPs known}
For episodic AMDPs with bandit feedback and known transitions, \Cref{alg:finite-horizon bandit feedback known transition} with $\eta=\nicefrac{1}{\sqrt{HSAK}}$ and $L=\sqrt{\nicefrac{SAK}{H}}$ ensures 
$
    \mathcal R_T = \Otil\big (H^{\nicefrac 32}\sqrt{SAK}\big ).
$
\end{theorem}

\subsection{Unknown Transition}
%!TEX root=main.tex
To handle unknown transitions, we mostly follow existing ideas.
First, for each episode $k$ we maintain a confidence set $\mathcal P_k$ of the transition function as \citet{jin2022near}, whose construction is given in \Cref{sec:confidence set construction}.
These confidence sets ensure that i) $\mathbb P\in \mathcal P_k$ with high probability and ii) $\mathcal P_{k+1}\subseteq \mathcal P_k$. %, and iii) the number of episodes such that $\mathcal P_{k+1}\ne \mathcal P_k$ is at most $\Otil(HSA)$. 
Generalizing the notation $V(\pi;\hat\ell)$, we use $V(\pi;\hat\ell,P)$ to denote the expected loss of policy $\pi$ for an episode with loss function $\hat\ell$ and transition $P$ (so $V(\pi;\hat\ell) = V(\pi;\hat\ell, \mathbb P)$). 
Then deploying the idea of optimism, we replace \Cref{line:find_the_leader} of \Cref{alg:finite-horizon bandit feedback known transition} with $\pi_k = \argmin_{\pi \in \Pi} \min_{P \in \mathcal P_k}V(\pi; \hat \ell_{0:k-1},P)$, which can be efficiently found using Extended Value Iteration~\citep{jaksch2010near}.
As~\citet{wang2020refined} argues, this is far more efficient than performing OMD over occupancy measure spaces.

We also need to modify the Geometric Re-sampling procedure accordingly since \Cref{line:GR2} requires using the true transition.
To do so, we combine the procedure with the idea of \textit{upper occupancy measures} from~\citet{jin2020learning}.
Specifically, in each trial we sample $\pi_k'$ in the same way as $\pi_k$ but with a fresh perturbation, then find the optimistic transition within $\mathcal P_k$ that maximizes the probability of $\pi_k'$ visiting $(s_k^h, a_k^h)$ (which can be done efficiently using dynamic programming as shown by~\citet{jin2020learning}),
and finally simulate $\pi_k'$ for $h$ steps following this optimistic transition.

Due to space limit, the full algorithm, \Cref{alg:appendix finite-horizon bandit feedback unknown transition}, is deferred to \Cref{sec:appendix episodic unknown}.
The analysis of the extra regret caused by the transition estimation error can be handled similarly to~\citet{jin2022near} (more specifically, their Delayed \textsc{Hedge} algorithm). As in previous works, this happens to be of order $\Otil(H^2S\sqrt{AK})$ and becomes the dominating term of the regret.
This makes our final regret the same as the state-of-the-art~\citep{jin2020learning}, despite the weaker single-step stability lemma discussed in \Cref{sec:known transition} (since this part is dominated now).
Formally, we have the following regret guarantee.

\begin{theorem}%[Regret of FTPL for Episodic AMDPs with Bandit Feedback and Unknown Transitions]
\label{thm:regret of episodic AMDPs unknown}
For episodic AMDPs with bandit feedback and unknown transitions, \Cref{alg:appendix finite-horizon bandit feedback unknown transition} with $\eta=\nicefrac{1}{\sqrt{HSAK}}$ and $L=\sqrt{\nicefrac{SAK}{H}}$ ensures
$
    \mathcal R_T= \Otil \big (H^2S\sqrt{AK}\big )
$.
\end{theorem}

\section{FTPL for Episodic AMDPs with Delayed Feedback}\label{sec:delay}
%!TEX root=main.tex
In this section, we show how our FTPL algorithm and analysis can be easily extended to the delayed feedback setting where the losses for episode $k$ are only observed at the end of episode $k+d_k$.
The only change to the algorithm is to naturally delay the loss estimator construction until the loss feedback is received, and at each episode $k$ only use the estimators constructed so far, i.e., $\Omega_k\triangleq\{k'\mid k'+d_{k'}<k\}$, to compute the current policy $\pi_k$.
See \Cref{alg:delayed} in \Cref{sec:appendix delayed}.

To show how the analysis works, we focus on the known transition case at this moment for simplicity.
Similar to the non-delayed case, the key is to bound the stability term, which was $\E\big [\sum_{k=1}^K V(\pi_k;\hat\ell_k)-\sum_{k=1}^K V(\pi_{k+1};\hat\ell_k)\big ]$ in \Cref{sec:known transition}, but now becomes $\E\big [\sum_{k=1}^K V(\pi_k;\hat\ell_k)-\sum_{k=1}^K V(\tilde\pi_{k+1};\hat\ell_k)\big ]$ where $\tilde\pi_{k+1} = \argmin_{\pi \in \Pi}V(\pi;\hat \ell_{0:k})$ is a ``cheating policy' \citep{gyorgy2021adapting,jin2022near} that uses all loss estimators from the first $k$ episodes (which matches $\pi_{k+1}$ for the non-delayed case).
By the exact same analysis as \Cref{eq:actual_stability1} and \Cref{eq:actual_stability2}, one can show 
\[
\E\left[ V(\pi_k;\hat\ell_k)- V(\tilde\pi_{k+1};\hat\ell_k) \right]
\leq \eta \E\Big[\underbrace{\Big(\sum_{k'\in [k]\setminus \Omega_k}\sum_{h=1}^H \lVert \hat \ell_{k'}^h\rVert_1 \Big)}_{\textsc{Diff}} \sum_{\pi \in \Pi} p_k(\pi)\left\langle \mu_\pi,\hat \ell_k \right\rangle \Big],
\]
where the \textsc{Diff} term is the cumulative $\ell_1$ norms of all the estimators used in computing $\tilde\pi_{k+1}$ but not $\pi_k$ (again, a direct generalization of \Cref{eq:actual_stability2} where only $k$ satisfies such conditions for $k'$).
It is then not hard to imagine that when summed over $k$, the \textsc{Diff} term is eventually related to the total amount of delay $\delay = \sum_k d_k$.
Indeed, the sum of all stability terms over $K$ episodes can be shown to be of order $\O(\eta H^2SA(K + \delay))$.
This is basically all the extra elements we need in the proof.
More generally for unknown transitions, we prove the following guarantee (see \Cref{sec:appendix delayed} for the proof).
\begin{theorem}%[Regret of FTPL for Episodic AMDPs with Delayed Bandit Feedback and Unknown Transitions]
\label{thm:regret of delayed AMDPs}
For episodic AMDPs with delayed bandit feedback and unknown transitions, \Cref{alg:delayed} with $\eta=\nicefrac{1}{\sqrt{HSA(K+\delay)}}$ and $L=\sqrt{\nicefrac{HSA}{H}}$ ensures
$
    \mathcal R_T= \Otil\big (H^2S\sqrt{AK}+H^{\nicefrac 32}\sqrt{SA\delay}\big ).
$
\end{theorem}

The simplicity of our analysis is similar to the Delayed \textsc{Hedge} algorithm~\citep{jin2022near}, but the latter is inefficient with time complexity $\Omega(A^S)$.
The efficient Delayed UOB-FTRL algorithm~\citep{jin2022near} requires a more complicated analysis and only achieves $\Otil\big (H^2S\sqrt{AK}+H^{\nicefrac 32}SA\sqrt{\delay}\big )$ regret (which is worse than ours),
while its improved variant Delayed UOB-REPS with a new \textit{delay-adapted} estimator achieves the current best bound $\Otil(H^2S\sqrt{AK}+H^{\nicefrac 54}(SA)^{\nicefrac 14}\sqrt{\delay})$.
However, it is unclear to us whether such delay-adapted estimators can help improve FTPL.
Finally, we again remark that the current best lower bound is $\Omega(H^{\nicefrac 32}\sqrt{SAK}+H\sqrt{\delay})$~\citep{lancewicki2020learning}.

\section{FTPL for Infinite-Horizon AMDPs}\label{sec:infinite}
%!TEX root=main.tex
At last, we discuss how FTPL can be used to derive the first no-regret algorithm for 
infinite-horizon communicating AMDPs with bandit feedback and (known) stochastic transition.
Note that learning infinite-horizon AMDPs is much more difficult due to the lack of resets (in a sense, this is like a finite-horizon problem but with only one long episode with $T$ steps).
Another way to see the difficulty is that the benchmark in the regret definition \Cref{eq:regret_def_infinite_horizon} is evaluated on states generated by following $\pi^\ast$ repeatedly for $T$ rounds, without any resets.
From a technical viewpoint, this requires the algorithm to also make sure that, when following a policy $\pi$, its suffered loss is indeed close to the total loss if $\pi$ has been followed since the very beginning, which is unnatural without ergodic assumptions.

\citet{chandrasekaran2021online} resolve this issue by the combination of two ideas.
First, under the mild \Cref{assump:staying state}, they show that whenever the agent wants to switch the current policy to another policy $\pi$, 
there exists a procedure to make sure that after $\O(D^2)$ steps of a transition phase, the agent's state distribution is exactly the same as that induced by following $\pi$ from the very beginning.
That is, after this switching procedure, the agent can ``pretend'' that she has followed $\pi$ all the time.
Second, since this procedure requires a cost of $\O(D^2)$ steps (where the loss of the agent can be arbitrarily bad and only trivially bounded by $\O(D^2)$), the algorithm needs to switch its policy infrequently.

Our algorithm follows the same ideas.
However, while low-switching is relatively easy to ensure in the full-information case without paying extra regret,
it is known that with bandit feedback  there is an unavoidable trade-off between the number of switches and the regret, which can be optimally balanced via a simple epoching scheme~\citep{dekel2014bandits}.
To this end, we divide the total $T$ steps into $J=o(T)$ epochs, each with length $H=\nicefrac TJ = \omega(D^2)$. At the beginning of the $j$-th epoch, we compute a new policy $\pi_j$, apply the switching procedure of \citet{chandrasekaran2021online} to adjust the state distribution (see \Cref{alg:policy switching}), and finally follow the same policy $\pi_j$ for the rest of the epoch.
This clearly only introduces $J$ switches, which contributes to at most $\O(JD^2)$ extra regret.

It remains to specify how to find $\pi_j$ in epoch $j$ using FTPL.
The key difference compared to the episodic case is that, due to the lack of resets, we need to add perturbation to \textit{every} time step instead of just to each of the $H$ steps of an episode. We then still play the policy that minimizes the cumulative estimated losses plus all the perturbed losses.
Formally, $\pi_j$ is defined as:
\begin{equation}\label{eq:FTPL update infinite-horizon}
    %\pi_j=\argmin_{\pi \in \Pi}\left (\sum_{j'=1}^{j-1}\sum_{t\in \mathcal T_{j'}}\langle \mu_\pi^t,\hat \ell^t\rangle+\sum_{t=1}^T \langle \mu_\pi^t,z^t\rangle\right ),
    \pi_j=\argmin_{\pi \in \Pi} \E\left [\sum_{t=1}^{(j-1)\nicefrac TJ} \hat\ell^t(s^t,\pi(s^t)) + \sum_{t=1}^T z^t(s^t,\pi(s^t))\middle \vert s^{t+1}\sim \mathbb P(\cdot \mid s^t,\pi(s^t)), \;\forall t\right ],
\end{equation}

where $\{z^t: \mathcal S \times \mathcal A \rightarrow \mathbb R\}_{t\in [T]}$ is such that each $z^t(s,a)$ is an independent sample of $\text{Laplace}(\eta)$, and each $\hat\ell^t$ is the estimator of $\ell^t$ constructed from the Geometric Re-sampling procedure.

Unfortunately, as far as we know, there is in fact no existing polynomial time algorithm for solving \Cref{eq:FTPL update infinite-horizon} (the difficulty comes from the restriction on \textit{stationary} policies whose behavior does not vary over time).
Even if the losses are stochastic, the problem is only known to be P-hard~\citep{papadimitriou1987complexity,mundhenk2000complexity} and
no polynomial algorithm has been developed.

However, note that this optimization is exactly in the same form as the benchmark in the regret definition \Cref{eq:regret_def_infinite_horizon}.
Following many prior works such as~\citet{dudik2020oracle, block2022smoothed, haghtalab2022oracle}, 
we thus assume access to a planning oracle that solves this offline problem, making our algorithm only oracle-efficient instead of truly polynomial-time-efficient.
Note that even given this oracle, the algorithm of~\citet{chandrasekaran2021online} is inefficient since it creates independent perturbation for \textit{each} of the $A^S$ policies, while our perturbation is much more compact.

In terms of the analysis, the key extra challenge is caused by having $T$ perturbed losses. Indeed, the same analysis from the episodic case (\Cref{lem:error term wang}) would lead to a term of order $\Otil(T/\eta)$, which is prohibitively large.
Instead, inspired by \citet{syrgkanis2016efficient}, we provide a different analysis showing that this can be improved to $\Otil(S\sqrt{AT}/\eta)$, which has worse dependencies on $S$ and $A$ but better dependency on $T$, the key to ensure sub-linear regret eventually.
To conclude, our FTPL algorithm achieves the following guarantee (see \Cref{sec:appendix infinite-horizon FTPL} for the full algorithm and analysis).

%At last, another difficulty that only rises for infinite-horizon AMDP is that we have to add $T$ perturbations instead of $H$. One may refer to \Cref{sec:infinite horizon algorithm} for more discussions.
%Consequently, we have to analyze the error term differently, as using \Cref{lem:error term wang} as in \Cref{lem:error term bound} will give $\Otil(T/\eta)$, making the regret linear in $T$. Instead, we will use \Cref{lem:error term syrg} from \citet{syrgkanis2016efficient} to ensure an $\Otil(S\sqrt{AT}/\eta)$ bound, which is better in $T$ albeit worsening the dependencies on $S$ and $A$.

%The regret guarantee of \Cref{alg:appendix infinite-horizon bandit feedback known transition} is stated below, whose proof is deferred to \Cref{sec:appendix infinite-horizon FTPL}.
\begin{theorem}%[Regret of FTPL for Infinite-horizon AMDPs with Bandit Feedback and Known Transitions]
\label{thm:regret of infinite-horizon AMDPs}
For infinite-horizon AMDPs with bandit feedback and known transitions, \Cref{alg:appendix infinite-horizon bandit feedback known transition} with $\eta=\frac{S^{\nicefrac 13}}{D^{\nicefrac 23}T^{\nicefrac 13}}$, $J=\frac{S^{\nicefrac 23}A^{\nicefrac 12}T^{\nicefrac 56}}{D^{\nicefrac 43}}$ and $L=\frac{S^{\nicefrac 13}A^{\nicefrac 12}T^{\nicefrac 16}}{D^{\nicefrac 23}}$ ensures $\mathcal R_T=\Otil\left (A^{\nicefrac 12}(SD)^{\nicefrac 23}T^{\nicefrac 56}\right )$.
\end{theorem}

We emphasize again that this is the first (oracle-efficient) algorithm for this setting.
Even in the easier full-information setting (where $\ell^t$ is fully revealed at the end of time $t$),
our algorithm also has its computational advantages compared to that of~\citet{chandrasekaran2021online}, since, as mentioned, their algorithm requires $\Omega(A^S)$ complexity (albeit with a better regret bound $\Otil(D^2\sqrt {ST})$).

The best lower bound for this setting is $\Omega(S^{\nicefrac 13}T^{\nicefrac 23})$~\citep{dekel2014bandits}.
\citet{dekel2013better} achieve $\Otil(S^3AT^{\nicefrac 23})$ but only when the transition is deterministic.
For completeness, we provide a \textsc{Hedge}-based \textit{inefficient} algorithm (\Cref{sec:appendix infinite-horizon Hedge}) for general stochastic transitions, which achieves the optimal regret in terms of the dependence on $T$, improving our oracle-efficient FTPL algorithm.  %See \Cref{sec:appendix infinite-horizon Hedge} for details.

% (\Cref{alg:appendix hedge} in \Cref{sec:appendix infinite-horizon Hedge}) for infinite-horizon, communicating AMDPs with bandit feedback, whose regret guarantee is $\Otil(T^{\nicefrac 23})$ as stated in the theorem below, matching the $\Omega(T^{\nicefrac 23})$ regret lower bound for communicating AMDPs with bandit feedback \citep{dekel2014bandits} in terms of $T$.

\begin{theorem}%[Regret of Hedge for Infinite-horizon AMDPs with Bandit Feedback and Known Transitions]
\label{thm:regret of Hedge}
For infinite-horizon AMDPs with bandit feedback and known transitions, \Cref{alg:appendix hedge} with $\eta=\frac{S^{\nicefrac 13}}{A^{\nicefrac 13}(DT)^{\nicefrac 23}}$ and $J=\frac{(ST)^{\nicefrac 23}A^{\nicefrac 13}}{D^{\nicefrac 43}}$ ensures $\mathcal R_T=\Otil\left (A^{\nicefrac 13}(SDT)^{\nicefrac 23}\right )$.
\end{theorem}

\section{Conclusion}
%!TEX root=main.tex
In this paper, we designed FTPL-based algorithms for adversarial MDPs with bandit feedback in various settings, including episodic settings, delayed feedback settings and infinite-horizon settings.
Our algorithms are easy to implement as they only require solving the offline planing problem, and in some cases they match the state-of-the-art performance or are even the first ever no-regret algorithms.

One interesting open question is whether, despite our counterexample, \Cref{eq:ideal_stability2} can still hold with a larger constant for the right-hand side, either with our current algorithm or via some modified versions (for example with a different kind of perturbation). Achieving this would lead to an improved version of \Cref{lem:single-step stability} and thus give the near-optimal delay-related regret term $\Otil(H^{\nicefrac 32}\sqrt{\delay})$ for the delayed feedback setting, which is not currently achieved by any existing algorithms.

An alternative direction is to try to equip our \Cref{alg:delayed} (for episodic AMDPs with delayed feedback) with the ``delay-adapted'' loss estimators proposed by \citet{jin2022near}. As their analysis heavily relies on the exponential weight scheme (see their Lemma D.7, which bounds KL divergences between consecutive policies), it is unclear to us whether FTPL enjoys a similar property.

Another important future direction is to improve our results in the infinite-horizon setting,
such as improving the $\Otil(T^{\nicefrac 56})$ oracle-efficient regret upper bound, removing the usage of oracles, or dealing with the unknown transition case (which has not yet been studied at all).

There are also several possible generalizations of our setting. For example, we only assume the losses to be adversarial. Further incorporating evolving transition is an important next step. There is already an FTPL-based algorithm \citep{yu2009arbitrarily} for evolving dynamics (though they are assuming ergodic infinite-horizon MDPs), which builds upon the FTPL analysis by \citet{even2009online} (see their Lemma III.3). Although our work directly improves the performance guarantee of \citet{even2009online}, it is highly unclear whether we can adopt the algorithm of \citet{yu2009arbitrarily} for unknown-transition episodic MDPs (they assumed the transitions to be revealed after each episode) or infinite-horizon weakly communicating MDPs. Solving either case will be interesting. Moreover, considering dynamic regret instead of static regret can also be challenging.

\begin{ack}
We greatly acknowledge Vasilis Syrgkanis for the helpful discussion about whether their single-step stability lemma \citep[Lemma 10]{syrgkanis2016efficient} holds for non-binary action spaces. We also thank the anonymous reviewers for their insightful comments, which we greatly benefit from.
HL is supported by NSF Award IIS-1943607 and a Google Faculty Research
Award.

\end{ack}

\bibliographystyle{plainnat}
\bibliography{references}

\newpage

\appendix
\renewcommand{\appendixpagename}{\centering \sffamily \LARGE Supplementary Materials}
\appendixpage

\startcontents[section]
\printcontents[section]{l}{1}{\setcounter{tocdepth}{2}}

\section{Notations}
%!TEX root=main.tex

We summarize our notations used in the appendix below:
\begin{itemize}
    \item For a policy $\pi\in \Pi$ and a transition $P\colon [H]\times \mS\times \mA\to \triangle(\mS)$, the occupancy measure of $\pi$ at the $h$-th step ($h\in [H]$) is defined as
    \begin{equation*}
        \mu_\pi^h(s,a;P)=\Pr\{(s^h,a^h)=(s,a)\mid a^h=\pi^h(s^h),s^{h+1}\sim P^h(\cdot \mid s^h,a^h),s^1\}.
    \end{equation*}
    
    We will use $\mu_\pi^h(P)$ to denote the vector $\{\mu_\pi^h(s,a;P)\}_{(s,a)\in \mS\times \mA}$. Specifically, if $P$ is the true transition $\mathbb P$, we will abbreviate it as $\mu_\pi^h$ for simplicity.
    \item With a slight abuse of notation, for infinite-horizon AMDPs, we will also use the same notation $\mu_\pi^t\in \triangle(\mS\times \mA)$ ($t\in [T]$) to refer to the occupancy measure of $\pi$ at time slot $t$, starting from the first state $s^1$ and following the transition $\mathbb P$ (as we do not consider unknown transition cases for infinite-horizon AMDPs, we will always abbreviate the transitions).
    \item For a policy $\pi\in \Pi$, a transition $P\colon [H]\times \mS\times \mA\to \triangle(\mS)$ and a loss function $\hat \ell\colon [H]\times \mS\times \mA\to \mathbb R_{\ge 0}$, the value function is defined as
    \begin{equation*}
        V(\pi;\hat \ell,P)=\E\left [\hat \ell^h(s^h,a^h)\middle \vert a^h=\pi(s^h),s^{h+1}\sim P^h(s^h,a^h),s^1\right ]=\sum_{h=1}^H \langle \mu_\pi^h(P),\hat \ell^h\rangle.
    \end{equation*}
    \item A perturbation $z\colon \mathbb R^{[H]\times \mS\times \mA}$ is a fresh sample such that
    \begin{equation*}
        z^h(s,a)\sim \text{Laplace}(\eta)\text{ and each entry is independently sampled}.
    \end{equation*}
    
    For simplicity in notations, we use $\hat \ell_0$ as an alias of $z$.
    \item For a sequence of loss functions $\hat \ell_1,\hat \ell_2,\ldots,\hat \ell_k$, we use $\hat \ell_{1:k}$ to denote $\sum_{k'=1}^k \hat \ell_{k'}$.
\end{itemize}

\section{Analysis of Episodic AMDP Algorithms}\label{sec:appendix episodic}
\subsection{Known Transition Case (\Cref{thm:regret of episodic AMDPs known})}\label{sec:appendix episodic known}
%!TEX root=main.tex
For convenience, we restate the algorithm for episodic AMDPs with bandit feedback and known transitions in \Cref{alg:appendix finite-horizon bandit feedback known transition}. As shown by \citet[Appendix A.2]{syrgkanis2016efficient}, for an oblivious adversary (which is our case), it suffices to draw the perturbations once at the beginning of the interaction (i.e., the perturbation $z$ is fixed throughout the game).

\begin{algorithm}[t]
\caption{FTPL for Episodic AMDPs with Bandit Feedback and Known Transition}
\label{alg:appendix finite-horizon bandit feedback known transition}
\begin{algorithmic}[1]
\Require{Laplace distribution parameter $\eta$. Geometric Re-sampling parameter $L$.}
\State  Sample perturbation $\hat\ell_0=z$ such that $z^h(s,a)$ is an independent sample of $\text{Laplace}(\eta)$.
\For{$k=1,2,\ldots,K$}
\State Calculate $\pi_k=\argmin_{\pi \in \Pi}V(\pi;\hat \ell_{0:k-1})$ (via dynamic programming).
\For{$h=1,2,\ldots,H$}
\State Observe $s_k^h$, play $a_k^h=\pi_k(s_k^h)$, suffer and observe loss $\ell_k^h(s_k^h,a_k^h)$.
\State Calculate loss estimator $\hat \ell_k^h$ via Geometric Re-sampling \citep{neu2016importance}:
\For{$M_k^h=1,2,\ldots,L$}
    \State Sample a fresh perturbation $\tilde z$ in the same way as $z$.
    \State Calculate $\pi_k'=\argmin_{\pi \in \Pi}V(\pi;\hat \ell_{1:k-1}+\tilde z)$. 
    \State Simulate $\pi_k'$ for $h$ steps starting from $s^1$ and following transitions $\mathbb P^1, \ldots, \mathbb P^h$.
    \If{$(s_k^h,a_k^h)$ is visited at step $h$ or $M_k^h=L$}
        \State{Set $\hat \ell_k^h(s,a)=M_k^h\cdot \ell_k^h(s_k^h,a_k^h)\cdot \mathbbm 1[(s_k^h,a_k^h)=(s,a)]$ and break.}
    \EndIf
\EndFor
\EndFor
\EndFor
\end{algorithmic}
\end{algorithm}

Then we give the proof of \Cref{thm:regret of episodic AMDPs known}. As sketched in the main text, we define the following probability, as-if we are resampling a purturbation $z$ for each round:
\begin{equation*}
    p_k(\pi)=\Pr\nolimits_z\{\pi_k=\pi\mid \hat \ell_1,\hat \ell_2,\ldots,\hat \ell_{k-1}\}.
\end{equation*}

Note that, as mentioned in \cite[Appendix A.2]{syrgkanis2016efficient}, $p_k$ is just the probability of picking $\pi$ at episode $k$ given all history from episodes $1,2,\ldots,k-1$. Now, we decompose our regret $\mathcal R_K$ into the following three terms:
\begin{align*}
    \mathcal R_K&=\E\left[ \sum_{k=1}^K\sum_{h=1}^H\sum_{\pi \in \Pi}p_k(\pi) \langle \mu_\pi^h, \ell_k^h\rangle - \sum_{k=1}^K \sum_{h=1}^H \langle \mu_{\pi^\ast}^h, \ell_k^h\rangle\right] \\    
    &=\underbrace{\E\left [\sum_{k=1}^K\sum_{h=1}^H \langle \mu_{\pi_k}^h,\ell_k^h-\hat \ell_k^h\rangle+\sum_{k=1}^K \sum_{h=1}^H \langle \mu_{\pi^\ast}^h,\hat \ell_k^h-\ell_k^h\rangle\right ]}_{\text{GR error term}}\\
    &\quad \underbrace{\E\left [\sum_{k=1}^K \sum_{h=1}^H \sum_{\pi \in \Pi}p_{k+1}(\pi)\langle \mu_\pi^h,\hat \ell_k^h\rangle-\sum_{k=1}^K \sum_{h=1}^H \langle \mu_{\pi^\ast}^h,\hat \ell_k^h\rangle\right ]}_{\text{Error term}}+\\
    &\quad \underbrace{\E\left [\sum_{k=1}^K \sum_{h=1}^H \sum_{\pi \in \Pi}(p_k(\pi)-p_{k+1}(\pi))\langle \mu_\pi^h,\hat \ell_k^h\rangle\right ]}_{\text{Stability term}}.
\end{align*}

\subsubsection{Bouding the GR Error Term}
\begin{lemma}[Bounding GR Error Term]\label{lem:appendix known GR error term}
The GR error term is bounded by
\begin{equation*}
    \E\left [\sum_{k=1}^K\sum_{h=1}^H \langle \mu_{\pi_k}^h,\ell_k^h-\hat \ell_k^h\rangle+\sum_{k=1}^K \sum_{h=1}^H \langle \mu_{\pi^\ast}^h,\hat \ell_k^h-\ell_k^h\rangle\right ]\le \frac{SAHK}{eL}.
\end{equation*}
\end{lemma}
\begin{proof}
First notice that, from \Cref{lem:expectation of GR}, $\E[\hat \ell_k^h(s,a)\mid \mathcal F_{k-1}]\le \ell_k^h(s,a)$. Moreover, as $\pi^\ast$ is deterministic (i.e., it does not depend on the randomness from the algorithm), the second term
\begin{equation*}
    \E\left [\sum_{k=1}^K \sum_{h=1}^H \langle \mu_{\pi^\ast}^h,\hat \ell_k^h-\ell_k^h\rangle\right ]=\E\left[\sum_{k=1}^K \sum_{h=1}^H \langle \mu_{\pi^\ast}^h,\E[\hat \ell_k^h\mid \mathcal F_{k-1}]-\ell_k^h\rangle \right]\le 0.
\end{equation*}

For the first term, again by \Cref{lem:expectation of GR}, we have
\begin{equation*}
    \E\left [\sum_{k=1}^K\sum_{h=1}^H \langle \mu_{\pi_k}^h,\ell_k^h-\hat \ell_k^h\rangle\right ]=\sum_{k=1}^K\sum_{h=1}^H \sum_{(s,a)\in \mS\times \mA}\E\left [\mu_{\pi_k}^h(s,a)\cdot (1-q_k^h(s,a))^L \ell_k^h(s,a)\right ],
\end{equation*}

where $q_k^h(s,a)$ is the probability of visiting $(s,a)$ in a single trial of the Geometric Re-sampling process, which is just (note that $q_k^h$ itself is also a random variable as $p_k$ is non-deterministic)
\begin{equation*}
    q_k^h(s,a)= \E[\mu_{\pi_k}^h(s,a)]=\sum_{\pi \in \Pi}p_k(\pi)\mu_\pi^h(s,a)
\end{equation*}

in our case. By noticing that $q(1-q)^L\le qe^{-Lq}\le \frac{1}{eL}$ for all $q\ge 0$ \citep{neu2016importance}, we have
\begin{equation*}
    \E\left [\sum_{k=1}^K\sum_{h=1}^H \langle \mu_{\pi_k}^h,\ell_k^h-\hat \ell_k^h\rangle\right ]\le HKSA\frac{1}{eL}=\frac{SAHK}{eL},
\end{equation*}

as claimed.
\end{proof}

\subsubsection{Bounding the Error Term}
\begin{lemma}[Bounding Error Term]\label{lem:appendix known error term}
The error term is bounded by
\begin{equation*}
    \E\left [\sum_{k=1}^K \sum_{h=1}^H \sum_{\pi \in \Pi}p_{k+1}(\pi)\langle \mu_\pi^h,\hat \ell_k^h\rangle\right ]-\sum_{k=1}^K \sum_{h=1}^H \langle \mu_{\pi^\ast}^h,\hat \ell_k^h\rangle\le \frac{2H}{\eta} (1+\ln (SA)).
\end{equation*}
\end{lemma}
\begin{proof}
The proof uses the standard ``be-the-leader'' technique. For simplicity, we rewrite the error term as
\begin{equation*}
    \E\left [\sum_{k=1}^K V(\pi_{k+1};\hat \ell_k,\mathbb P)-\sum_{k=1}^K V(\pi^\ast;\hat \ell_k,\mathbb P)\right ].
\end{equation*}

Now consider the summation inside the expectation. If we add an extra term $V(\pi_1;\hat \ell_0,\mathbb P)-V(\pi^\ast;\hat \ell_0,\mathbb P)$ where $\hat \ell_0=z$ is the perturbation, we will have
\begin{align*}
    &\quad \sum_{k=0}^K V(\pi_{k+1};\hat \ell_k,\mathbb P)-V(\pi^\ast;\hat \ell_{0:K},\mathbb P)\\
    &\overset{(a)}{\le} \sum_{k=0}^K V(\pi_{k+1};\hat \ell_k,\mathbb P)-V(\pi_{K+1};\hat \ell_{0:K},\mathbb P)=\sum_{k=0}^{K-1} V(\pi_{k+1};\hat \ell_k,\mathbb P)-V(\pi_{K+1};\hat \ell_{0:K-1},\mathbb P)\\
    &\overset{(b)}{\le} \sum_{k=0}^{K-1} V(\pi_{k+1};\hat \ell_k,\mathbb P)-V(\pi_{K};\hat \ell_{0:K-1},\mathbb P)=\sum_{k=0}^{K-2} V(\pi_{k+1};\hat \ell_k,\mathbb P)-V(\pi_{K};\hat \ell_{0:K-2},\mathbb P)\\
    &\le \cdots \le V(\pi_1;\hat \ell_0,\mathbb P)-V(\pi_2;\hat \ell_0,\mathbb P)\overset{(c)}{\le} 0,
\end{align*}

where (a) used the optimality of $\pi_{K+1}$ w.r.t. $\hat \ell_{0:K}$, (b) used the optimality of $\pi_{K}$ w.r.t. $\hat \ell_{0:K-1}$ and so on, until the last step (c) where the optimality of $\pi_1$ w.r.t. $\hat \ell_0$ is used. So we have
\begin{equation*}
    \E\left [\sum_{k=1}^K V(\pi_{k+1};\hat \ell_k,\mathbb P)-\sum_{k=1}^K V(\pi^\ast;\hat \ell_k,\mathbb P)\right ]\le \E[V(\pi^\ast;\hat \ell_0,\mathbb P)-V(\pi_1;\hat \ell_0,\mathbb P)].
\end{equation*}

By the notation of occupancy measures, we can rewrite it as
\begin{equation*}
    \E\left [\sum_{h=1}^H\langle \mu_{\pi^\ast}^h,\hat \ell_0^h\rangle-\sum_{h=1}^H \langle \mu_{\pi_1}^h,\hat \ell_0^h\rangle\right ]\le 2\sum_{h=1}^H \E[\lVert \hat \ell_0^h\rVert_\infty].
\end{equation*}

Recall that $\ell_0^h(s,a)\sim \text{Laplace}(\eta)$, so we have
\begin{equation*}
    \E[\lvert \hat \ell_0^h\rVert_\infty]=\E\left [\max_{s,a} \lvert \ell_0^h(s,a)\rvert\right ]\le \frac{1+\ln(SA)}{\eta},
\end{equation*}

where the last step is due to the fact that $\lvert \ell_0^h(s,a)\rvert$ is an exponential distribution and \Cref{lem:error term wang}.
\end{proof}

\subsubsection{Bounding the Stability Term}
For the stability term, we first prove the following ``single-step stability'' lemma that we stated without proof in the main body.
\begin{lemma}[Single-Step Stability]\label{lem:appendix known single-step stability}
For all $k\in [K]$ and $(s,a)\in \mS\times \mA$,
\begin{equation*}
    p_{k+1}(\pi)\ge p_k(\pi)\exp\left (-\eta \sum_{h=1}^H \lVert \hat \ell_k^h\rVert_1\right ),\quad \forall \pi \in \Pi.
\end{equation*}
\end{lemma}
\begin{proof}
For simplicity, we use $\pi=\text{best}(\ell)$ to denote $\pi=\argmin_{\pi \in \Pi}V(\pi;\ell,\mathbb P)$. Then we have
\begin{align*}
    p_{k}(\pi)&=\int_{z}\mathbbm 1\left [\pi=\text{best}\left (\hat \ell_{1:k-1}+z\right )\right ]f\left (z\right )~\mathrm{d}z\\
    &=\int_z \mathbbm 1\left [\pi=\text{best}\left (\hat \ell_{1:k-1}+\left (z+\hat \ell_k\right )\right )\right ]f\left (z+\hat \ell_k\right )~\mathrm{d}z\\
    &=\int_z \mathbbm 1\left [\pi=\text{best}\left (\hat \ell_{1:k}+z\right )\right ]f\left (z+\hat \ell_k\right )~\mathrm{d}z,
\end{align*}

where $f(z)$ is the probability density function of $z$ and the second step made use of the fact that $z+\hat \ell_k$ is still linear in $z$. Moreover, 
\begin{equation*}
    p_{k+1}(\pi)=\int_{z}\mathbbm 1\left [\pi=\text{best}\left (\hat \ell_{1:k}+z\right )\right ]f\left (z\right )~\mathrm{d}z.
\end{equation*}

Recall that the definition of $f(z)$ is just $f(z)=\prod_{h=1}^H \sum_{s,a} \exp(-\eta \lvert z^h(s,a)\rvert)=\prod_{h=1}^H \exp(-\eta \lVert z^h\rVert_1)$ as each entry of $z$ is i.i.d. We thus have
\begin{equation*}
    f\left (z+\hat \ell_k\right )=\prod_{h=1}^H\exp\left (-\eta \left (\lVert z^h+\hat \ell_k\rVert_1-\lVert z^h\rVert_1\right )\right )f(z),
\end{equation*}

which gives
\begin{equation*}
    \frac{f\left (z+\hat \ell_k\right )}{f(z)}\in \left [\exp\left (-\eta \sum_{h=1}^H 
    \lVert \hat \ell_k^h\rVert_1\right ),\exp\left (\eta \sum_{h=1}^H 
    \lVert \hat \ell_k^h\rVert_1\right )\right ]
\end{equation*}

by triangle inequality. Therefore, $\nicefrac{p_{k+1}(\pi)}{p_k(\pi)}$ lies in this interval as well, which is just our claim.
\end{proof}

\begin{lemma}[Bounding Stability Term]\label{lem:appendix known stability term}
The stability term is bounded by
\begin{equation*}
    \E\left [\sum_{k=1}^K \sum_{h=1}^H \sum_{\pi \in \Pi}(p_k(\pi)-p_{k+1}(\pi))\langle \mu_\pi^h,\hat \ell_k^h\rangle\right ]\le 3\eta H^2SAK.
\end{equation*}
\end{lemma}
\begin{proof}
By summing up \Cref{lem:appendix known single-step stability} for all $\pi\in \Pi$ and using the fact that $1-\exp(-x)\le x$, we have
\begin{equation}\label{eq:stability RHS}
    \sum_{\pi\in \Pi}(p_k(\pi)-p_{k+1}(\pi))\sum_{h=1}^H\langle \mu_\pi^h,\hat \ell_k^h\rangle\le \eta \sum_{h'=1}^H\lVert \hat \ell_k^h\rVert_1\cdot \sum_{\pi \in \Pi} p_k(\pi)\sum_{h=1}^H\langle \mu_\pi^h,\hat \ell_k^h\rangle,\quad \forall k\in [K].
\end{equation}

%Further notice that $\hat \ell_k^h$ is one-hot at the visited state-action pair $(s_k^h,a_k^h)$ (i.e., has only one non-zero entry) as it is yielded via Geometric Re-sampling, we can rewrite the right-handed-side of the previous inequality as
%\begin{equation}\label{eq:stability RHS}
%    \eta\sum_{h'=1}^H\hat \ell_k^{h'}(s_k^{h'},a_k^{h'})\cdot \sum_{h=1}^H\sum_{\pi \in \Pi} p_k(\pi)\mu_\pi^h(s_k^h,a_k^h)\hat \ell_k^h(s_k^h,a_k^h).
%\end{equation}

To proceed, we need to investigate the Geometric Re-sampling process. Consider the random variable $M_k^h$ whose value is determined in the last line of \Cref{alg:finite-horizon bandit feedback known transition}. One may view it as a ``truncated'' geometric random variable, where $\text{Geo}(q)$ is a geometric random variable with parameter $q$, i.e., $\Pr\{\text{Geo}(q)=n\}=(1-q)^{n-1}q$. Formally, we have:
\begin{equation}\label{eq:equation of M_k^h}
    M_k^h=\min\{\text{Geo}(q_k^h(s_k^h,a_k^h)),L\},\quad \text{where }q_k^h(s,a)=\E_{\pi \sim p_k}[\mu_\pi^h(s,a)]=\sum_{\pi \in \Pi} p_k(\pi) \mu_\pi^h(s,a).
\end{equation}

So if we calculate the expectation of $\hat \ell_k^h(s,a)$ only with respect to $M_k^h$, we will have
\begin{equation*}
    \E\left [\hat \ell_k^h(s,a)\middle \vert (s_k^h,a_k^h)=(s,a)\right ]\le \frac{\ell_k^h(s,a)}{q_k^h(s,a)}.
\end{equation*}

Let $\mathbbm 1_k^h(s,a)$ be the shorthand notation of $\mathbbm 1[(s_k^h,a_k^h)=(s,a)]$. Then for those $h'\ne h$ in the RHS of \Cref{eq:stability RHS}, we have
\begin{align*}
    &\quad \eta \E\left [\sum_{h=1}^H\sum_{s,a} \sum_{\pi\in \Pi}p_k(\pi)\mu_\pi^h(s,a)\hat \ell_k^h(s,a) \sum_{h'\ne h}\|\hat \ell_k^{h'}\|_1 \middle \vert \mathcal F_{k-1}\right ]\\
    &\overset{(a)}{\le} \eta \E\left [\sum_{h=1}^H \sum_{s,a} \mathbbm 1_k^h(s,a) \ell_k^h(s,a)\frac{\sum_{\pi \in \Pi}p_k(\pi)\mu_\pi^h(s,a)}{q_k^h(s,a)}\sum_{h'\ne h}\|\hat \ell_k^{h'}\|_1\middle \vert \mathcal F_{k-1}\right ]\\
    &\overset{(b)}{\le} \eta H \E\left [\sum_{h'\ne h}\|\hat \ell_k^{h'}\|_1\middle \vert \mathcal F_{k-1}\right ]\overset{(c)}{\le} \eta H^2SA.
\end{align*}

where (a) is taking expectation w.r.t. $M_k^h$, (b) used the definition of $q_k^h$ together with the fact that $\sum_{(s,a)}\mathbbm 1_k^h(s,a)=1$, and (c) used the fact that $\E[\hat \ell_k^{h'}(s',a')\mid \mathcal F_{k-1}]\le \ell_k^{h'}(s',a')\le 1$ (\Cref{lem:expectation of GR}).

For those terms with $h=h'$ in \Cref{eq:stability RHS}, by direct calculation and the fact that $\hat\ell_k^{h}$ is a one-hot vector, we can write them as
\begin{equation*}
    \eta \E\left [\sum_{h=1}^H\sum_{s,a}\sum_{\pi \in \Pi}p_k(\pi)\mu_\pi^h(s,a)\left (\hat \ell_k^h(s,a)\right )^2\middle\vert \mathcal F_{k-1}\right ]\le 2\eta \E\left [\sum_{h,s,a}\frac{q_k^h(s,a)}{q_k^h(s,a)}\middle \vert \mathcal F_{k-1}\right ]\le 2\eta HSA
\end{equation*}
where we use $\E[(\hat \ell_k^h(s,a))^2\mid \mathcal F_{k-1}]\le 2(q_k^h(s,a))^{-1}$ (\Cref{lem:variance of GR}). Combining the terms with $h'\ne h$ and the ones with $h'=h$ gives our conclusion.
\end{proof}

\subsubsection{Proof of \Cref{thm:regret of episodic AMDPs known}}
\begin{proof}[Proof of \Cref{thm:regret of episodic AMDPs known}]
From \Cref{lem:appendix known GR error term,lem:appendix known error term,lem:appendix known stability term}, we have
\begin{equation*}
    \mathcal R_K\le \frac{SAHK}{eL}+\frac{2H}{\eta}(1+\ln (SA))+3\eta H^2SAK.
\end{equation*}

Therefore, if we pick $\eta^{-1}=\sqrt{HSAK}$ and $L=\sqrt{SAK/H}$,
\begin{equation*}
    \mathcal R_K\le \frac{H^{\nicefrac 32}\sqrt{SAK}}{e}+2H^{\nicefrac 32}\sqrt{SAK}(1+\ln(SA))+3H^{\nicefrac 32}\sqrt{SAK}=\Otil\left (H^{\nicefrac 32}\sqrt{SAK}\right ),
\end{equation*}

as desired.
\end{proof}

\subsubsection{Comparism with the \textsc{Context-FTPL} algorithm}\label{sec:comparism with Syrgkanis16}
One may think that our algorithm together with its analysis looks quite similar to the \textsc{Context-FTPL} algorithm \citep[Algorithm 2]{syrgkanis2016efficient} for adversarial contextual bandits. In fact, we can even convert the episodic AMDP problem with known transition as an instance of their contextual semi-bandit problem:
for time slot $(k,h)$, the ``context'' is $h$ and the loss vector is $\hat \ell_k^h$. A policy $\pi$ under context $x=h$ will then give an ``action'' $\pi(h) = \mu_\pi^h$ (the occupancy measure), which means it will suffer loss $\langle \mu_{\pi}^h,\hat \ell_k^h\rangle$. Both algorithms add perturbations to each of the contexts, $1,2,\ldots,H$, denoted by $z^1,z^2,\ldots,z^H\in \mathbb R^{SA}$ respectively.

However, there is a main differences between our setting and theirs: in their setting, the action space (where $\pi(x)$ belongs) is binary. However, in our case, $\mu_\pi^h\in [0,1]^{SA}$ is continuous. Though this difference may look tiny, it actually induces extra difficulties: this subtle difference will make their Lemma 10, stated as follows, no longer hold.
\begin{lemma}[{\citet[Lemma 10]{syrgkanis2016efficient}}]
For any contexts $x^1,x^2,\ldots,x^T$ and non-negative linear loss functions $\ell^1,\ell^2,\ldots,\ell^T$, suppose that $z^h(s,a)\sim \text{Laplace}(\eta)$, \textsc{Context-FTPL} satisfies
\begin{equation}
    \E\nolimits_z\left [\langle \pi^t(x^t),\ell^t\rangle-\langle \pi^{t+1}(x^t),\ell^t\rangle\right ]\le \eta \cdot \E[\langle \pi^t(x^t),\ell^t\rangle^2],\quad \forall 1\le t<T.\label{eq:lemma 10 of syrg}
\end{equation}
\end{lemma}

To see this, consider the simple case that there is only one possible value of the context together with two policies, each associated with action vectors $(0.1,0.1,0.2)$ and $(0.2,0.1,0.1)$, denoted by $\pi_1$ and $\pi_2$, respectively. Set the cumulative (perturbed) loss vector $\ell_{0:t-1}$ as $(0.75,0.2,0.6)$ and $\ell_{t}=(0.1,0,0)$ (this is set to be one-hot, so it can be yielded from our Geometric Re-sampling process). Set the Laplace distribution parameter $\eta=3$. Then, by direct calculation via integration, $p_t(\pi_1)=0.609453$ and $p_{t+1}(\pi_1)=0.675248$. As $\langle \pi_1,\ell^t\rangle=0.01$ and $\langle \pi_2,\ell^t\rangle=0.02$, the LHS of the \Cref{eq:lemma 10 of syrg} will be $0.00065795$ while the RHS will be $0.000651492$. Therefore, \Cref{eq:lemma 10 of syrg} simply \textit{does not} hold, even if there are only $2$ policies, $3$ dimensions and $1$ context.

Fortunately, as explained in the main text, though this strong version of ``single-step stability lemma'' does not hold, we are still able to prove a weaker version, \Cref{lem:single-step stability} (which is restated as \Cref{lem:appendix known single-step stability} in the appendix), to bound the stability term, which is worse only by a factor $H$, instead of $\lVert \hat \ell_t\rVert_\infty\le L$.

\subsection{Unknown Transition Case (\Cref{thm:regret of episodic AMDPs unknown})}\label{sec:appendix episodic unknown}
%!TEX root=main.tex

We first present our algorithm for the unknown transion case in \Cref{alg:appendix finite-horizon bandit feedback unknown transition}.

\begin{algorithm}[!t]
\caption{FTPL for Episodic AMDPs with Bandit Feedback and Unknown Transition}
\label{alg:appendix finite-horizon bandit feedback unknown transition}
\begin{algorithmic}[1]
\Require{Laplace distribution parameter $\eta$. Geometric Re-sampling parameter $L$.}
\State Initialize $\mathcal P_1\gets (\triangle(\mathcal S))^{[H]\times \mathcal S\times \mathcal A}$ (the set of all possible transition functions).
\State Sample perturbation $\hat\ell_0=z$ such that $z^h(s,a)$ is an independent sample of $\text{Laplace}(\eta)$.
\For{$k=1,2,\ldots,K$}
\State Let $(\pi_k,P_k)=\argmin_{(\pi,P) \in \Pi\times \mathcal P_k}V(\pi;\hat \ell_{1:k-1}+z,P)$ by Extended Value Iteration \citep{jaksch2010near}. (See also \Cref{remark:EVI} for more details.)
\For{$h=1,2,\ldots,H$}
\State Observe $s_k^h$, play $a_k^h=\pi_k(s_k^h)$, suffer and observe loss $\ell_k^h(s_k^h,a_k^h)$.
% \State Calculate loss estimator $\hat \ell_k^h$ via Geometric Re-sampling \citep{neu2016importance}:
\For{$M_k^h=1,2,\ldots,L$}
    \State Sample a fresh perturbation $\tilde z$ in the same way as $z$.
    \State Calculate $(\pi_k',P_k')=\argmin_{(\pi,P) \in \Pi\times \mathcal P_k}V(\pi;\hat \ell_{1:k-1}+\tilde z, P)$.
    \State Pick the transition $\hat P_k'\in \mathcal P_k$ such that $\mu_\pi^h(s_k^h,a_k^h;\hat P_k')$ is maximized via the \textsc{Comp-UOB} procedure proposed by \citet{jin2020learning}.
    \State Simulate $\pi_k'$ for $h$ steps starting from $s^1$ and following transitions $(\hat P_{k}')^1, \ldots, (\hat P_{k}')^h$.
    \If{$(s_k^h,a_k^h)$ is visited at step $h$ or $M_k^h=L$}
        \State{Set $\hat \ell_k^h(s,a)=M_k^h\cdot \ell_k^h(s_k^h,a_k^h)\cdot \mathbbm 1[(s_k^h,a_k^h)=(s,a)]$ and break.}
    \EndIf
\EndFor
\EndFor
\State Calculate $\mathcal P_{k+1}$ according to \Cref{eq:definition of confidence set}.
\EndFor
\end{algorithmic}
\end{algorithm}

\subsubsection{Transitions' Confidence Set Construction}\label{sec:confidence set construction}
We first discuss our construction of transitions' confidence sets. As in \cite{jin2022near}, we maintain a confidence set of transitions $\mathcal P_k$ for each episode $k\in [K]$ as \Cref{eq:definition of confidence set}, where $\mathcal P_1=(\triangle(\mathcal S))^{[H]\times \mathcal S\times \mathcal A}$.

As mentioned in the main text, we also want to ensure that $\mathcal P_{k+1}\subseteq \mathcal P_k$. Instead of taking $\mathcal P_1\cap \mathcal P_2\cap \cdots \mathcal P_k$ when doing the optimization, we directly ensure $\mathcal P_{k+1}\subseteq \mathcal P_k$ when constructing the confidence sets, such that they are always \textit{shrinking}. This is to ensure a well-bounded error term, as we will illustrate in \Cref{lem:unknown error bound}.
\begin{align}
    &\mathcal P_{k+1}=\mathcal P_{k}\cap \left \{\hat P\colon [H]\times \mathcal S\times \mathcal A\to \triangle(\mathcal S)\middle \vert \left \lvert \hat P^h(s'\mid s,a)-\bar P_k^h(s'\mid s,a)\right \rvert \le \varepsilon_i^h(s'\mid s,a),\forall s,s'\in \mathcal S,a\in \mathcal A\right \},\label{eq:definition of confidence set}\\
    &\text{where }\varepsilon_k^h(s'\mid s,a)=4\sqrt{\frac{\bar P_k^h(s'\mid s,a)\ln (10HSAK/\delta)}{\max\{1,N_k^h(s,a)\}}}+10\frac{\ln(10HSAK/\delta)}{\max\{1,N_k^h(s,a)\}},\label{eq:confidence radius definition}\\
    &\text{and }\bar P_k^h(s'\mid s,a)=\frac{N_k^h(s'\mid s,a)}{N_k^h(s,a)},\nonumber\\
    &\quad N_k^h(s,a)=\sum_{k'=1}^k \mathbbm 1[(s_{k'}^h,a_{k'}^h)=(s,a)], N_k^h(s'\mid s,a)=\sum_{k'=1}^k \mathbbm 1[s_{k'}^{h+1}=s',(s_{k'}^h,a_{k'}^h)=(s,a)]. \nonumber
\end{align}

By the following lemma from \citet{jin2020learning}, we define $K$ good events, $\mathcal E_1,\mathcal E_2,\ldots,\mathcal E_K$, where $\mathcal E_k$ means $\mathbb P\in \mathcal P_k$. From the following lemma, we can conclude that $\Pr\{\mathcal E_1,\mathcal E_2,\ldots,\mathcal E_K\}\ge 1-4\delta$. For simplicity, we also denote $\mathcal E=\mathcal E_1\wedge \mathcal E_2\wedge \cdots \wedge \mathcal E_K$. Hence, $\Pr\{\mathcal E\}\ge 1-4\delta$ (in fact, we have $\mathcal E=\mathcal E_K$ as $\mathcal P_k\subseteq \mathcal P_{k-1}$).
\begin{lemma}[{\citep[Lemma 2]{jin2020learning}}]
With probability $1-4\delta$, we have $\mathbb P\in \mathcal P_k$ for all $k\in [K]$.
\end{lemma}
\begin{remark}
Note that the original definition is slightly different from ours, where there is no intersection operations taken with previous confidence sets. However, as long as $\mathbb P$ belongs to all the confidence sets, it clearly belongs to the intersection of them.
\end{remark}
\begin{remark}\label{remark:EVI}
Note that the Extended Value Iteration \citep{jaksch2010near} approach works as long as $\mathcal P_k$ has the form$\{P\mid P^h(s'\mid s,a)\in [L^h(s'\mid s,a),R^h(s'\mid s,a)]\}$, but does not require $[L^h(s'\mid s,a),R^h(s'\mid s,a)]$ to be centered exactly at $\bar P_k^h(s'\mid s,a)$ (which is indeed the case for our algorithm due to the intersection operations).
\end{remark}

\subsubsection{Regret Decomposition}
For the unknown-transition cases, we first do the following regret decomposition as \citet{jin2020learning}:
\begin{align*}
    \mathcal R_K&=\underbrace{\E\left [\sum_{k=1}^K \left (V(\pi_k;\ell_k,\mathbb P)-V(\pi_k;\ell_k,P_k)\right )\right ]}_{\textsc{Error}}+\underbrace{\E\left [\sum_{k=1}^K \left (V(\pi_k;\ell_k,P_k)-V(\pi_k;\hat \ell_k,P_k)\right )\right ]}_{\textsc{Bias1}}+\\&\quad \underbrace{\E\left [\sum_{k=1}^K \left (V(\pi_k;\hat \ell_k,P_k)-V(\pi^\ast;\hat \ell_k,\mathbb P)\right )\right ]}_{\textsc{EstReg}}+\underbrace{\E\left [\sum_{k=1}^K \left (V(\pi^\ast;\hat \ell_k,\mathbb P)-V(\pi^\ast;\ell_k,\mathbb P)\right )\right ]}_{\textsc{Bias2}}.
\end{align*}

Intuitively, the \textsc{Error} term is due to the transition estimation, \textsc{Bias1} and \textsc{Bias2} terms are due to loss estimation for $\pi_k$ and $\pi^\ast$, respectively, and \textsc{EstReg} is the regret of our FTPL algorithm on the estimated transitions $P_k$ and the estimated losses $\hat \ell_k$.

\subsubsection{Bounding the \textsc{EstReg} Term}
\begin{theorem}[Bounding \textsc{EstReg} Term]\label{lem:unknown EstReg}
The \textsc{EstReg} term is bounded by
\begin{equation*}
   \textsc{EstReg}= \E\left [\sum_{k=1}^K\left (V(\pi_k;\hat \ell_k,P_k)-V(\pi^\ast;\hat \ell_k,\mathbb P)\right )\right ]\le \frac{2H}{\eta} \left (1+\ln(SA)\right )+3\eta H^2SAK+8\delta KHL.
\end{equation*}
\end{theorem}
\begin{proof}
For the \textsc{EstReg} term, we will also decompose it into an error term (not to be confused with the \textsc{Error} term which occurs in the decomposition of $\mathcal R_K$; this error term appears in the decomposition of \textsc{EstReg} and is related to the `be-the-leader' lemma) and a stability term (as it is defined for the estimated losses, there is no GR error term anymore). However, here we should define our ``leader'' as
\begin{equation*}
    (\tilde \pi_{k+1},\tilde P_{k+1})=\argmin_{(\pi,P)\in \Pi\times {\color{red}\mathcal P_k}}V\left (\pi;\hat \ell_{0:k},P\right ).
\end{equation*}

Instead of directly using $(\pi_{k+1},P_{k+1})$ as the leader (as we did in the known transition case), we allow the transition $\tilde P_{k+1}$ selected from $\mathcal P_{k}\supseteq \mathcal P_{k+1}$. This is critical to ensure a low stability term, as we can only derive the ``single-step stability lemma'' (\Cref{lem:unknown single step stability} in this case) for two probability distributions sharing a same support ($\Pi\times \mathcal P_k$ here).

As an analog to the known transition case, we define $p_k(\pi,P)$ as the probability density function (with respect of the perturbation $z$) of $(\pi_k,P_k)$ conditioning on $\hat \ell_1,\hat \ell_2,\ldots,\hat \ell_{k-1}$. Note that as there are infinitely many transitions, we cannot directly write $\Pr_z$ as in the known-transition setting.

Moreover, as explained before, we allow $\tilde P_{k+1}$ to be picked from $\mathcal P_k$ instead of $\mathcal P_{k+1}$, so the probability of picking $(\tilde \pi_{k+1},\tilde P_{k+1})$ as $(\pi,P)$ is not simply $p_{k+1}(\pi,P)$. %(which is what we did in the known-transition case). 
Therefore, we have to define another notation representing the probability density of picking each $(\pi,P)$ as $(\tilde \pi_{k+1},\tilde P_{k+1})$, namely $\tilde p_{k+1}(\pi,P)$, which is the probability density of $(\tilde \pi_{k+1},\tilde P_{k+1})$ with respect to $z$, conditioning on $\hat \ell_1,\hat \ell_2,\ldots,\hat \ell_k$.

Hence, we can write
\begin{align*}
    \textsc{EstReg}&=\underbrace{\E\left [\sum_{k=1}^K\sum_{h=1}^H \langle \mu_{\tilde \pi_{k+1}}^h(\tilde P_{k+1})-\mu_{\pi^\ast}^h(\mathbb P),\hat \ell_k^h\rangle \right ]}_{\text{Error term}}+\\
    &\quad \underbrace{\E\left [\sum_{k=1}^K\sum_{h=1}^H \sum_{\pi \in \Pi}\int_{\mathcal P_k}(p_k(\pi,P)-\tilde p_{k+1}(\pi,P)) \langle \mu_{\pi}^h(P),\hat \ell_k^h\rangle\D P \right ]}_{\text{Stability Term}}.
\end{align*}

For the error term, we only need to verify that the ``be-the-leader argument'' that we used in \Cref{lem:appendix known error term} still holds. Fortunately, it turns out as long as $\mathcal P_1\supseteq \mathcal P_2\supseteq \cdots \supseteq \mathcal P_K\supseteq \{\mathbb P\}$, we can always conclude the following lemma, whose proof is presented later.
\begin{lemma}[Bounding Error Term]\label{lem:unknown error bound}
The error term in this case is bounded by
\begin{equation*}
    \E\left [\sum_{k=1}^K \sum_{h=1}^H \langle \mu_{\tilde \pi_{k+1}}^h(\tilde P_{k+1})-\mu_{\pi^\ast}^h(\mathbb P),\hat \ell_k^h\rangle\right ]
    \le \frac{2H}{\eta} (1+\ln(SA))+4\delta KHL.
\end{equation*}
\end{lemma}

For the stability term, we need a similar but different single-step stability bound, as
\begin{lemma}[Single Step Stability]\label{lem:unknown single step stability}
For all $k\in [K]$, $(s,a)\in \mS\times \mA$ and $(\pi,P) \in \Pi\times \mathcal P_k$,
\begin{equation*}
    %\tilde p_{k+1}(\pi,P)\le p_k(\pi,P)\exp\left (\eta \sum_{h=1}^H \lVert \hat \ell_k^h\rVert_1\right ),\quad 
    \tilde p_{k+1}(\pi,P)\ge p_k(\pi,P)\exp\left (-\eta \sum_{h=1}^H \lVert \hat \ell_k^h\rVert_1\right ).
\end{equation*}
\end{lemma}

With this lemma, our derivation for the stability term in known-transition cases (\Cref{lem:appendix known stability term}) also holds, except that we are using the upper occupancy measures in the Geometric Re-sampling process, instead of the actual occupancy measures. Technically, this means that the event $(s_k^h,a_k^h)=(s,a)$ will happen with a probability
\begin{equation}\label{eq:unknown hat q_k^h informal}
    \hat q_k^h(s,a)=\Pr\{(s_k^h,a_k^h)=(s,a)\mid \mathcal F_{k-1}\}=\sum_{\pi \in \Pi}p_k(\pi)\mu_\pi^h(s,a;\mathbb P),
\end{equation}

where $p_k(\pi)=\int_{\mathcal P_k}p_k(\pi,P)\D P$ is the marginal probability of picking $\pi$ for episode $k$ (with a slight abuse of notation). However, in each execution of the Geometric Re-sampling process, the probability of visiting $(s,a)$ is another probability
\begin{equation}\label{eq:unknown q_k^h informal}
q_k^h(s,a)=\sum_{\pi \in \Pi}p_k(\pi)\max_{P'\in \mathcal P_k}\mu_\pi^h(s,a;P')\ne \hat q_k^h(s,a),
\end{equation}

which means we cannot use \Cref{lem:expectation of GR,lem:variance of GR} anymore.

Fortunately, we are able to derive \Cref{lem:expectation of GR varying,lem:variance of GR varying} in such a case, which actually implies the previous two lemmas, given that the actual occupancy measure $\hat q_k^h(s,a)$ is bounded by the upper occupancy measure $q_k^h(s,a)$ (which is indeed this case as long as $\mathbb P\in \mathcal P_k$, i.e., event $\mathcal E_k$ holds). However, for the \textsc{Bias1} term (\Cref{lem:unknown BIAS1}), as we will see later, this inconsistency will indeed induce extra difficulties, leading to a $\Otil(H^2S\sqrt{AK})$ dominating term as in \citet{jin2020learning}.

The detailed proof of \Cref{lem:unknown stability bound} will be presented after the proof of this theorem.
\begin{lemma}[Bounding Stability Term]\label{lem:unknown stability bound}
The stability term in this case is bounded by
\begin{equation*}
    \E\left [\sum_{k=1}^K \sum_{h=1}^H \sum_{\pi \in \Pi}\int_{\mathcal P_k}(p_k(\pi,P)-\tilde p_{k+1}(\pi,P)) \langle \mu_{\pi}^h(P),\hat \ell_k^h\rangle\D P\right ]\le 3\eta H^2SAK+4\delta KHL.
\end{equation*}
\end{lemma}

Combining them together gives
\begin{equation*}
    \textsc{EstReg} \le \frac {2H}{\eta} \left (1+\ln(SA)\right )+3\eta H^2SAK+8\delta KHL,
\end{equation*}

as claimed.
\end{proof}

\begin{proof}[Proof of \Cref{lem:unknown error bound}]
The proof still follows the idea of \Cref{lem:appendix known error term}. We rewrite the error term as
\begin{align*}
    &\E\left [\left (\sum_{k=1}^K V(\tilde \pi_{k+1};\hat \ell_k,\tilde P_{k+1})-\sum_{k=1}^K V(\pi^\ast;\hat \ell_k,\mathbb P) \right )\mathbbm 1[\mathcal E]\right ]+\\
    &\E\left [\left (\sum_{k=1}^K \sum_{h=1}^H \langle \mu_{\tilde \pi_{k+1}}^h(\tilde P_{k+1})-\mu_{\pi^\ast}^h(\mathbb P),\hat \ell_k^h\rangle\right ) \mathbbm 1[\neg \mathcal E]\right ].
\end{align*}

For the second term, since by definition $0 \leq  \hat \ell_k^h(s,a)\le L$ for all $k,h,s,a$ and both $\mu_{\tilde \pi_{k+1}}^h(\tilde P_{k+1})$ and $\mu_{\pi^\ast}^h(\mathbb P)$ are probability distributions, we can bound it as
\begin{equation*}
    KHL\cdot \Pr\{\neg \mathcal E\}\le 4\delta KHL.
\end{equation*}

Now consider the summation inside the first expectation. If we add an extra term $V(\tilde \pi_1;\hat \ell_0,\tilde P_1)-V(\pi^\ast;\hat \ell_0,\mathbb P)$ where $\hat \ell_0=z$ is the perturbation. The following deduction holds under the event $\mathcal E$:
\begin{align*}
    &\quad \sum_{k=0}^K V(\tilde \pi_{k+1};\hat \ell_k,\tilde P_{k+1})-V(\pi^\ast;\hat \ell_{0:K},\mathbb P)\\
    &\overset{(a)}{\le} \sum_{k=0}^K V(\tilde \pi_{k+1};\hat \ell_k,\tilde P_{k+1})-V(\tilde \pi_{K+1};\hat \ell_{0:K},\tilde P_{K+1})=\sum_{k=0}^{K-1} V(\tilde \pi_{k+1};\hat \ell_k,\tilde P_{k+1})-V(\tilde \pi_{K+1};\hat \ell_{0:K-1},\tilde P_{K+1})\\
    &\overset{(b)}{\le} \sum_{k=0}^{K-1} V(\tilde \pi_{k+1};\hat \ell_k,\tilde P_{k+1})-V(\tilde \pi_{K};\hat \ell_{0:K-1},\tilde P_{K})=\sum_{k=0}^{K-2} V(\tilde \pi_{k+1};\hat \ell_k,\tilde P_{k+1})-V(\tilde \pi_{K};\hat \ell_{0:K-2},\tilde P_K)\\
    &\le \cdots \le V(\tilde \pi_1;\hat \ell_0,\tilde P_1)-V(\tilde \pi_2;\hat \ell_0,\tilde P_2)\overset{(c)}{\le} 0.
\end{align*}

Here, (a) used the optimality of $(\tilde \pi_{K+1},\tilde P_{K+1})$ over the set $\Pi\times \mathcal P_K$ w.r.t. losses $\hat \ell_{0:K}$, which is valid due to $\mathcal E$; (b) used the optimality of $(\tilde \pi_K,\tilde P_K)$ over the set $\Pi\times \mathcal P_{K-1}$ w.r.t. losses $\hat \ell_{0:K-1}$, which is again valid since $\tilde P_{K+1}\in \mathcal P_K\subseteq \mathcal P_{K-1}$; similarly (c) used the optimality of $(\tilde \pi_1,\tilde P_1)$ over $\Pi\times \mathcal P_0$, which again holds as $\tilde P_2\in \mathcal P_0$ (which is the set of all transitions). So we still have the following inequality as \Cref{lem:appendix known error term}:
\begin{equation*}
    \E\left [\left (\sum_{k=1}^K V(\pi_{k+1};\hat \ell_k,\tilde P_{k+1})-\sum_{k=1}^K V(\pi^\ast;\hat \ell_k,\mathbb P)\right )\mathbbm 1[\mathcal E]\right ]\le \E[V(\pi^\ast;\hat \ell_0,\mathbb P)-V(\tilde\pi_1;\hat \ell_0,\tilde P_1)].
\end{equation*}

By the notation of occupancy measures, we can rewrite the last term as
\begin{equation*}
    \E\left [\sum_{h=1}^H\langle \mu_{\pi^\ast}^h(\mathbb P),\hat \ell_0^h\rangle-\sum_{h=1}^H \langle \mu_{\tilde \pi_1}^h(\tilde P_1),\hat \ell_0^h\rangle\right ]\le 2\sum_{h=1}^H \E[\lVert \hat \ell_0^h\rVert_\infty],
\end{equation*}

which is again bounded by $\frac{2H} \eta (1+\ln (SA))$ due to \Cref{lem:error term wang}. Combining these two parts (with or without $\mathcal E$) together gives our conclusion.
\end{proof}

\begin{proof}[Proof of \Cref{lem:unknown single step stability}]
We follow the proof of \Cref{lem:appendix known single-step stability}. For a fixed episode $k\in [K]$, we consider any $(\pi,P)\in \Pi\times \mathcal P_k$. We use the notation $(\pi,P)=\text{best}(\ell;\mathcal P)$ to denote $(\pi,P)=\argmin_{(\pi,P) \in \Pi\times \mathcal P}V(\pi;\ell,P)$. Then we have
\begin{align*}
    p_{k}(\pi,P)&=\int_{z}\mathbbm 1\left [(\pi,P)=\text{best}\left (\hat \ell_{1:k-1}+z;\mathcal P_k\right )\right ]f\left (z\right )~\mathrm{d}z\\
    &=\int_z \mathbbm 1\left [(\pi,P)=\text{best}\left (\hat \ell_{1:k-1}+\left (z+\hat \ell_k\right );\mathcal P_k\right )\right ]f\left (z+\hat \ell_k\right )~\mathrm{d}z\\
    &=\int_z \mathbbm 1\left [(\pi,P)=\text{best}\left (\hat \ell_{1:k}+z;\mathcal P_k\right )\right ]f\left (z+\hat \ell_k\right )~\mathrm{d}z,
\end{align*}

where $f(z)$ is the probability density function of $z$ and the second step made use of the fact that $z+\hat \ell_k$ is still linear in $z$. Moreover, 
\begin{equation*}
    \tilde p_{k+1}(\pi,P)=\int_{z}\mathbbm 1\left [(\pi,P)=\text{best}\left (\hat \ell_{1:k}+z;\mathcal P_k\right )\right ]f\left (z\right )~\mathrm{d}z.
\end{equation*}

Again by the fact that $f(z)=\prod_{h=1}^H \exp(-\eta \lVert z^h\rVert_1)$, which we used in the proof of \Cref{lem:appendix known single-step stability}, we have
\begin{equation*}
    f\left (z+\hat \ell_k\right )=\prod_{h=1}^H\exp\left (-\eta \left (\lVert z^h+\hat \ell_k\rVert_1-\lVert z^h\rVert\right )\right )f(z),
\end{equation*}

which gives
\begin{equation*}
    \frac{f\left (z+\hat \ell_k\right )}{f(z)}\in \left [\exp\left (-\eta \sum_{h=1}^H 
    \lVert \hat \ell_k^h\rVert_1\right ),\exp\left (\eta \sum_{h=1}^H 
    \lVert \hat \ell_k^h\rVert_1\right )\right ]
\end{equation*}

by triangle inequality. Therefore, $\nicefrac{\tilde p_{k+1}(\pi,P)}{p_k(\pi,P)}$ lies in this interval as well, which is just our claim.
\end{proof}

\begin{proof}[Proof of \Cref{lem:unknown stability bound}]
Let us focus on a single episode, say $k\in [K]$. We should first make sure that $\hat q_k^h(s,a)\le q_k^h(s,a)$ (defined in \Cref{eq:unknown hat q_k^h informal,eq:unknown q_k^h informal}), which happens when $\mathbb P\in \mathcal P_k$, i.e., $\mathcal E_k$ holds. Therefore, we rewrite the $k$-th summand of the stability term as
\begin{align}
&\E\left [\sum_{h=1}^H\sum_{\pi \in \Pi}\int_{\mathcal P_k}(p_k(\pi,P)-\tilde p_{k+1}(\pi,P))\langle \mu_\pi^h(P),\hat \ell_k^h\rangle\D P\mathbbm 1[\mathcal E_k]\right ]+\nonumber\\
&\E\left [\sum_{h=1}^H\sum_{\pi \in \Pi}\int_{\mathcal P_k}(p_k(\pi,P)-\tilde p_{k+1}(\pi,P))\langle \mu_\pi^h(P),\hat \ell_k^h\rangle\D P\mathbbm 1[\neg \mathcal E_k]\right ]\label{eq:unknown stability term}.
\end{align}

For the second term, we will bound it trivially as $HL\Pr\{\neg \mathcal E_k\}\le 4\delta HL$ as $p_k,\tilde p_{k+1}\in \triangle(\Pi\times \mathcal P_k)$ and $\mu_\pi^h\in \triangle(\mS\times \mA)$. For the first term, we will do something similar to \Cref{lem:appendix known stability term}, as follows:

Summing up \Cref{lem:unknown single step stability} for all $(\pi,P)\in \Pi\times \mathcal P_k$ and using the fact that $\exp(-x)\ge (1-x)$ gives
\begin{align}
    &\quad \sum_{\pi\in \Pi}\int_{\mathcal P_k} (p_k(\pi,P)-\tilde p_{k+1}(\pi,P))\sum_{h=1}^H\langle \mu_\pi^h(P),\hat \ell_k^h\rangle \D P \nonumber\\
    &\le \eta \sum_{h'=1}^H\lVert \hat \ell_k^{h'}\rVert_1\cdot \sum_{\pi \in \Pi}\int_{\mathcal P_k} p_k(\pi,P)\sum_{h=1}^H\langle \mu_\pi^h(P),\hat \ell_k^h\rangle \D P.\label{eq:stability RHS unknown transition}
\end{align}

%which can still be rewritten as (due to the fact that $\hat \ell_k^h$ is still one-hot):
%\begin{equation}
%    \eta\sum_{h'=1}^H\hat \ell_k^{h'}(s_k^{h'},a_k^{h'})\cdot \sum_{h=1}^H\sum_{\pi \in \Pi}\int_{\mathcal P_k} p_k(\pi,P)\mu_\pi^h(s_k^h,a_k^h;P)\hat \ell_k^h(s_k^h,a_k^h)\D P.    
%\end{equation}

By considering the randomness of $M_k^h$, we will still have the following property, except for a different definition of $q_k^h$:
\begin{equation}
    \E\left [\hat \ell_k^h(s,a)\middle \vert (s_k^h,a_k^h)=(s,a)\right ]\le \frac{\ell_k^h(s,a)}{q_k^h(s,a)},\quad \text{where }q_k^h(s,a)=\sum_{\pi \in \Pi}p_k(\pi)\max_{P'\in \mathcal P_k}\mu_\pi^{h}(s,a;P'),\label{eq:unknown q_k^h}
\end{equation}

as when doing the Geometric Re-sampling process, we are picking the transition in $\mathcal P_k$ that maximizes the probability of reaching $(s,a)$. Still use $\mathbbm 1_k^h(s,a)$ as the shorthand notation of $\mathbbm 1[(s_k^h,a_k^h)=(s,a)]$. Then for any history $\mathcal F_{k-1}$ and those $h'\ne h$ in \Cref{eq:stability RHS unknown transition},
\begin{align*}
    &\quad \eta \E\left [\sum_{h=1}^H\sum_{s,a} \sum_{\pi \in \Pi}\int_{\mathcal P_k}p_k(\pi,P)\mu_\pi^h(s,a;P) \hat \ell_k^h(s,a)\D P \cdot \sum_{h'\ne h} \|\hat \ell_k^{h'}\|_1 \mathbbm 1[\mathcal E_k] \middle \vert \mathcal F_{k-1}\right ]\\
    &\overset{(a)}{\le} \eta \E\left [\sum_{h=1}^H \sum_{s,a} \mathbbm 1_k^h(s,a) \ell_k^h(s,a)\frac{\sum_{\pi \in \Pi}\int_{\mathcal P_k}p_k(\pi,P)\mu_\pi^h(s,a;P)\D P}{q_k^h(s,a)}\sum_{h'\ne h}\|\hat \ell_k^{h'}\|_1 \mathbbm 1[\mathcal E_k] \middle \vert \mathcal F_{k-1}\right ]\\
    &\overset{(b)}{\le} \eta H \E\left [\sum_{h'\ne h}\|\hat \ell_k^{h'}\|_1 \mathbbm 1[\mathcal E_k]\middle \vert \mathcal F_{k-1}\right ]\overset{(c)}{\le} \eta H\sum_{s,a} \E\left [\sum_{h=1}^H \frac{\hat q_k^h(s,a)}{q_k^h(s,a)} \mathbbm 1[\mathcal E_k] \middle \vert \mathcal F_{k-1}\right ]\overset{(d)}{\le}\eta H^2SA,
\end{align*}

where (a) is taking expectation w.r.t. $M_k^h$, (b) used the (new) definition of $q_k^h$ together with the fact that $\sum_{(s,a)}\mathbbm 1_k^h(s,a)=1$, (c) used \Cref{lem:expectation of GR varying} and (d) used $\hat q_k^h(s,a)\le q_k^h(s,a)$ (which is due to $\mathbbm 1[\mathcal E_k]$).

For those terms with $h'=h$ in \Cref{eq:stability RHS unknown transition}, by direct calculation and the fact that $\hat\ell_k^{h}$ is a one-hot vector, we can write them as
\begin{align*}
    &\quad \eta \E\left [\sum_{h=1}^H\sum_{s,a}\sum_{\pi \in \Pi}\int_{\mathcal P_k}p_k(\pi,P)\mu_\pi^h(s,a;P)\left (\hat \ell_k^h(s,a)\right )^2\D P\mathbbm 1[\mathcal E_k] \middle\vert \mathcal F_{k-1}\right ]\\
    &\le 2\eta \E\left [\sum_{h,s,a}\frac{\sum_{\pi \in \Pi}\int_{\mathcal P_k}p_k(\pi,P)\mu_\pi^h(s,a;P)\D P}{q_k^h(s,a)}\frac{\hat q_k^h(s,a)}{q_k^h(s,a)}\mathbbm 1[\mathcal E_k] \middle \vert \mathcal F_{k-1}\right ]\\
    &\le 2\eta \sum_{s,a}\E\left [\sum_{h=1}^H \frac{\hat q_k^h(s,a)}{q_k^h(s,a)}\mathbbm 1[\mathcal E_k]\middle \vert \mathcal F_{k-1}\right ]\le 2\eta HSA,
\end{align*}
by applying \Cref{lem:variance of GR varying} together with the fact that $\hat q_k^h(s,a)\le q_k^h(s,a)$ when $\mathcal E_k$ happens. Combining the terms with $h'\ne h$ and the ones with $h'=h$ gives
\begin{equation*}
\text{\Cref{eq:unknown stability term}}\le 3\eta H^2SA+4\delta HL.
\end{equation*}

Therefore, the stability term is bounded by $3\eta H^2SAK+4\delta KHL$, as claimed.
\end{proof}

\subsubsection{Bounding Other Terms}
The terms other than \textsc{EstReg} can be bounded similarly to \citet{jin2020learning}, as follows:
\begin{theorem}[Bounding \textsc{Error} Term]\label{lem:unknown ERROR}
The \textsc{Error} term is bounded by
\begin{equation*}
    \textsc{Error}=\E\left [\sum_{k=1}^K \left (V(\pi_k;\ell_k,\mathbb P)-V(\pi_k;\ell_k,P_k)\right )\right ]\le \Otil\left (H^2S\sqrt{AK}+\delta KH\right ).
\end{equation*}
\end{theorem}

\begin{theorem}[Bounding \textsc{Bias1} Term]\label{lem:unknown BIAS1}
The \textsc{Bias} term is bounded by
\begin{align*}
    \textsc{Bias1}
    &=\E\left [\sum_{k=1}^K \left (V(\pi_k;\ell_k,P_k)-V(\pi_k;\hat \ell_k,P_k)\right )\right ]\\
    &\le \Otil\left (\frac{HKSA}{L}+H^2S\sqrt{AK}+H^3S^3A+\delta KH\right ).
\end{align*}
\end{theorem}
\begin{remark}
This term looks quite similar to the GR error term (\Cref{lem:appendix known GR error term}).
However, they are in fact different as we will have some extra terms due to the UOB technique. In other words, we are having different probabilities when reaching $(s_k^h,a_k^h)$ and when doing Geometric Re-sampling (c.f. \Cref{lem:expectation of GR,lem:expectation of GR varying}).
Therefore, this term will be further decomposed into two parts, where the first one is due to bias of the GR estimator and the second one is due to the UOB technique and can be bounded similar to \citet[Lemma A.3]{jin2022near}. Check the proof below for more details.
\end{remark}

\begin{theorem}[Bounding \textsc{Bias2} Term]\label{lem:unknown BIAS2}
The \textsc{Bias2} term is bounded by
\begin{equation*}
    \textsc{Bias2}=\E\left [\sum_{k=1}^K \left (V(\pi^\ast;\hat \ell_k,\mathbb P)-V(\pi^\ast;\ell_k,\mathbb P)\right )\right ]=\Otil(\delta KHL).
\end{equation*}
\end{theorem}

\begin{proof}[Proof of \Cref{thm:regret of episodic AMDPs unknown}]
By combining \Cref{lem:unknown ERROR,lem:unknown BIAS1,lem:unknown EstReg,lem:unknown BIAS2} together, we will have
\begin{align*}
    \mathcal R_T&\le \Otil\left (H^2S\sqrt{AK}+\frac H\eta+\eta H^2SAK+\frac{HSAK}{L}+H^2S\sqrt{AK}+H^3S^3A+\delta KHL\right ).
\end{align*}

Picking $\eta=\left (\sqrt{HSAK}\right )^{-1}$, $L=\sqrt{SAK/H}$ and $\delta=1/K$ gives
\begin{equation*}
    \mathcal R_T\le \Otil\left (H^2S\sqrt{AK}+H^{\nicefrac 32}\sqrt{SAK}+H\sqrt K+H^3S^3A\right )=\Otil\left (H^2S\sqrt{AK}+H^3S^3A\right ),
\end{equation*}
which finishes the proof.
\end{proof}

\begin{proof}[Proof of \Cref{lem:unknown ERROR}]
We need the following key lemma from \citet{jin2020learning}:\footnote{The original paper has a slightly different notation as they assumed the states to be `layered', i.e., $\mS=\mS_1\cup \mS_2\cup \cdots \cup \mS_H$ such that the states in $\mS_h$ can only transit to $\mS_{h+1}$, $\forall 1\le h<H$. Therefore, their $S$ should be $H$ times larger than ours. They also used $T$ for our $K$, $L$ for our $H$ and $X$ for our $\mS$.}
\begin{lemma}[{\citet[Lemma 4]{jin2020learning}}]
Conditioning on $\mathcal E$, for any set of policies $\{\pi_k\in \Pi\}_{k\in [K]}$ and any collection of transitions $\{P_k^{s,h}\}_{s\in \mS,h\in [H]}$ such that $P_k^{s,h}\in \mathcal P_k$, with probability $1-2\delta$,
\begin{equation*}
    \sum_{k=1}^K \sum_{h=1}^H \sum_{(s,a)\in \mS\times \mA} \lvert \mu_{\pi_k}^h(s,a;P_k^{s,h})-\mu_{\pi_k}^h(s,a;\mathbb P)\rvert\le \Otil\left (H^2S\sqrt{AK}\right ).
\end{equation*}
\end{lemma}

As all losses are in $[0,1]$ (note that in the \textsc{Error} term we are considering true losses), we have
\begin{equation*}
    \sum_{k=1}^K \left (V(\pi_k;\ell_k,\mathbb P)-V(\pi_k;\ell_k,P_k)\right )\le \sum_{k=1}^K \sum_{h=1}^H \sum_{(s,a)\in \mS\times \mA} \lvert \mu_{\pi_k}^h(s,a;P_k)-\mu_{\pi_k}^h(s,a;\mathbb P)\rvert,
\end{equation*}

which is bounded by $\Otil(H^2S\sqrt{AK})$ with probability $1-2\delta$ by the previous lemma. Let the event (i.e., it is bounded by $\Otil(H^2S\sqrt{AK})$) be $\mathcal E'$. Then
\begin{equation*}
\Pr\{\mathcal E'\wedge \mathcal E\}=\Pr\{\mathcal E'\mid \mathcal E\}\Pr\{\mathcal E\}\ge 1-6\delta.
\end{equation*}

Therefore, we write
\begin{align*}
\E\left [\sum_{k=1}^K \left (V(\pi_k;\ell_k,\mathbb P)-V(\pi_k;\ell_k,P_k)\right )\right ]
&=\E\left [\sum_{k=1}^K \left (V(\pi_k;\ell_k,\mathbb P)-V(\pi_k;\ell_k,P_k)\right )\mathbbm 1[\mathcal E\wedge \mathcal E']\right ]+\\
&\quad \E\left [\sum_{k=1}^K \sum_{h=1}^H \langle \mu_{\pi_k}^h(\mathbb P)-\mu_{\pi_k}^h(P_k),\ell_k^h\rangle \mathbbm 1[\neg \mathcal E\vee \neg \mathcal E']\right ]\\
&=\Otil\left (H^2S\sqrt{AK}+\delta KH\right ),
\end{align*}

where the last step used the fact that $\mu_{\pi_k}^h(\mathbb P)$ and $\mu_{\pi_k}^h(P_k)$ are both probability distributions and $0 \leq \ell_k^h(s,a)\le 1$.
\end{proof}

\begin{proof}[Proof of \Cref{lem:unknown BIAS1}]
Write our \textsc{Bias1} term in terms of occupancy measures:
\begin{align*}
    \textsc{Bias1}&=\E\left [\sum_{k=1}^K \sum_{\pi \in \Pi}\int_{\mathcal P_k}p_k(\pi,P)\sum_{h=1}^H \langle \mu_\pi^h(P),\ell_k^h-\hat \ell_k^h\rangle\D P\right ].
\end{align*}

Consider the $k$-th summand of it, denoted as $\textsc{Bias1}_k$. We decompose it into two parts, depending on whether $\mathcal E_k$ holds:
\begin{align*}
    \textsc{Bias1}_k&\le \E\left [\sum_{\pi \in \Pi}\int_{\mathcal P_k}p_k(\pi,P)\sum_{h=1}^H \E\left [\langle \mu_\pi^h(P),\ell_k^h-\hat \ell_k^h\rangle\mid \mathcal F_{k-1}\right ]\D P\mathbbm 1[\mathcal E_k]\right ]+\\
    &\quad \E\left [\sum_{\pi \in \Pi}\int_{\mathcal P_k}p_k(\pi,P)\sum_{h=1}^H \langle \mu_\pi^h(P),\ell_k^h\rangle\D P\mathbbm 1[\neg \mathcal E_k]\right ]\\
    &\triangleq \textsc{Bias1}_k^{\mathcal E}+\textsc{Bias1}_k^{\neg \mathcal E}.
\end{align*}

For $\textsc{Bias1}_k^{\neg \mathcal E}$, we bound it trivially as $H\Pr\{\neg \mathcal E_k\}\le 4\delta H$ as $p_k\in \triangle(\pi\times \mathcal P_k)$, $\mu_\pi^h(P)\in \triangle(\mS\times \mA)$ and $\ell_k^h(s,a)\in [0,1]$. For $\textsc{Bias1}_k^{\mathcal E}$, we still adopt the notations of $\hat q_k^h(s,a)$ and $q_k^h(s,a)$, which are defined as
\begin{align*}
    \hat q_k^h(s,a)&=\sum_{\pi \in \Pi}\int_{\mathcal P_k}p_k(\pi,P)\mu_\pi^h(s,a;\mathbb P)\D P,\\
    q_k^h(s,a)&=\sum_{\pi \in \Pi}\int_{\mathcal P_k}p_k(\pi,P)\max_{P'\in \mathcal P_k}\mu_\pi^h(s,a;P')\D P.
\end{align*}

Applying \Cref{lem:expectation of GR varying} to $\E[\hat \ell_k^h(s,a)\mid \mathcal F_{k-1}]$, $\forall (s,a)\in \mS\times \mA$ then gives
\begin{align*}
    \textsc{Bias1}_k^{\mathcal E}&=\E\left [\sum_{\pi \in \Pi}\int_{\mathcal P_k}p_k(\pi,P)\sum_{h=1}^H \left \langle \mu_\pi^h(P),\left (1-\frac{\hat q_k^h}{q_k^h}+\frac{\hat q_k^h}{q_k^h}(1-q_k^h)^L\right ) \ell_k^h\right \rangle\D P \mathbbm 1[\mathcal E_k]\right ]
\end{align*}
(every operation for the second term of the inner product is element-wise).
As $\mathbbm 1[\mathcal E_k]$ implies $\hat q_k^h(s,a)\le q_k^h(s,a)$, we can further bound $\textsc{Bias1}_k^{\mathcal E}$ as
\begin{align*}
    \textsc{Bias1}_k^{\mathcal E}&\le \E\left [\sum_{h=1}^H  \sum_{(s,a)\in \mS\times \mA} \sum_{\pi \in \Pi}\int_{\mathcal P_k} p_k(\pi,P)\mu_\pi^h(s,a;P) \frac{q_k^h(s,a)-\hat q_k^h(s,a)}{q_k^h(s,a)}\ell_k^h(s,a)\D P\mathbbm 1[\mathcal E_k]\right ]+\\
    &\quad \E\left [\sum_{h=1}^H \sum_{(s,a)\in \mS\times \mA} \sum_{\pi \in \Pi}\int_{\mathcal P_k}p_k(\pi,P) \mu_\pi^h(s,a;P)(1-q_k^h(s,a))^L\D P\mathbbm 1[\mathcal E_k]\right ].
\end{align*}

For the second term, we can simply make use of the fact that
\begin{equation}
\mathbbm 1[\mathcal E_k]=1\Longrightarrow \sum_{\pi \in \Pi}\int_{\mathcal P_k}p_k(\pi,P)\mu_\pi^h(s,a;P)\D P\le q_k^h(s,a)\label{eq:fact in Bias1}
\end{equation}

together with the condition that $\ell_k^h(s,a)\in [0,1]$ and consequently bound it by
\begin{equation*}
    \E\left [\sum_{h=1}^H \sum_{(s,a)\in \mS\times \mA} q_k^h(s,a)(1-q_k^h(s,a))^L \mathbbm 1[\mathcal E_k]\right ]
    \overset{(c)}{\le} \frac{HSA}{eL},
\end{equation*}

where (c) used the fact that $q(1-q)^L\le qe^{-Lq}\le \frac{1}{eL}$, just as what we did in \Cref{lem:appendix known GR error term}. For the first term, with a slight abuse of notations, we still use $p_k(\pi)$ to denote the probability of playing $\pi$ at episode $k$, i.e., $p_k(\pi)=\int_{P\in \mathcal P_k}p_k(\pi,P)\D P$. Then again by \Cref{eq:fact in Bias1}, we are actually facing
\begin{align}
    &\quad \E\left [\sum_{h=1}^H  \sum_{(s,a)\in \mS\times \mA} \left (q_k^h(s,a)-\hat q_k^h(s,a)\right )\mathbbm 1[\mathcal E_k]\right ]\\
    &=\E\left [\sum_{\pi\in \Pi}p_k(\pi) \sum_{h=1}^H \sum_{(s,a)\in \mS\times \mA}\left (\max_{P'\in \mathcal \mathcal P_k}\mu_\pi^h(s,a;P')-\mu_\pi^h(s,a;\mathbb P)\right )\mathbbm 1[\mathcal E_k]\right ].\label{eq:unknown BIAS1 to bound}
\end{align}

Then we follow the idea of \citet[Lemma A.3]{jin2022near}. We fix the step $h\in [H]$ and the state-action pair $(s,a)\in \mS\times \mA$. Therefore, for each policy $\pi\in\Pi$, we can define $\hat P_{\pi}\in \mathcal P_k$ to be transition corresponding to the upper-occupancy bound, i.e., it maximizes $\mu_\pi^h(s,a;P)$ over all transitions $P\in \mathcal P_k$. Therefore, with the help of the so-called ``occupancy difference lemma'' \citep[Lemma D.3.1]{jin2021best}, we can write the summand in \Cref{eq:unknown BIAS1 to bound} corresponding to $k,h,s,a$ as
\begin{align*}
    &\quad q_k^h(s,a)-\hat q_k^h(s,a)=\sum_{\pi \in \Pi}p_k(\pi) \left (\mu_\pi^h(s,a;\hat P_{\pi,h,s,a})-\mu_\pi^h(s,a;\mathbb P)\right )\\
    &=\sum_{\pi \in \Pi}p_k(\pi)\sum_{h'=0}^{h-1}\sum_{x\in \mS,y\in \mA,z\in \mS}\mu_\pi^{h'}(x,y;\mathbb P)\left (\mathbb P^{h'}(z\mid x,y)-\hat P_\pi^{h'}(z\mid x,y)\right )\mu_{\pi}^{h\mid h'+1}(s,a\mid z;\hat P_\pi),
\end{align*}

where $\mu_\pi^{h\mid h'+1}(s,a\mid z;P)$ is the so-called ``conditional occupancy measure'', which is defined as the conditional probability of reaching the state-action pair $(s,a)$ at step $h$ from state $z$ at step $h'+1$ with policy $\pi$ and transition $P$. By $\mathcal E_k$, we have $\mathbb P\in \mathcal P_k$. Therefore, by the definition of confidence radii, we can further bound %(recall again that $\mu_\pi^{h\mid h'+1}$ is the conditional occupancy measure)
\begin{equation*}
    \lvert q_k^h(s,a)-\hat q_k^h(s,a)\rvert\le \sum_{\pi \in \Pi}p_k(\pi)\sum_{h'=0}^{h-1}\sum_{x,y,z}\mu_\pi^{h'}(x,y;\mathbb P)\epsilon_k^{h'}(z\mid x,y)\mu_{\pi}^{h\mid h'+1}(s,a\mid z;\hat P_\pi),
\end{equation*}

where $\epsilon_k^h$ is defined as in \Cref{eq:confidence radius definition}.

Then, we consider the conditional occupancy measure $\mu_{\pi}^{h\mid h'+1}$ w.r.t. $\hat P_\pi$. We can still use occupancy difference lemmas (but now we only consider steps between $h'+1$ and $h$) to write its difference with the conditional occupancy measure w.r.t. $\mathbb P$ as
\begin{align*}
    &\quad \mu_{\pi}^{h\mid h'+1}(s,a\mid z;\hat P_\pi)-\mu_{\pi}^{h\mid h'+1}(s,a\mid z;\mathbb P)\\
    &\le \sum_{h''=h'+1}^{h-1}\sum_{u\in \mS,v\in \mA,w\in \mS}\mu_{\pi}^{h''\mid h'+1}(u,v\mid z;\mathbb P)\epsilon_{k}^{h''}(w\mid u,v)\mu_{\pi}^{h\mid h''+1}(s,a\mid w;\hat P_\pi)\\
    &\le \sum_{h''=h'+1}^{h-1}\sum_{u\in \mS,v\in \mA,w\in \mS}\mu_{\pi}^{h''\mid h'+1}(u,v\mid z;\mathbb P)\epsilon_{k}^{h''}(w\mid u,v)\pi^h(a\mid s),
\end{align*}

where the first step follows from the same reasoning as the unconditioned ones and the second step,
\begin{equation*}
    \mu_{\pi}^{h\mid h''+1}(s,a\mid w;\hat P_\pi)=\pi^h(a\mid s)\Pr\{s^h=s\mid s^{h''+1}=w,\pi,\hat P_\pi\}\le \pi^h(a\mid s).
\end{equation*}

Hence, plugging back into \Cref{eq:unknown BIAS1 to bound} gives its bound as
\begin{align*}
    &\quad \E\left [\sum_{h,s,a}\sum_{\pi \in \Pi}p_k(\pi)\sum_{h'=0}^{h-1}\sum_{x,y,z}\mu_\pi^{h'}(x,y;\mathbb P)\epsilon_k^{h'}(z\mid x,y)\mu_{\pi}^{h\mid h'+1}(s,a\mid z;\hat P_\pi)\right ]\\
    &\le \sum_{h,s,a} \E\left [\sum_{\pi \in \Pi}p_k(\pi)\sum_{h'=0}^{h-1}\sum_{x,y,z}\mu_\pi^{h'}(x,y;\mathbb P)\epsilon_k^{h'}(z\mid x,y)\mu_{\pi}^{h\mid h'+1}(s,a\mid z;\mathbb P)\right ]+\\
    &\quad \sum_{h,s,a} \E\left [\sum_{\pi \in \Pi}p_k(\pi)\sum_{h'=0}^{h-1}\sum_{x,y,z}\mu_\pi^{h'}(x,y;\mathbb P)\epsilon_k^{h'}(z\mid x,y)
    \sum_{h''=h'+1}^{h-1}\sum_{u,v,w}\mu_{\pi}^{h''\mid h'+1}(u,v\mid z;\mathbb P)\epsilon_{k}^{h''}(w\mid u,v)\pi^h(a\mid s)\right ].
\end{align*}

The remaining part of the proof is exactly the same as that for Lemma A.3 of \citet{jin2022near}, which eventually shows,
\begin{equation}
    \sum_{k=1}^K\textsc{Bias1}_{k}^{\mathcal E}=\Otil\left (\frac{HSAK}{L}+H^2S\sqrt{AK}+H^3S^3A+\delta HK\right ).
\end{equation}

Combining the two parts together (with or without $\mathcal E_k$) gives,
\begin{equation*}
    \textsc{Bias1}=\Otil\left (\frac{HSAK}{L}+H^2S\sqrt{AK}+H^3S^3A+\delta HK\right ),
\end{equation*}

as claimed.
\end{proof}

\begin{proof}[Proof of \Cref{lem:unknown BIAS2}]
This proof is quite simple. We still decompose \textsc{Bias2} into two parts:
\begin{align*}
\textsc{Bias2}&=\E\left [\sum_{k=1}^K \left (V(\pi^\ast;\hat \ell_k,\mathbb P)-V(\pi^\ast;\ell_k,\mathbb P)\right )\right ]\\
&=\sum_{k=1}^K \E\left [\left (V(\pi^\ast;\hat \ell_k,\mathbb P)-V(\pi^\ast;\ell_k,\mathbb P)\right )\mathbbm 1[\mathcal E_k]\right ]+\sum_{k=1}^K \E\left [\left (V(\pi^\ast;\hat \ell_k,\mathbb P)-V(\pi^\ast;\ell_k,\mathbb P)\right )\mathbbm 1[\neg \mathcal E_k]\right ].
\end{align*}

For the first term, as $\mathbbm 1[\mathcal E_k]$ infers $\hat q_k^h(s,a)\le q_k^h(s,a)$, from \Cref{lem:expectation of GR varying}, we have
\begin{equation*}
    \E[\hat \ell_k(s,a)\mid \mathcal F_{k-1}]\le \ell_k(s,a),\quad \forall k\in [K], (s,a)\in \mS\times \mA.
\end{equation*}

Therefore, as both $\pi^\ast$ and $\mathbb P$ are deterministic, this term is upper bounded by $0$. For the second term, we trivially bound each of the summand by $HL\Pr\{\neg \mathcal E_k\}\le 4\delta HL$ as $\lvert \hat \ell_k^h(s,a)\rvert\le L$. Therefore, combining two terms together completes the proof.
\end{proof}

\section{Analysis of Episodic AMDP Algorithms with Delayed Feedback (\Cref{thm:regret of delayed AMDPs})}\label{sec:appendix delayed}
%!TEX root=main.tex

In this section, we consider episodic AMDPs with delayed bandit feedback and unknown transitions. The algorithm is presented in \Cref{alg:delayed}, which is very similar to \Cref{alg:appendix finite-horizon bandit feedback unknown transition} except for the part on handling delayed feedback, highlighted in violet.

\begin{algorithm}[!t]
\caption{FTPL for Episodic AMDPs with Delayed Bandit Feedback and Unknown Transition}
\label{alg:delayed}
\begin{algorithmic}[1]
\Require{Laplace distribution parameter $\eta$. Geometric Re-sampling parameter $L$.}
\State Initialize $\mathcal P_1\gets (\triangle(\mathcal S))^{[H]\times \mathcal S\times \mathcal A}$.
\State Sample perturbation $\hat\ell_0=z$ such that $z^h(s,a)$ is an independent sample of $\text{Laplace}(\eta)$.
\For{$k=1,2,\ldots,K$}
\State Let $\color{violet} (\pi_k,P_k)=\argmin_{(\pi,P) \in \Pi\times \mathcal P_k}V(\pi;\sum_{k'\in \Omega_k}\hat \ell_{k'}+z,P)$ by Extended Value Iteration \citep{jaksch2010near}, where $\color{violet} \Omega_k=\{k'\mid k'+d_{k'}<k\}$. (See also \Cref{remark:EVI} for more details.)
\For{$h=1,2,\ldots,H$}
\State Observe $s_k^h$, play $a_k^h=\pi_k(s_k^h)$, suffer loss $\ell_k^h(s_k^h,a_k^h)$.
\EndFor
\For{{\color{violet} All $k'<k$ such that $k'+d_{k'}=k$}}
\For{$h=1,2,\ldots,H$}
% \State Receive $\ell_{k'}^h(s_{k'}^h,a_{k'}^h)$,  calculate loss estimator $\hat \ell_{k'}^h$ via Geometric Re-sampling \citep{neu2016importance}:
\For{$M_{k'}^h=1,2,\ldots,L$}
    \State Sample a fresh perturbation $\tilde z$ in the same way as $z$.
    \State Calculate $(\pi_{k'}',P_{k'}')=\argmin_{(\pi,P) \in \Pi\times \mathcal P_{k'}}V(\pi;\sum_{j\in \Omega_{k'}}\hat \ell_{j}+\tilde z)$.
    \State Pick the transition $\hat P_{k'}'\in \mathcal P_{k'}$ such that $\mu_\pi^h(s_{k'}^h,a_{k'}^h;\hat P_{k'}')$ is maximized via the \textsc{Comp-UOB} procedure proposed by \citet{jin2020learning}.
    \State Simulate $\pi_{k'}'$ for $h$ steps starting from $s^1$ and following transitions $(\hat P_{k'}')^1, \ldots, (\hat P_{k'}')^h$.
    \If{$(s_{k'}^h,a_{k'}^h)$ is visited at step $h$ or $M_{k'}^h=L$}
        \State{Set $\hat \ell_{k'}^h(s,a)=M_{k'}^h\cdot \ell_{k'}^h(s_{k'}^h,a_{k'}^h)\cdot \mathbbm 1[(s_{k'}^h,a_{k'}^h)=(s,a)]$ and break.}
    \EndIf
\EndFor
\EndFor
\EndFor
\State Calculate $\mathcal P_{k+1}$ according to \Cref{eq:definition of confidence set}.
\EndFor
\end{algorithmic}
\end{algorithm}

\subsection{Regret Decomposition}
\begin{proof}[Proof of \Cref{thm:regret of delayed AMDPs}]
For this case, we still use the regret decomposition as \Cref{thm:regret of episodic AMDPs unknown}, as follows:
\begin{align*}
    \mathcal R_K&=\underbrace{\E\left [\sum_{k=1}^K \left (V(\pi_k;\ell_k,\mathbb P)-V(\pi_k;\ell_k,P_k)\right )\right ]}_{\textsc{Error}}+\underbrace{\E\left [\sum_{k=1}^K \left (V(\pi_k;\ell_k,P_k)-V(\pi_k;\hat \ell_k,P_k)\right )\right ]}_{\textsc{Bias1}}+\\&\quad \underbrace{\E\left [\sum_{k=1}^K \left (V(\pi_k;\hat \ell_k,P_k)-V(\pi^\ast;\hat \ell_k,\mathbb P)\right )\right ]}_{\textsc{EstReg}}+\underbrace{\E\left [\sum_{k=1}^K \left (V(\pi^\ast;\hat \ell_k,\mathbb P)-V(\pi^\ast;\ell_k,\mathbb P)\right )\right ]}_{\textsc{Bias2}}.
\end{align*}

Note that as delays will not affect transitions as well as the loss estimators (viewed in hindsight, i.e., the sequence $\{\hat \ell_k\}_{k\in [K]}$ will be the same as if there is no delays), so the \textsc{Error}, \textsc{Bias1} and \textsc{Bias2} can still be bounded by \Cref{lem:unknown ERROR,lem:unknown BIAS1,lem:unknown BIAS2}, respectively. The only difference occurs when bounding \textsc{EstReg}, which we show as follows.
\begin{lemma}[Bounding \textsc{EstReg} Term with Delayed Feedback]\label{lem:delayed EstReg}
The \textsc{EstReg} term is bounded by
\begin{equation*}
    \E\left [\sum_{k=1}^K \left (V(\pi_k;\hat \ell_k,P_k)-V(\pi^\ast;\hat \ell_k,\mathbb P)\right )\right ]\le \frac{2H}{\eta} \left (1+\ln(SA)\right )+5\eta H^2 SA K+\eta H^2SA\delay+12\delta KHL.
\end{equation*}
\end{lemma}

As mentioned in the main body, the key difference is that, we will compete a learner that is not only cheating but also stepping one episode further. However, as it is still using FTPL, we can still bound the stability term as in \Cref{lem:unknown stability bound}. Therefore, the proof is postponed to the end of this section.

Combining the bounds for the four terms together, we will have
\begin{equation*}
    \mathcal R_T\le \Otil\left (H^2SA\sqrt K+\frac H\eta+\eta H^2SA (K+\delay)+\frac{SAHK}{L}+\delta HKL\right ).
\end{equation*}

Therefore, picking $\eta=\left (\sqrt{HSA(K+\delay)}\right )$, $L=\sqrt{SAK/H}$ and $\delta=1/K$ gives
\begin{equation*}
    \mathcal R_T\le \Otil\left (H^2SA\sqrt K+H^{\nicefrac 32}\sqrt{SA\delay}\right ),
\end{equation*}

as claimed.
\end{proof}

% \subsubsection{Bounding \textsc{EstReg} Term with Delayed Feedback}
\begin{proof}[Proof of \Cref{lem:delayed EstReg}]
Slightly different from the main text, we now consider the following \textit{two} learners, where the first one is a ``cheating learner'' that does not suffer any delays, and the second one is a ``cheating leader'' that not only does not suffer any delays, but also looks one step further.
\begin{equation*}
    (\hat \pi_{k},\hat P_{k})\triangleq \argmin_{(\pi,P)\in \Pi\times \mathcal P_{k}}V(\pi;\hat \ell_{0:k-1},P),\quad (\tilde \pi_{k+1},\tilde P_{k+1})\triangleq \argmin_{(\pi,P)\in \Pi\times \mathcal P_{k}}V(\pi;\hat \ell_{0:k},P).
\end{equation*}

Note that both of them are defined w.r.t. transitions in $\mathcal P_k$ instead of the subset $\mathcal P_{k+1}$, which is the same as \Cref{sec:appendix episodic unknown}. We also define the following three density functions with respect to the perturbation $z$: $p_k(\pi,P)$ for $(\pi_k,P_k)$ conditioning on $\hat \ell_1,\hat \ell_2,\ldots,\hat \ell_{k-1}$, $\hat p_k(\pi,P)$ for $(\hat \pi_k,\hat P_k)$ conditioning on $\hat \ell_1,\hat \ell_2,\ldots,\hat \ell_{k-1}$ and $\tilde p_{k+1}(\pi,P)$ for $(\tilde \pi_{k+1},\tilde P_{k+1})$ conditioning on $\hat \ell_1,\hat \ell_2,\ldots,\hat \ell_k$.

The purpose of defining two learners is to decouple the effects from delays and the inherent FTPL regret. One can see that our $(\hat \pi_k,\hat P_k)$ is equivalent to $(\pi_k,P_k)$ in \Cref{sec:appendix episodic unknown} while $(\tilde \pi_{k+1},\tilde P_{k+1})$ remains the same. Therefore, the difference between $(\hat \pi_k,\hat P_k)$ and $(\tilde \pi_k,\tilde P_k)$ can be bounded exactly the same as \Cref{sec:appendix episodic unknown} and we only need to care about delays, i.e., the difference between $(\pi_k,P_k)$ and $(\hat \pi_k,\hat P_k)$.
Formally, we decompose the \textsc{EstReg} into three terms:
\begin{align*}
    &\quad \E\left [\sum_{k=1}^K \sum_{h=1}^H \langle \mu_{\pi_k}^h(P_k)-\mu_{\pi^\ast}^h(\mathbb P),\hat \ell_k^h\rangle\right ]=\underbrace{\E\left [\sum_{k=1}^K \sum_{h=1}^H \langle \mu_{\pi_k}^h(P_k)-\mu_{\hat \pi_k}^h(\hat P_k),\hat \ell_k^h\rangle\right ]}_{\text{Cheating regret}}+\\
    &\quad \underbrace{\E\left [\sum_{k=1}^K \sum_{h=1}^H \langle \mu_{\hat \pi_k}^h(\hat P_k)-\mu_{\tilde \pi_{k+1}}^h(\tilde P_{k+1}),\hat \ell_k^h\rangle\right ]}_{\text{Stability term}}+\underbrace{\E\left [\sum_{k=1}^K \sum_{h=1}^H \langle \mu_{\tilde \pi_{k+1}}^h(\tilde P_{k+1})-\mu_{\pi^\ast}^h(\mathbb P),\hat \ell_k^h\rangle\right ]}_{\text{Error term}}.
\end{align*}

Note that the error term and the stability term are exactly the same as \Cref{sec:appendix episodic unknown}, so we can directly make use of \Cref{lem:unknown error bound,lem:unknown stability bound} and bound them by $\frac{2H}\eta (1+\ln (SA))+4\delta KHL$ and $3\eta H^2 SAK+4\delta KHL$, respectively. Now consider the cheating regret. Similar to the stability term, we will have the following single-step stability lemma:
\begin{lemma}\label{lem:delayed single step stability}
For any $k\in [K]$, $(s,a)\in \mS\times \mA$ and $(\pi,P)\in \Pi\times \mathcal P_k$, we have
\begin{equation*}
    \hat p_k(\pi,P)\ge p_k(\pi,P)\exp\left (-\eta \sum_{k'\in \tilde \Omega_k}\sum_{h=1}^H \lVert \hat \ell_{k'}^h\rVert_1\right ),
\end{equation*}

where $\tilde \Omega_k\triangleq [k-1]\setminus \Omega_k=\{k'< k\mid k'+d_{k'}\ge k\}$, i.e., the first $k-1$ rounds excluding those where the feedback is available before round $k$.
\end{lemma}
\begin{proof}
Note that the proof of \Cref{lem:unknown single step stability} does not rely on the concrete choice of $\pi_k$ and $\tilde \pi_{k+1}$. Therefore, adopting the proof of \Cref{lem:unknown single step stability} with $\hat \pi_k$ as $\tilde \pi_{k+1}$ will complete our proof.
\end{proof}

With the help of \Cref{lem:delayed single step stability}, we can bound the cheating regret similar to the stability term. To see this, consider a fixed $k\in [K]$, we have
\begin{align*}
    &\quad \E\left [\sum_{\pi \in \Pi}\int_{\mathcal P_k}(p_k(\pi,P)-\hat p_k(\pi,P))\sum_{h=1}^H \langle \mu_\pi^h(P),\hat \ell_k^h\rangle\D P\right ]\\
    &=\E\left [\sum_{\pi \in \Pi}\int_{\mathcal P_k}(p_k(\pi,P)-\hat p_k(\pi,P))\sum_{h=1}^H \langle \mu_\pi^h(P),\hat \ell_k^h\rangle\D P\mathbbm 1[\mathcal E_k]\right ]+\\
    &\quad \E\left [\sum_{\pi \in \Pi}\int_{\mathcal P_k}(p_k(\pi,P)-\hat p_k(\pi,P))\sum_{h=1}^H \langle \mu_\pi^h(P),\hat \ell_k^h\rangle\D \mathbbm 1[\neg \mathcal E_k]\right ]\\
    &\le \E\left [\eta \sum_{k'\in \tilde \Omega_k}\sum_{h'=1}^H \lVert \hat \ell_{k'}^{h'}\rVert_1\cdot \sum_{\pi \in \Pi}\int_{\mathcal P_k}p_k(\pi,P)\sum_{h=1}^H \langle \mu_\pi^h(P),\hat \ell_k^h\rangle\D P\mathbbm 1[\mathcal E_k]\right ]+4\delta KHL\\
    &\le \eta \E\left [\sum_{h=1}^H \sum_{(k',h')\in \tilde \Omega_k\times [H]}\mathbbm 1[(k',h')\ne (k,h)]\lVert \hat \ell_{k'}^{h'}\rVert_1\sum_{(s,a)\in \mS\times \mA} \hat q_k^h(s,a)\hat \ell_k^h(s,a)\mathbbm 1[\mathcal E_k]\right ]+\\
    &\quad \eta \E\left [\sum_{(s,a)\in \mS\times \mA}\sum_{h=1}^H \hat q_k^h(s,a) \left (\hat \ell_k^h(s,a)\right )^2\mathbbm 1[\mathcal E_k]\right ]+4\delta KHL,
\end{align*}

where $\hat q_k^h(s,a)$ is the actual probability of reaching $(s,a)$ and $q_k^h(s,a)$ is the probability of reaching $(s,a)$ in a single Geometric Re-sampling trial, as defined in \Cref{eq:unknown hat q_k^h informal,eq:unknown q_k^h informal}. Note that $\mathbbm 1[\mathcal E_k]$ implies $\hat q_k^h(s,a)\le q_k^h(s,a)$.

For the second term, using \Cref{lem:variance of GR} and $\hat q_k^h(s,a)\le q_k^h(s,a)$ gives $2\eta SAH$. For the first one, taking expectation w.r.t. $M_k^h$ in $\hat \ell_k^h(s,a)=\mathbbm 1[(s_k^h,a_k^h)=(s,a)]\ell_k^h(s,a)M_k^h(s,a)$ and then w.r.t. $\lVert \hat \ell_{k'}^{h'}\rVert_1$ as in \Cref{lem:appendix known stability term,lem:unknown stability bound} gives $\eta H^2SA\lvert \tilde \Omega_k\rvert$. Further noticing that
\begin{equation*}
    \sum_{k=1}^K \lvert \tilde \Omega_k\rvert=\sum_{k=1}^K\sum_{k'=1}^{k-1}\mathbbm 1[k'+d_{k'}\ge k]=\sum_{k'=1}^{K-1}\sum_{k=k'+1}^K\mathbbm 1[k'+d_{k'}\ge k]=\sum_{k'=1}^{K-1}d_{k'}=\delay,
\end{equation*}

we have the cheating regret is bounded by
\begin{equation*}
    \E\left [\sum_{k=1}^K \sum_{h=1}^H \langle \mu_{\pi_k}^h(P_k)-\mu_{\hat \pi_k}^h(\hat P_k),\hat \ell_k^h\rangle\right ]\le \eta H^2 SA \delay+2\eta HSAK+4\delta KHL.
\end{equation*}

The \textsc{EstReg} term is then consequently bounded by
\begin{equation*}
    \textsc{EstReg}\le \frac{2H}{\eta} (1+\ln(SA))+3\eta H^2SAK+\eta H^2 SA \delay+2\eta HSAK+12\delta KHL,
\end{equation*}

which is at most $\frac{2H}{\eta} (1+\ln(SA))+5\eta H^2SAK+\eta H^2SA\delay+12\delta KHL$, as claimed.
\end{proof}

\section{Analysis of Infinite-horizon AMDP Algorithms}\label{sec:appendix infinite-horizon}
\subsection{FTPL-Based Efficient Algorithm (\Cref{thm:regret of infinite-horizon AMDPs})}\label{sec:appendix infinite-horizon FTPL}
%!TEX root=main.tex

In this section present our \Cref{alg:appendix infinite-horizon bandit feedback known transition} together with its analysis. As described in the main body, we will divide the time horizon $[T]$ into $J$ epochs and fix a policy $\pi_j$ for the $j$-th epoch, namely $\mathcal T_j=\{(j-1)H+1,(j-1)H+2,\ldots,jH\}$ where $H=\frac TJ$ is the length of each epoch (overloading the notation $H$ from the episodic setting since they have a similar meaning).

\subsubsection{Switching Procedure}
The most significant difference between infinite-horizon AMDPs and episodic AMDPs is that the agent will not be reset to $s^1$ at the beginning of an ``epoch''. To formalize our problem as a online linear optimization problem (i.e., the total loss represented as $\sum_{t=1}^T \langle \mu_{\pi_t}^t,\ell_t\rangle$), we have to ensure the distribution over all states is exactly $\mu_{\pi_j}^t$ for most $t\in\mathcal T_j$. Before presenting the switching procesure from \citet{chandrasekaran2021online}, we first restate the assumption together with several properties that they used. For the sake of completeness, we also include their proofs here.

\begin{assumption}[Existance of a Staying State, Restatement of \Cref{assump:staying state} and {\citet[Assumption 5.1]{chandrasekaran2021online}}]\label{assump:appendix staying state}
The MDP $\mathcal M$ has a state $s^\ast$ and an action $a^\ast$ such that $\mathbb P(s^\ast\mid s^\ast,a^\ast)=1$.
\end{assumption}
\begin{lemma}[{\citet[Lemma 5.2]{chandrasekaran2021online}}]\label{lem:policy between states}
For any two distinct states $s,s'\in \mS$, there exists a policy $\pi_{s,s'}$ and $l_{s,s'}\le 2D$ such that
\begin{equation*}
    \Pr\left \{T(s'\mid \pi_{s,s'},s)=l_{s,s'}\right \}\ge \frac{1}{4D}.
\end{equation*}
\end{lemma}
\begin{proof}
By definition of diameter (as in \Cref{def:communicating MDP}), there exists a policy $\pi_{s,s'}$ such that $\E[T(s'\mid \mathcal M,\pi_{s,s'},s)]\le D$. By Markov's inequality, this implies $\Pr\{T(s'\mid \mathcal M,\pi_{s,s'},s)]\le 2D\}\ge \frac 12$. By pigeonhole principle, there consequently exists $l_{s,s'}\le 2D$ such that $\Pr\{T(s'\mid \mathcal M, \pi_{s,s'},s)=l_{s,s'}\}\ge \frac 12\cdot \frac{1}{2D}=\frac{1}{4D}$.
\end{proof}

\begin{lemma}[{\citet[Theorem 5.3]{chandrasekaran2021online}}]\label{lem:from s* to s}
For an MDP that satisfies \Cref{assump:appendix staying state}, there exists $l^\ast\le 2D$ such that for all states $s'\ne s^\ast$, there exists policy $\pi_{s'}$ such that
\begin{equation*}
    \Pr\{T(s'\mid \mathcal M,\pi_{s'},s^\ast)=l^\ast\}\ge \frac{1}{4D}.
\end{equation*}

Furthermore, denote $p_{s'}$ as the probability above. Let $p^\ast=\min_{s\in \mS}p_s$. Then $p^\ast\ge \frac{1}{4D}$.
\end{lemma}
\begin{proof}
From the previous lemma, there exists an $l_{s'}\le 4D$ for all $s'\ne s^\ast$ such that there is a policy $\pi_{s^\ast,s'}$ hitting $s'$ from $s^\ast$ in time exactly $l_{s'}$ with probability at least $\frac{1}{4D}$. Let $l^\ast=\max_{s'\ne s^\ast}l_{s'}$ and $\pi_{s'}$ be the policy that first stays at $s^\ast$ for $(l^\ast-l_{s'})$ steps and then follows $\pi_{s'}$ for $l_{s'}$ steps suffices.
\end{proof}

Now we are able to present the switching procedure from \citet{chandrasekaran2021online}, as in \Cref{alg:policy switching}.

\begin{algorithm}[!t]
\caption{Policy Switching in Infinite-Horizon AMDP \citep{chandrasekaran2021online}}
\label{alg:policy switching}
\begin{algorithmic}[1]
\Require{Current state $s\in \mS$. Goal policy $\pi \in \Pi$. Current time $t$.}
\While{\textbf{true}}
\State Move to state $s^\ast$ using policy $\pi_{s,s^\ast}$ as defined by \Cref{lem:policy between states} and update $s,t$ concurrently.
\State Sample the target state $g\sim \mu_{\pi}^t$.
\State Use policy $\pi_g$ from \Cref{lem:from s* to s} to move $l^\ast$ steps from $s^\ast$ and update $s,t$ concurrently.
\If{$s=g$}
\State Sample a Bernoulli random variable $I\sim \text{Ber}(\frac{p^\ast}{p_g})$.
\If{$I=1$}
\State \textbf{return}
\EndIf
\EndIf \Comment{The \texttt{while} loop will repeat if $s\ne g$ or $I\ne 1$.}
\EndWhile
\end{algorithmic}
\end{algorithm}

\begin{theorem}[Correctness of \Cref{alg:policy switching}, {\citet[Lemma 5.6]{chandrasekaran2021online}}]\label{thm:correctness of switching}
Let the random variable denoting the time that \Cref{alg:policy switching} terminates be $t_{\text{switch}}$. Then for any state $s\in \mS$
\begin{equation*}
    \Pr\{s_t=s\mid t_{\text{switch}}=t\}=\mu_\pi^t(s),\quad \forall t\in [T].
\end{equation*}
\end{theorem}
\begin{proof}
The key idea is to write
\begin{equation*}
    \Pr\{s_t=s\mid t_{\texttt{switch}}=t\}=\frac{\Pr\{s_t=s,{g=s}, t_{\texttt{switch}}=t\}}{\Pr\{t_{\texttt{switch}}=t\}}
\end{equation*}

and then bound the numerator and denominator separately. For the denominator,
\begin{align*}
&\Pr\{t_{\texttt{switch}}=t\}\\
=&\sum_{s\in \mathcal S}\Pr\{s_t=s,g=s,s_{t-l^\ast}=s^\ast\}\times\Pr\{t_{\texttt{switch}}=t\mid s_t=s,g=s,s_{t-l^\ast}=s^\ast\}\\
=&\sum_{s\in \mathcal S}\Pr\{s_t=s\mid g=s,s_{t-l^\ast}=s^\ast\}\times \Pr\{g=s,s_{t-l^\ast}=s^\ast\}\times\\
&\quad \Pr\{t_{\texttt{switch}}=t\mid s_t=s,g=s,s_{t-l^\ast}=s^\ast\}\\
=&\sum_{s\in \mathcal S}p_s\times \Pr\{g=s,s_{t-l^\ast}=s^\ast\}\times \frac{p^\ast}{p_s}=
p^\ast\times \Pr\{s_{t-l^\ast}=s^\ast\},
\end{align*}

where the last step used definition of $p_s$ and $I$. For the numerator,
\begin{align*}
&\quad \Pr\{g=s,s_t=s,t_{\text{switch}}=t\}\\
&=\Pr\{g=s,s_t=s,{s_{t-l^\ast}=s^\ast},t_{\text{switch}}=t\}\\
&=\Pr\{t_{\text{switch}}=t,s_t=s \mid g=s, s_{t-l^\ast}=s^\ast\}\times\Pr\{g=s, s_{t-l^\ast}=s^\ast\}\\
&=\Pr\{s_t=s\mid g=s, s_{t-l^\ast}=s^\ast\}\times\Pr\{t_{\text{switch}}=t\mid s_t=s, g=s, s_{t-l^\ast}=s^\ast\}\times \\
&\quad \Pr\{ s_{t-l^\ast}=s^\ast\}\times \Pr\{g=s\mid s_{t-l^\ast}=s^\ast\}\\
&=p_s\times \frac{p^\ast}{p_s}\times \Pr\{s_{t-l^\ast}=s^\ast\}\times \mu_{\pi'}^t(s).
\end{align*}

Plugging them back gives our desired result.
\end{proof}

\begin{theorem}[Efficiency of \Cref{alg:policy switching}, {\citet[Lemma 5.7]{chandrasekaran2021online}}]\label{thm:efficiency of switching}
The expected time spent on \Cref{alg:policy switching} is bounded by $12D^2$ for each execution.
\end{theorem}
\begin{proof}
Every time we try to catch the policy from $s^\ast$, we succeed with probability $p^\ast\ge \frac{1}{4D}$. Thus, the expected number of times we try is $4D$ and each attempt takes $l^\ast\le 2D$ steps. Between each of these attempts, we move at most D steps in expectation to reach $s^\ast$ again. Thus, in total, we have
\begin{equation*}
    \E[t_{\text{switch}}-t_0]\le 4D(2D+D)\le 12D^2,
\end{equation*}

as claimed.
\end{proof}

\subsubsection{The Algorithm}\label{sec:infinite horizon algorithm}
With the help of \Cref{alg:policy switching}, we now present our algorithm, \Cref{alg:appendix infinite-horizon bandit feedback known transition}. As mentioned in the main text, another important difference due to the ``non-resetting'' nature of an infinite-horizon AMDP is that, we have to generate $T$ perturbations $z^1,z^2,\ldots,z^T$, whereas only $H$ perturbations is needed in the episodic settings.
For each FTPL update, we will include all of them in the argmin operation, as in \Cref{eq:FTPL update infinite-horizon appendix}.

This difference can be explained from the contextual bandits' point of view (c.f. \Cref{sec:comparism with Syrgkanis16}). In infinite-horizon AMDPs, the possible number of ``contexts'' is now $T$, as for each policy $\pi$, it will have $T$ distinct features $\mu_\pi^1,\mu_\pi^2,\ldots,\mu_\pi^T$.
In contrast, for episodic AMDPs, there are only $H$ different contexts as only $\{\mu_\pi^h\}_{h=1}^H$ can appear. Therefore, as noticed by \citet{syrgkanis2016efficient}, we have to add perturbations to each of the contexts, which are in total $T$ of them.

\begin{algorithm}[!t]
\caption{FTPL for Infinite-horizon AMDPs with Bandit Feedback and Known Transition}
\label{alg:appendix infinite-horizon bandit feedback known transition}
\begin{algorithmic}[1]
\Require{Laplace distribution parameter $\eta$. Geometric Re-sampling parameter $L$.}
\State Sample perturbations $\{z^t\in \mathbb R^{\mathcal S\times \mathcal A}\}_{t\in [T]}$ where $z^t(s,a)\sim \text{Laplace}(\eta)$.
\For{$j=1,2,\ldots,J$}
\State Calculate the policy $\pi_j$ for this epoch as
\begin{equation}\label{eq:FTPL update infinite-horizon appendix}
    \pi_j=\argmin_{\pi \in \Pi}\left (\sum_{j'=1}^{j-1}\sum_{t\in \mathcal T_{j'}}\langle \mu_\pi^t,\hat \ell^t\rangle+\sum_{t=1}^T \langle \mu_\pi^t,z^t\rangle\right ).
\end{equation}
\State Execute \Cref{alg:policy switching} with parameters $s^t,\pi_j,t$ (note that \Cref{alg:policy switching} will update $t$ internally).
\For{All remaining time slots in $\mathcal T_j$, i.e., $\mathcal T_j\cap [t,T]$}
\State Play $a^t=\pi_j(s^t)$, observe the loss $\ell^t(s^t,a^t)$ and the next state $s^{t+1}\in \mS$.
\For{$M^t=1,2,\ldots,L$}
\State Resample a fresh perturbation and get new policy $\pi_j'$ from \Cref{eq:FTPL update infinite-horizon appendix}.
\State Draw a sample from $\text{Ber}(\mu_{\pi_j}^t(s^t,a^t;\mathbb P))$. If it is $1$ or $M^t=L$, terminate and set
\begin{equation*}
    \hat \ell^t(s,a)=\mathbbm 1[(s,a)=(s^t,a^t)]\ell^t(s^t,a^t)M^t,\quad \forall (s,a)\in \mS\times \mA.
\end{equation*}
\EndFor
\EndFor
\EndFor
\end{algorithmic}
\end{algorithm}

\subsubsection{Proof of Main Theorem}
\begin{proof}[Proof of \Cref{thm:regret of infinite-horizon AMDPs}]
To calculate the regret guarantee of \Cref{alg:appendix infinite-horizon bandit feedback known transition}, we consider the following quantity $\tilde{\mathcal R_T}$ defined as if there is no cost for a policy switching. By \Cref{thm:efficiency of switching}, there can be at most $JD^2$ time slots spent on executing \Cref{alg:policy switching}. Henceforth, the difference between $\mathcal R_T$ and $\tilde{\mathcal R_T}$ is at most $JD^2$.
\begin{equation}\label{eq:def of R_T tilde}
    \tilde{\mathcal R_T}\triangleq \E\left [\sum_{t=1}^T \langle \mu_{\pi_{j(t)}}^t,\ell^t\rangle-\langle \mu_{\pi^\ast}^t,\ell^t\rangle\right ]=\E\left [\sum_{j=1}^J \sum_{\pi \in \Pi} p_j(\pi) \sum_{t\in \mathcal T_j} \langle \mu_\pi^t,\ell^t\rangle-\sum_{t=1}^T \langle \mu_{\pi^\ast}^t,\ell^t\rangle\right ],
\end{equation}

where $p_j(\pi)$ is the probability of picking $\pi$ w.r.t. $z$, conditioning on $\mathcal F_{(j-1)H}$ and $j(t)$ is the epoch that $t$ belongs to, namely $j(t)=\lceil \frac jH\rceil$.
Then, we can decompose $\tilde{\mathcal R_T}$ into three terms exactly the same as what we did in \Cref{sec:appendix episodic known}:
\begin{align*}
    \tilde{\mathcal R_T}&=\underbrace{\E\left [\sum_{t=1}^T \langle \mu_{\pi_{j(t)}}^t,\ell^t-\hat \ell^t\rangle+\sum_{t=1}^T \langle \mu_{\pi^\ast}^t,\hat \ell^t-\ell^t\rangle\right ]}_{\text{GR error term}}+\\
    &\quad \underbrace{\E\left [\sum_{j=1}^J\sum_{\pi \in \Pi}p_{j+1}(\pi)\sum_{t\in \mathcal T_j} \langle \mu_{\pi}^t-\mu_{\pi^\ast}^t,\hat \ell^t\rangle\right]}_{\text{Error term}}+\\
    &\quad \underbrace{\E\left [\sum_{j=1}^J \sum_{\pi \in \Pi} (p_j(\pi)-p_{j+1}(\pi))\sum_{t\in \mathcal T_j}\langle \mu_\pi^t,\hat \ell^t\rangle\right]}_{\text{Stability term}}.
\end{align*}

The GR error term is quite similar to \Cref{sec:appendix episodic known}:
\begin{lemma}\label{lem:infinite GR error bound}
The GR error term is bounded by 
\begin{equation*}
    \E\left [\sum_{t=1}^T \langle \mu_{\pi_{j(t)}}^t,\ell^t-\hat \ell^t\rangle+\sum_{t=1}^T \langle \mu_{\pi^\ast}^t,\hat \ell^t-\ell^t\rangle\right ]\le \frac{SAT}{eL}.
\end{equation*}
\end{lemma}

For the error term, we still use the similar ``be-the-leader'' analysis as \Cref{lem:appendix known error term}, except for we are now facing a slightly different $V$-function (which is defined for infinite-horizon). Moreover, as mentioned in the main text, we are using a different bound when facing $T$ different perturbations. As a result, we will have worse dependency on $S$ and $A$, but with better dependency on the number of contexts, which is $T$ here (and is $H$ in episodic settings). The result is stated as follows:
\begin{lemma}\label{lem:infinite error bound}
The error term is bounded by
\begin{equation*}
\E\left [\sum_{j=1}^J\sum_{\pi \in \Pi}p_{j+1}(\pi)\sum_{t\in \mathcal T_j} \langle \mu_{\pi}^t-\mu_{\pi^\ast}^t,\hat \ell^t\rangle \right] \le \frac{10}{\eta}S\sqrt{AT\ln A}.
\end{equation*}
\end{lemma}

For the stability term, again much similar to \Cref{sec:appendix episodic known}, we have
\begin{lemma}\label{lem:infinite stability bound}
The stability term is bounded by
\begin{equation*}
\E\left [\sum_{j=1}^J \sum_{\pi \in \Pi} (p_{j+1}(\pi)-p_j(\pi))\sum_{t\in \mathcal T_j}\langle \mu_\pi^t,\hat \ell^t\rangle \right] \le 2\eta H^2SA J.
\end{equation*}
\end{lemma}

Therefore, our regret is bounded by
\begin{equation*}
    \mathcal R_T\le \tilde{\mathcal R_T}+JD^2\le \frac{SAT}{eL}+\frac{10}{\eta}S\sqrt{AT\ln A}+2\eta H^2SA J+JD^2.
\end{equation*}

Picking $\eta=S^{\nicefrac 13}D^{-\nicefrac 23}T^{-\nicefrac 13}$, $J=S^{\nicefrac 23}A^{\nicefrac 12}D^{-\nicefrac 43}T^{\nicefrac 56}$ and $L=S^{\nicefrac 13}A^{\nicefrac 12}D^{-\nicefrac 23}T^{\nicefrac 16}$ gives $\mathcal R_T=\Otil\left (S^{\nicefrac 23}A^{\nicefrac 12}D^{\nicefrac 23}T^{\nicefrac 56}\right )$.
\end{proof}

\begin{proof}[Proof of \Cref{lem:infinite GR error bound}]
We follow the proof of \Cref{lem:appendix known GR error term} by replacing $K$ with $J$ and the GR estimator $\hat \ell_k^h$ with $\hat \ell_t$. First notice that, from \Cref{lem:expectation of GR}, $\E[\hat \ell_t(s,a)\mid \mathcal F_{(j-1)H}]\le \ell^t(s,a)$ for all $t\in \mathcal T_j$. Moreover, as $\pi^\ast$ is deterministic (i.e., it does not depend on the randomness from the algorithm), the term related to $\mu_{\pi^\ast}^t$ is bounded by
\begin{equation*}
    \E\left [\sum_{j=1}^J \sum_{t\in \mathcal T_j} \langle \mu_{\pi^\ast}^t,\hat \ell^t-\ell^t\rangle\right ]=\E\left [\sum_{j=1}^J \sum_{t\in \mathcal T_j} \langle \mu_{\pi^\ast}^h,\E[\hat \ell^t\mid \mathcal F_{(j-1)H}]-\ell^t\rangle \right] \le 0.
\end{equation*}

For the first term, again by \Cref{lem:expectation of GR}, we have
\begin{equation*}
    \E\left [\sum_{j=1}^J \sum_{t\in \mathcal T_j} \langle \mu_{\pi_j}^t,\ell^t-\hat \ell^t\rangle\right ]=\sum_{j=1}^J \sum_{t\in \mathcal T_j} \sum_{(s,a)\in \mS\times \mA}\E\left [\mu_{\pi_j}^t(s,a)\cdot (1-q^t(s,a))^L \ell^t(s,a)\right ],
\end{equation*}

where $q^t(s,a)$ is the probability of visiting $(s,a)$ in a single execution of the Geometric Re-sampling process, which is just
\begin{equation*}
    q^t(s,a)= \E[\mu_{\pi_j}^t(s,a)]=\sum_{\pi \in \Pi}p_j(\pi)\mu_\pi^t(s,a)
\end{equation*}

in our case. By noticing that $q(1-q)^L\le qe^{-Lq}\le \frac{1}{eL}$ \citep{neu2013efficient}, we have
\begin{equation*}
    \E\left [\sum_{j=1}^J \sum_{t\in \mathcal T_j} \langle \mu_{\pi_j}^t,\ell^t-\hat \ell^t\rangle\right ]\le HJSA\frac{1}{eL}=\frac{SAT}{eL},
\end{equation*}

as claimed.
\end{proof}

\begin{proof}[Proof of \Cref{lem:infinite error bound}]
The proof still uses the standard ``be-the-leader'' technique, but in a slightly different manner as we are adding perturbations to all time indices. Instead, we follow the idea of \citet[Lemma 7]{syrgkanis2016efficient} and prove by induction that the following inequality holds for all $J$ and any policy $\pi\in \Pi$:
\begin{equation*}
    \sum_{t=1}^T \langle \mu_{\pi_1}^t,z^t\rangle+\sum_{j=1}^J \sum_{t\in\mathcal T_j}\langle \mu_{\pi_{j+1}}^t,\ell^t\rangle\le \sum_{t=1}^T \langle \mu_{\pi}^t,z^t\rangle+\sum_{j=1}^J \sum_{t\in\mathcal T_j}\langle \mu_{\pi}^t,\ell^t\rangle.
\end{equation*}

Obviously, for $J=0$, this inequality holds. Suppose that this inequality holds for $J$, then we consider $J+1$. Let $\pi=\pi_{J+2}$. Adding $\sum_{t\in \mathcal T_{J+1}}\langle \mu_{\pi_{J+2}}^t,\hat \ell^t\rangle$ to both sides gives
\begin{equation*}
    \sum_{t=1}^T \langle \mu_{\pi_1}^t,z^t\rangle+\sum_{j=1}^{J+1} \sum_{t\in\mathcal T_j}\langle \mu_{\pi_{j+1}}^t,\ell^t\rangle\le \sum_{t=1}^T \langle \mu_{\pi_{J+2}}^t,z^t\rangle+\sum_{j=1}^{J+1} \sum_{t\in\mathcal T_j}\langle \mu_{\pi_{J+2}}^t,\ell^t\rangle.
\end{equation*}

However, by definition of $\pi_{J+2}$ (which is the argmin of the right-handed-side for all policies), it is further bounded by
\begin{equation*}
    \sum_{t=1}^T \langle \mu_{\pi_1}^t,z^t\rangle+\sum_{j=1}^{J+1} \sum_{t\in\mathcal T_j}\langle \mu_{\pi_{j+1}}^t,\ell^t\rangle\le \sum_{t=1}^T \langle \mu_{\pi}^t,z^t\rangle+\sum_{j=1}^{J+1} \sum_{t\in\mathcal T_j}\langle \mu_{\pi}^t,\ell^t\rangle
\end{equation*}

for any policy $\pi\in\Pi$, which means that the induction hypothesis for $J+1$. Therefore, by picking $\pi=\pi^\ast$ for the real $J$, we can conclude that
\begin{equation*}
    \sum_{j=1}^{J} \sum_{t\in\mathcal T_j}\langle \mu_{\pi_{j+1}}^t,\ell^t\rangle-\sum_{j=1}^{J} \sum_{t\in\mathcal T_j}\langle \mu_{\pi^\ast}^t,\ell^t\rangle\le \sum_{t=1}^T \langle \mu_{\pi^\ast}^t,z^t\rangle-\sum_{t=1}^T \langle \mu_{\pi_1}^t,z^t\rangle.
\end{equation*}

Then taking expectation on both sides gives the error term is bounded by
\begin{equation*}
    \E\nolimits_{z}\left [\max_{\pi \in \Pi}\sum_{t=1}^T\langle \mu_\pi^t,z^t\rangle-\min_{\pi \in \Pi}\sum_{t=1}^T\langle \mu_\pi^t,z^t\rangle\right ],
\end{equation*}

which is bounded by $\frac{10}{\eta}\sqrt{TSA\cdot \ln \lvert \Pi\rvert}=\frac{10}{\eta}S\sqrt{AT\ln A}$ by \Cref{lem:error term syrg} (note that as $\ln \lvert \Pi\rvert=S\ln A<SAT$, the condition of applying \Cref{lem:error term syrg} indeed holds).
\end{proof}

\begin{proof}[Proof of \Cref{lem:infinite stability bound}]
This follows directly from \Cref{lem:appendix known stability term} with some slight modifications as well. For clarity, we rewrite the full proof here.

We first give the single-step stability lemma for infinite-horizon AMDPs, whose proof will be presented later:
\begin{lemma}\label{lem:infinite single-step stability}
For all $j\in [J]$ and $(s,a)\in \mS\times \mA$,
\begin{equation*}
    p_{j+1}(\pi)\ge p_j(\pi)\exp\left (-\eta \sum_{t\in \mathcal T_j} \lVert \hat \ell^t\rVert_1\right ),\quad \forall \pi \in \Pi.
\end{equation*}
\end{lemma}

By summing up \Cref{lem:infinite single-step stability} for all $\pi\in \Pi$ and using the fact that $1-\exp(-x)\le x$, we have
\begin{equation}\label{eq:infinite horizon stability RHS}
    \sum_{\pi\in \Pi}(p_j(\pi)-p_{j+1}(\pi))\sum_{t\in \mathcal T_j}\langle \mu_\pi^t,\hat \ell^t\rangle\le \eta \sum_{t'\in \mathcal T_j}\lVert \hat \ell^{t'}\rVert_1\cdot \sum_{\pi \in \Pi} p_j(\pi)\sum_{t\in \mathcal T_j}\langle \mu_\pi^t,\hat \ell^t\rangle,\quad \forall j\in [J].
\end{equation}

%Further notice that $\hat \ell^t$ is one-hot as it is yielded via Geometric Re-sampling, we can rewrite the right-handed-side of the previous inequality as
%\begin{equation}\label{eq:infinite horizon stability RHS}
%    \eta \sum_{t'\in \mathcal T_j} \hat \ell^{t'}(s^{t'},a^{t'})\cdot \sum_{\pi \in \Pi} p_j(\pi)\sum_{t\in \mathcal T_j} \mu_\pi^t(s^t,a^t),\hat \ell^t(s^t,a^t).
%\end{equation}

Again noticing that $M^t=\min\{\text{Geo}(q^t(s^t,a^t)),L\}$ where $q^t(s,a)=\sum_{\pi \in \Pi}p_j(\pi)\mu_\pi^t(s,a)$ if $t\in \mathcal T_j$. Then calculate the expectation of $\hat \ell^t(s,a)$ only with respect to $M^t$, we will have
\begin{equation*}
    \E\left [\hat \ell^t(s,a)\middle \vert (s^t,a^t)=(s,a)\right ]\le \frac{\ell^t(s,a)}{q^t(s,a)}.
\end{equation*}

Let $\mathbbm 1^t(s,a)$ be the shorthand notation of $\mathbbm 1[(s^t,a^t)=(s,a)]$. Then for those $t'\ne t$ in \Cref{eq:infinite horizon stability RHS},
\begin{align*}
    &\quad \eta \E\left [\sum_{t\in \mathcal T_j}\sum_{s,a} \sum_{\pi\in \Pi}p_j(\pi)\mu^t(s,a)\hat \ell^t(s,a) \sum_{t'\ne t}\| \hat \ell^{t'}\|_1 \middle \vert \mathcal F_{(j-1)H}\right ]\\
    &\overset{(a)}{\le} \eta \E\left [\sum_{h=1}^H \sum_{s,a} \mathbbm 1^t(s,a) \ell^t(s,a)\frac{\sum_{\pi \in \Pi}p_j(\pi)\mu_\pi^t(s,a)}{q^t(s,a)}\sum_{t'\ne t}\| \hat \ell^{t'}\|_1\middle \vert \mathcal F_{(j-1)H}\right ]\\
    &\overset{(b)}{\le} \eta H \E\left [\sum_{t'\in \mathcal T_j}\sum_{s,a}\| \hat \ell^{t'}\|_1\middle \vert \mathcal F_{(j-1)H}\right ]\overset{(c)}{\le} \eta H^2SA.
\end{align*}

where (a) is taking expectation w.r.t. $M^t$, (b) used the definition of $q^t$ together with the fact that $\sum_{(s,a)}\mathbbm 1^t(s,a)=1$, and (c) used the fact that $\E[\hat \ell^{t'}(s,a)\mid \mathcal F_{(j-1)H}]\le \ell^{t'}(s,a)\le 1$ (\Cref{lem:expectation of GR}).

For those terms with $t'=t$ in \Cref{eq:infinite horizon stability RHS}, by direct calculation and the fact that $\hat\ell^t$ is a one-hot vector, we can bound them as
\begin{equation*}
    \eta \E\left [\sum_{t\in \mathcal T_j}\sum_{s,a}\sum_{\pi \in \Pi}p_j(\pi)\mu_\pi^t(s,a)\left (\hat \ell^t(s,a)\right )^2\middle\vert \mathcal F_{(j-1)H}\right ]\le 2\eta \E\left [\sum_{h,s,a}\frac{q^t(s,a)}{q^t(s,a)}\middle \vert \mathcal F_{(j-1)H}\right ]\le 2\eta HSA
\end{equation*}

by noticing $\E[(\hat \ell^t(s,a))^2\mid \mathcal F_{(j-1)H}]\le 2(q^t(s,a))^{-1}$ (\Cref{lem:variance of GR}). Combining the terms with $t'\ne t$ and the ones with $t'=t$ gives our conclusion.
\end{proof}

\begin{proof}[Proof of \Cref{lem:infinite single-step stability}]
The proof will be similar to, but different from \Cref{lem:appendix known single-step stability}, as we are now adding different perturbations. We now use a slightly different definition of the $\text{best}$-function. Let $\pi=\text{best}(\ell,z)$ where $\ell=\{\ell^1,\ell^2,\ldots,\ell^m\}$ and $z=\{z^1,z^2,\ldots,z^T\}$ to denote
\begin{equation*}
    \pi=\argmin_{\pi \in \Pi}\left (\sum_{t=1}^m\langle \mu_\pi^t,\ell^t\rangle+ \sum_{t=1}^T \langle \mu_\pi^t,z^t\rangle\right ).
\end{equation*}

Then we have
\begin{align*}
    &\quad p_{j}(\pi)=\int_{z}\mathbbm 1\left [\pi=\text{best}\left (\{\hat \ell^1,\ldots,\hat \ell^{(j-1)H}\},\{z^1,z^2,\ldots,z^T\}\right )\right ]f\left (z\right )~\mathrm{d}z\\
    &=\int_z \mathbbm 1\left [\pi=\text{best}\left (\{\hat \ell^1,\ldots,\hat \ell^{(j-1)H}\},\{z^1,\ldots,z^{(j-1)H},z^{(j-1)H+1}+\hat \ell^{(j-1)H+1},\ldots,z^{jH}+\hat \ell^{jH},z^{jH+1},\ldots,z^T\}\right )\right ]\\
    &\qquad f\left (z+\{0,\ldots,0,\hat \ell^{(j-1)H+1},\ldots,\hat \ell^{jH},0,0,\ldots,0\}\right )~\mathrm{d}z\\
    &=\int_z \mathbbm 1\left [\pi=\text{best}\left (\{\hat \ell^1,\ldots,\hat \ell^{(j-1)H},\hat \ell^{(j-1)H+1},\ldots \hat \ell^{jH}\},z\right )\right ]\\
    &\qquad f\left (z+\{0,\ldots,0,\hat \ell^{(j-1)H+1},\ldots,\hat \ell^{jH},0,0,\ldots,0\}\right )~\mathrm{d}z,
\end{align*}

where $f(z)$ is the probability density function of $z$ and the second step makes use of the fact that $z+\hat \ell^j$ is still linear in $z$. Moreover, 
\begin{equation*}
    p_{j+1}(\pi)=\int_{z}\mathbbm 1\left [\pi=\text{best}\left (\{\hat \ell^1,\ldots,\hat \ell^{jH}\},z\right )\right ]f\left (z\right )~\mathrm{d}z.
\end{equation*}

For simplicity, denote $\tilde \ell^j=\{0,\ldots,0,\hat \ell^{(j-1)H+1},\ldots,\hat \ell^{jH},0,0,\ldots,0\}=\{\mathbbm 1[t\in \mathcal T_j]\hat \ell^t\}_{t=1}^T$. Again using the fact that $f(z)=\prod_{h=1}^H \exp(-\eta \lVert z^h\rVert_1)$, we have
\begin{equation*}
    f\left (z+\tilde \ell^j\right )=\prod_{t\in \mathcal T_j}\exp\left (-\eta \left (\lVert z^t+\hat \ell^{t}\rVert_1-\lVert z^t\rVert\right )\right )f(z),
\end{equation*}

which gives
\begin{equation*}
    \frac{f\left (z+\tilde \ell^j\right )}{f(z)}\in \left [\exp\left (-\eta \sum_{t\in \mathcal T_j}
    \lVert \hat \ell^t\rVert_1\right ),\exp\left (\eta \sum_{t\in \mathcal T_j} 
    \lVert \hat \ell^t\rVert_1\right )\right ]
\end{equation*}

by triangle inequality. Therefore, $\nicefrac{p_{j+1}(\pi)}{p_j(\pi)}$ lies in this interval as well, which is just our claim.
\end{proof}

%\subsubsection{Proof of \Cref{thm:regret of episodic AMDPs known}}
%\begin{proof}[Proof of \Cref{thm:regret of episodic AMDPs known}]
%From \Cref{lem:appendix known GR error term,lem:appendix known error term,lem:appendix known stability term}, we have
%\begin{equation*}
%    \mathcal R_K\le \frac{SAHK}{eL}+2\frac{H}{\eta}(1+\ln (SA))+3\eta H^2SAK.
%\end{equation*}
%
%Therefore, if we pick $\eta^{-1}=\sqrt{HSAK}$ and $L=\sqrt{SAK/H}$, we will have
%\begin{equation*}
%    \mathcal R_K\le \frac{H^{\nicefrac 32}\sqrt{SAK}}{e}+2H^{\nicefrac 32}\sqrt{SAK}(1+\ln(SA))+3H^{\nicefrac 32}\sqrt{SAK}=\Otil\left (H^{\nicefrac 32}\sqrt{SAK}\right ),
%\end{equation*}
%
%as desired.
%\end{proof}

\subsection{Hedge-Based Inefficient Algorithm (\Cref{thm:regret of Hedge})}\label{sec:appendix infinite-horizon Hedge}
%!TEX root=main.tex
In this section, we present our Hedge-based inefficient algorithm for infinite-horizon AMDPs with bandit feedback and known transitions. We still use the same epoching mechanism as \Cref{alg:appendix infinite-horizon bandit feedback known transition}.

For Hedge, which is different from FTPL, we will \textit{explicitly} maintain a distribution $p_j\in \triangle(\Pi)$ over all policies for each epoch, and randomly draw one $\pi_j\sim p_j$ for the $j$-th epoch. As the distribution $p_j$ can be directly calculated (we do not care about computational efficiency now), we can use importance weighting estimator to estimate the losses. The algorithm is presented in \Cref{alg:appendix hedge}.

\begin{algorithm}[!t]
\caption{Hedge for Infinite-horizon AMDPs with Bandit Feedback and Known Transition}
\label{alg:appendix hedge}
\begin{algorithmic}[1]
\Require Learning rate $\eta$. Number of epochs $J$.
\For{$j=1,2,\ldots,J$}
\State Calculate the distribution of policies for the $j$-th epoch as
\begin{equation}\label{eq:update rul of Hedge}
    p_j(\pi)\propto \exp\left (-\eta \sum_{t=1}^{(j-1)H}\langle \mu_\pi^t,\hat \ell^t\rangle \right ).
\end{equation}
\State Sample the policy $\pi_j\sim p_j$ for this epoch.
\State Execute \Cref{alg:policy switching} with parameters $s^t,\pi_j,t$ (note that \Cref{alg:policy switching} will update $t$ internally).
\For{All remaining $\mathcal T_j\cap [t,T]$}
\State Play $a^t=\pi_j(s^t)$, observe loss $\ell^t(s^t,a^t)$ and the next state $s^{t+1}$. Set
\begin{equation*}
    \hat \ell^t(s,a)=\mathbbm 1[(s,a)=(s^t,a^t)]\frac{\ell^t(s^t,a^t)}{\sum_{\pi \in \Pi}p_j(\pi)\mu_\pi^t(s,a)},\quad \forall (s,a)\in \mS\times \mA.
\end{equation*}
\EndFor
\EndFor
\end{algorithmic}
\end{algorithm}

\begin{proof}[Proof of \Cref{thm:regret of Hedge}]
As \Cref{sec:appendix infinite-horizon FTPL}, we still define $\tilde{\mathcal R_T}$ as \Cref{eq:def of R_T tilde}. We can still conclude that $\mathcal R_T\le \tilde{\mathcal R_T}+JD^2$. We first show that the importance weighting estimator is indeed unbiased. Notice that the probability of visiting $(s,a)$ at some slot $t\in \mathcal T_j$ is exactly $\sum_{\pi\in \Pi}p_j(\pi)\mu_\pi^t(s,a)$, which means, by \Cref{lem:mean of IW}, we have
\begin{equation*}
    \E\left [\hat \ell^{t}(s,a)\middle \vert \mathcal F_{(j-1)H}\right ]=\ell^t(s,a),\quad \forall (s,a)\in \mS\times \mA,t\in \mathcal T_j,j\in [J].
\end{equation*}

Let $\tilde \ell_j(\pi)$ be the random variable denoting the total loss of policy $\pi$ for epoch $j$:
\begin{equation*}
    \tilde \ell_j(\pi)=\sum_{t\in \mathcal T_j}\langle \mu_\pi^t, \hat \ell^t\rangle.
\end{equation*}

So \Cref{eq:update rul of Hedge} is just $p_j(\pi)\propto \exp(-\eta \sum_{j'=1}^{j-1}\tilde \ell_{j'}(\pi))$. Therefore, by standard properties of Hedge (\Cref{lem:property about Hedge}), for any realization of $\{\tilde \ell_j\}_{j\in [J]}$ (and also $\{\hat \ell^t\}_{t\in [T]}$), we will have
\begin{equation}
    \sum_{j=1}^J \langle p_j,\tilde \ell_j\rangle-\sum_{j=1}^J \tilde \ell_j(\pi^\ast)\le \frac{\ln \lvert \Pi\rvert}{\eta}+\eta \sum_{j=1}^J \sum_{\pi \in \Pi}p_j(\pi)\tilde \ell_j^2(\pi).\label{eq:Hedge non-randomized regret}
\end{equation}

Consider the second term of the right-handed-side. For a fixed $j\in [J]$, it becomes
\begin{align*}
    \sum_{\pi \in \Pi}p_j(\pi)\tilde \ell_j^2(\pi)&=\sum_{\pi \in \Pi}p_j(\pi)\left (\sum_{t\in \mathcal T_j}\sum_{(s,a)\in \mS\times \mA}\mu_\pi^t(s,a)\hat \ell^t(s,a)\right )^2\\
    &\le H\sum_{\pi \in \Pi}p_j(\pi)\sum_{t\in \mathcal T_j}\left (\sum_{(s,a)\in \mS\times \mA}\mu_\pi^t(s,a)\hat \ell^t(s,a)\right )^2\\
    &= H\sum_{\pi \in \Pi}p_j(\pi)\sum_{t\in \mathcal T_j}\sum_{(s,a)\in \mS\times \mA}\left (\mu_\pi^t(s,a)\hat \ell^t(s,a)\right )^2,
\end{align*}

where the first inequality made use of Cauchy-Schwartz inequality while the second equality used the fact that $\hat \ell_t$ is one-hot. Plugging back into \Cref{eq:Hedge non-randomized regret} and taking expectation on both sides,
\begin{align*}
    &\quad \E\left [\sum_{j=1}^J \langle p_j,\tilde \ell_j\rangle-\sum_{j=1}^J \tilde \ell_j(\pi^\ast)\right ]\\
    &\overset{(a)}{\le} \frac{S\ln A}{\eta}+H\sum_{j=1}^J\E\left [\sum_{\pi \in \Pi}p_j(\pi)\sum_{t\in \mathcal T_j}\sum_{(s,a)\in \mS\times \mA}\mu_\pi^t(s,a)\left (\hat \ell^t(s,a)\right )^2\middle \vert \mathcal F_{(j-1)H}\right ]\\
    &\overset{(b)}{\le}\frac{S\ln A}{\eta}+\eta H\sum_{j=1}^J\E\left [\sum_{\pi \in \Pi}p_j(\pi)\sum_{t\in \mathcal T_j}\sum_{(s,a)\in \mS\times \mA}\mu_\pi^t(s,a)\cdot \frac{1}{\sum_{\pi \in \Pi}p_j(\pi)\mu_\pi^t(s,a)}\middle \vert \mathcal F_{(j-1)H}\right ]\\
    &=\frac{S\ln A}{\eta}+\eta H^2JSA=\frac{S\ln A}{\eta}+\eta HSAT,
\end{align*}

where (a) used $\mu_\pi^t(s,a)\le 1$ and (b) used \Cref{lem:variance of IW}. Moreover, for the left-hand side, we have
\begin{equation*}
    \E\left [\sum_{j=1}^J \langle p_j,\tilde \ell_j\rangle-\sum_{j=1}^J \tilde \ell_j(\pi^\ast)\right ]=\sum_{j=1}^J \E\left [\sum_{\pi\in \Pi}p_j(\pi)\sum_{t\in \mathcal T_j}\langle \mu_\pi^t-\mu_{\pi^\ast}^t,\hat \ell^t\rangle\middle \vert \mathcal F_{(j-1)H}\right ].
\end{equation*}

By using \Cref{lem:mean of IW}, this is exactly
\begin{equation*}
    \sum_{j=1}^J \E\left [\sum_{\pi\in \Pi}p_j(\pi)\sum_{t\in \mathcal T_j}\langle \mu_\pi^t-\mu_{\pi^\ast}^t,\hat \ell^t\rangle\middle \vert \mathcal F_{(j-1)H}\right ]=\E\left [\sum_{j=1}^J\sum_{t\in \mathcal T_j}\langle \mu_{\pi_j}^t-\mu_{\pi^\ast}^t,\ell^t\rangle\right ]=\tilde{\mathcal R_T}.
\end{equation*}

Therefore, we will have
\begin{equation*}
    \mathcal R_T\le \tilde{\mathcal R_T}+JD^2\le \frac{S\ln A}{\eta}+\eta HSAT+JD^2,
\end{equation*}

which gives $\mathcal R_T=\Otil\left (S^{\nicefrac 23}A^{\nicefrac 13}D^{\nicefrac 23}T^{\nicefrac 23}\right )$ when picking $J=S^{\nicefrac 23}A^{\nicefrac 13}D^{-\nicefrac 43}T^{\nicefrac 23}$ and $\eta=S^{\nicefrac 13}A^{-\nicefrac 13}D^{-\nicefrac 23}T^{-\nicefrac 23}$.
\end{proof}

\begin{lemma}[Property of Hedge]\label{lem:property about Hedge}
Suppose that we are using Hedge for $T$-round online learning problem that has $K$ actions, i.e., at time slot $t\in [T]$, picking $i_t$ according to the probability distribution $p_t\in \triangle([K])$ which is defined as:
\begin{equation*}
    p_t(i)\propto \exp\left (-\eta \sum_{\tau=1}^{t-1}\ell_\tau(i)\right ),\quad \forall i\in [K],t\in [T],
\end{equation*}

where $\ell_t(i)\ge 0$ is the non-negative loss associated with action $i$ at time slot $t$. Then, for all $i^\ast\in [K]$, we have
\begin{equation*}
    \sum_{t=1}^T (\langle p_t,\ell_t\rangle-\ell_t(i^\ast))\le \frac{\ln K}{\eta}+\eta \sum_{t=1}^T \sum_{i=1}^K p_t(i)\ell_t^2(i).
\end{equation*}

Note that here we are considering non-randomized loss functions here.
\end{lemma}
\begin{proof}
For simplicity, define $L_t(i)$ as $\sum_{\tau=1}^t \ell_\tau(i)$. Let
\begin{equation*}
    \Phi_t=\frac 1\eta \ln \left (\sum_{i=1}^K \exp\left (-\eta L_t(i)\right )\right ),
\end{equation*}

then
\begin{align*}
    \Phi_t-\Phi_{t-1}&=\frac 1\eta \ln \left (\frac{\sum_{i=1}^K \exp(-\eta L_t(i))}{\sum_{i=1}^K \exp(-\eta L_{t-1}(i))}\right )\\
    &=\frac 1\eta \ln \left (\sum_{i=1}^Kp_t(i) \exp(-\eta \ell_t(i))\right )\\
    &\overset{(a)}{\le} \frac 1\eta \ln \left (\sum_{i=1}^K p_t(i) (1-\eta \ell_t(i)+\eta^2 \ell_t^2(i))\right )\\
    &=\frac 1\eta \ln \left (1-\eta \langle p_t,\ell_t\rangle+\eta^2 \sum_{i=1}^K p_t(i)\ell_t^2(i)\right )\\
    &\overset{(b)}{\le}-\langle p_t,\ell_t\rangle+\eta \sum_{i=1}^K p_t(i)\ell_t^2(i),
\end{align*}

where (a) used $\exp(-x)\le 1-x+x^2$ for all $x\ge 0$ and (b) used $\ln(1+x)\le x$. Therefore, summing over $t$ gives
\begin{align*}
    \sum_{t=1}^T \langle p_t,\ell_t\rangle&\le \Phi_0-\Phi_T+\eta \sum_{t=1}^T \sum_{i=1}^N p_t(i)\ell_t^2(i)\\
    &\le \frac{\ln N}{\eta}-\frac 1\eta \ln \left (\exp(-\eta L_T(i^\ast))\right )+\eta \sum_{t=1}^T \sum_{i=1}^N p_t(i)\ell_t^2(i)\\
    &\le \frac{\ln N}{\eta}+L_T(i^\ast)+\eta \sum_{t=1}^T \sum_{i=1}^N p_t(i)\ell_t^2(i).
\end{align*}

Moving $L_T(i^\ast)$ to the left-handed-side then gives our conclusion.
\end{proof}

\section{Auxiliary Lemmas}
%!TEX root=main.tex

\subsection{Geometric Re-sampling Properties}
In this section, we list two properties of the Geometric Re-sampling estimator \citep{neu2013efficient} that we used in the analysis. For the sake of completeness, we also include their proofs here.
\begin{lemma}[{\citet[Lemma 1]{neu2013efficient}}]\label{lem:expectation of GR}
Consider the Geometric Re-sampling estimator
\begin{equation}\label{eq:GR def appendix}
    \hat \ell_k^h(s,a)=\mathbbm 1[(s_k^h,a_k^h)=(s,a)]M_k^h(s,a)\ell_k^h(s,a).
\end{equation}

Let $\Pr\{(s_k^h,a_k^h)=(s,a)\mid \mathcal F_{k-1}\}=q_k^h(s,a)$. Suppose that the probability of visiting $(s,a)$ in the re-sampling process is also $q_k^h(s,a)$, then we have
\begin{equation*}
    \E\left [\hat \ell_k^h(s,a)\middle \vert \mathcal F_{k-1}\right ]=\left (1-(1-q_k^h(s,a))^L\right )\ell_k^h(s,a).
\end{equation*}
\end{lemma}
\begin{proof}
By direct calculation, we have
\begin{align*}
    &\quad \E\left [M_k^h(s,a)\middle \vert \mathcal F_{k-1},(s_k^h,a_k^h)=(s,a)\right ]=\sum_{n=1}^\infty n(1-q)^{n-1}q-\sum_{n=L}^\infty (n-L)(1-q)^{n-1}q\\
    &=\left (1-(1-q)^L\right )\sum_{n=1}^\infty n(1-q)^{n-1}q=\frac{1-(1-q)^L}{q}.
\end{align*}

So we have
\begin{align*}
    &\quad \E\left [\hat \ell_k^h(s,a)\middle \vert \mathcal F_{k-1}\right ]=\Pr\{(s_k^h,a_k^h)=(s,a)\mid \mathcal F_{k-1}\}\ell_k^h(s,a)\E\left [M_k^h(s,a)\middle \vert \mathcal F_{k-1},(s_k^h,a_k^h)=(s,a)\right ]\\&=\left (1-(1-q_k^h(s,a))^L\right )\ell_k^h(s,a),
\end{align*}

as desired.
\end{proof}

\begin{lemma}\label{lem:variance of GR}
For the Geometric Re-sampling estimator as defined in the previous lemma, we have
\begin{equation*}
    \E\left [(\hat \ell_k^h(s,a))^2\middle \vert \mathcal F_{k-1}\right ]\le \frac{2(\ell_k^h(s,a))^2}{q_k^h(s,a)}.
\end{equation*}
\end{lemma}
\begin{proof}
By definition, write
\begin{align}
    &\quad \E\left [(\hat \ell_k^h(s,a))^2\middle \vert \mathcal F_{k-1}\right ]=\E[\mathbbm 1[(s_k^h,a_k^h)=(s,a)]^2(\ell_k^h(s,a))^2(M_k^h(s,a))^2]\nonumber\\
    &=\Pr\{(s_k^h,a_k^h)=(s,a)\}(\ell_k^h(s,a))^2\E\left [(M_k^h(s,a))^2\middle \vert \mathcal F_{k-1},(s_k^h,a_k^h)=(s,a)\right ].\label{eq:GR variance decompose}
\end{align}

Simply write $q_k^h(s,a)$ as $q$. Note that $M_k^h(s,a)=\min\{L,\text{Geo}(q)\}$, it is stochastically dominated by the geometric distribution with parameter $q$, whose second moment is bounded by
\begin{align}
    &\quad \E\left [(M_k^h(s,a))^2\middle \vert \mathcal F_{k-1},(s_k^h,a_k^h)=(s,a)\right ]\nonumber\\
    &\le \E[(\text{Geo}(q))^2]=\text{Var}(\text{Geo}(q))+(\E[\text{Geo}(q)])^2=\frac{1-q}{q^2}+\frac{1}{q}\le \frac{2}{q^2},\label{eq:Geo variance}
\end{align}

which means
\begin{equation*}
    \E\left [(\hat \ell_k^h(s,a))^2\middle \vert \mathcal F_{k-1}\right ]\le q(\ell_k^h(s,a))^2\frac{2}{q^2}=\frac{2(\ell_k^h(s,a))^2}{q},
\end{equation*}

as claimed.
\end{proof}

\begin{corollary}\label{lem:expectation of GR varying}
Still consider the GR estimator defined in \Cref{eq:GR def appendix}. Suppose that $\Pr\{(s_k^h,a_k^h)=(s,a)\mid \mathcal F_{k-1}\}=\hat q_k^h(s,a)$ and the probability of visiting $(s,a)$ in each re-sampling procedure is $q_k^h(s,a)$ (where $\hat q_k^h(s,a)\ne q_k^h(s,a)$). We then have
\begin{equation*}
    \E\left [\hat \ell_k^h(s,a)\middle \vert \mathcal F_{k-1}\right ]=\frac{\hat q_k^h(s,a)}{q_k^h(s,a)}\left (1-(1-q_k^h(s,a))^L\right )\ell_k^h(s,a).
\end{equation*}
\end{corollary}
\begin{proof}
The calculation of $\E[M_k^h(s,a)\mid \mathcal F_{k-1},(s_k^h,a_k^h)=(s,a)]$ is the same as the one in \Cref{lem:expectation of GR}. Therefore,
\begin{align*}
    &\quad \E\left [\hat \ell_k^h(s,a)\middle \vert \mathcal F_{k-1}\right ]=\Pr\{(s_k^h,a_k^h)=(s,a)\mid \mathcal F_{k-1}\}\ell_k^h(s,a)\E\left [M_k^h(s,a)\middle \vert \mathcal F_{k-1},(s_k^h,a_k^h)=(s,a)\right ]\\&=\frac{\hat q_k^h(s,a)}{q_k^h(s,a)}\left (1-(1-q_k^h(s,a))^L\right )\ell_k^h(s,a),
\end{align*}

as claimed.
\end{proof}

\begin{corollary}\label{lem:variance of GR varying}
Suppose the same condition as the previous corollary, i.e., still considering the GR estimator defined in \Cref{eq:GR def appendix} where $\Pr\{(s_k^h,a_k^h)=(s,a)\mid \mathcal F_{k-1}\}=\hat q_k^h(s,a)$ and the probability of visiting $(s,a)$ in each re-sampling procedure is $q_k^h(s,a)$. We have
\begin{equation*}
    \E\left [(\hat \ell_k^h(s,a))^2\middle \vert \mathcal F_{k-1}\right ]\le \frac{2(\ell_k^h(s,a))^2}{q_k^h(s,a)}\frac{\hat q_k^h(s,a)}{q_k^h(s,a)}.
\end{equation*}
\end{corollary}
\begin{proof}
Still decompose the variance as \Cref{eq:GR variance decompose}. Still write $\hat q_k^h(s,a)$ as $\hat q$ and $q_k^h(s,a)$ as $q$. Then we still have $M_k^h(s,a)=\min\{L,\text{Geo}(q)\}$, which gives $\E[(M_k^h(s,a))^2\mid \mathcal F_{k-1},(s_k^h,a_k^h)=(s,a)]\le \frac{2}{q^2}$ by \Cref{eq:Geo variance}. Therefore,
\begin{equation*}
    \E\left [(\hat \ell_k^h(s,a))^2\middle \vert \mathcal F_{k-1}\right ]\le \hat q(\ell_k^h(s,a))^2\frac{2}{q^2}=\frac{2(\ell_k^h(s,a))^2}{q_k^h(s,a)}\frac{\hat q_k^h(s,a)}{q_k^h(s,a)},
\end{equation*}

as claimed.
\end{proof}

\subsection{Importance Weighting Properties}
\begin{lemma}\label{lem:mean of IW}
For the Importance Weighting estimator
\begin{equation*}
    \hat \ell^t(s,a)=\mathbbm 1[(s^t,a^t)=(s,a)]\frac{\ell^t(s^t,a^t)}{\Pr\{(s^t,a^t)=(s,a)\mid \mathcal F\}},\quad \forall (s,a)\in \mS\times \mA,
\end{equation*}

where $\mathcal F$ is a filtration, we will have
\begin{equation*}
    \E[\hat \ell^t(s,a)\mid \mathcal F]=\ell^t(s,a),\quad \forall (s,a)\in \mS\times \mA.
\end{equation*}
\end{lemma}
\begin{proof}
For simplicity, denote $q^t(s,a)=\Pr\{(s^t,a^t)=(s,a)\mid \mathcal F\}$. Then
\begin{equation*}
    \E[\hat \ell^t(s,a)\mid \mathcal F]=q^t(s,a)\cdot \frac{\ell^t(s,a)}{q^t(s,a)}=\ell^t(s,a)
\end{equation*}

for all $(s,a)\in \mS\times \mA$.
\end{proof}

\begin{lemma}\label{lem:variance of IW}
For the same Importance Weighting Estimator, we will have
\begin{equation*}
    \E\left [(\hat \ell^t(s,a))^2\middle \vert \mathcal F\right ]=\frac{(\ell^t(s,a))^2}{q^t(s,a)},\quad \forall (s,a)\in \mS\times \mA,
\end{equation*}

where $q^t(s,a)\triangleq \Pr\{(s^t,a^t)=(s,a)\mid \mathcal F\}$.
\end{lemma}
\begin{proof}
Direct calculation gives $\E\left [(\hat \ell^t(s,a))^2\middle \vert \mathcal F\right ]=q^t(s,a)\cdot \left (\frac{\ell^t(s,a)}{q^t(s,a)}\right )^2=\frac{(\ell^t(s,a))^2}{q^t(s,a)}$, $\forall (s,a)$.
\end{proof}

\subsection{Auxiliary Lemmas for Error Terms}
In this section, we present two lemmas that will play an important role when bounding the error terms (as used in \Cref{lem:appendix known error term,lem:unknown error bound,lem:infinite error bound}).
\begin{lemma}[{\citet[Fact 2]{wang2020refined}}]\label{lem:error term wang}
Let $X_1,X_2,\ldots,X_n$ be i.i.d. random variables drawn from $\text{Exp}(\eta)$ which is the exponential distribution, then
\begin{equation*}
    \E\left [\max_{1\le i\le n}X_i\right ]\le \frac{1+\ln n}{\eta}.
\end{equation*}
\end{lemma}
\begin{lemma}[Generalization of {\citet[Lemma 8]{syrgkanis2016efficient}}]\label{lem:error term syrg}
Let $\{z^t\in \mathbb R^d\}_{t=1}^T$ be a sequence of $d$-dimensional random variable such that $z_i^t\sim \text{Laplace}(\eta)$ for all $i\in[m]$ and $t\in [T]$. Let $X$ be a set  of sequences of the form $\{x^t\in [0,1]^d\}_{t=1}^T$. As long as $\ln \lvert X\rvert<dT$, we have
\begin{equation*}
    \E\nolimits_{z}\left [\max_{x\in X}\sum_{t=1}^T \langle x^t,z^t\rangle\right ]-\E\nolimits_{z}\left [\min_{x\in X}\sum_{t=1}^T \langle x^t,z^t\rangle\right ]\le \frac{10}{\eta}\sqrt{dT\ln \lvert X\rvert}.
\end{equation*}
\end{lemma}
\begin{proof}
Note that the key difference between this theorem and \citet[Lemma 8]{syrgkanis2016efficient} is that, their theorem assumed a binary decision set, i.e., $x_i^t\in \{0,1\}$ instead of $[0,1]$. However, their proof still holds with only a little modification. The first step is still noticing that the distribution of Laplace random variables is symmetric around $0$, so we only need to bound $2\E\nolimits_{z}\left [\max_{x\in X}\sum_{t=1}^T \langle x^t,z^t\rangle\right ]$, which is bounded by, for any $\lambda\ge 0$,
\begin{align*}
    \E\nolimits_{z}\left [\max_{x\in X}\sum_{t=1}^T \langle x^t,z^t\rangle\right ]&=\frac 1\lambda \ln \left (\exp \left (\E\nolimits_{z}\left [\max_{x\in X}\lambda\sum_{t=1}^T \langle x^t,z^t\rangle\right ]\right )\right )\\
    &\le \frac 1\lambda \ln \left (\E\nolimits_{z}\left [\max_{x\in X}\exp \left (\lambda\sum_{t=1}^T \langle x^t,z^t\rangle\right )\right ]\right )\\
    &\le \frac 1\lambda \ln \left (\sum_{x\in X}\E\nolimits_{z}\left [\exp \left (\lambda\sum_{t=1}^T \langle x^t,z^t\rangle\right )\right ]\right )\\
    &\le \frac 1\lambda \ln \left (\sum_{x\in X}\prod_{t=1}^T\E\nolimits_{z}\left [\exp \left (\lambda \langle x^t,z^t\rangle\right )\right ]\right )\\
    &= \frac 1\lambda \ln \left (\sum_{x\in X}\prod_{t=1}^T\E\nolimits_{z}\left [\exp \left (\lambda \sum_{i=1}^d x_i^tz_i^t\right )\right ]\right )\\
    &\leq \frac 1\lambda \ln \left (\sum_{x\in X}\prod_{t=1}^T\prod_{i=1}^d \left (\E\nolimits_{z}\left [\exp \left (\lambda z_i^t\right )\right ]\right )^{x_i^t}\right ),
\end{align*}

where the last step used the fact that $x_i^t\le 1$ (and thus $y^{x_i^t}$ is a concave function in $y$). Furthermore, by using the fact that $\E\nolimits_{z}\left [\exp \left (\lambda z_i^t\right )\right ]$ is just the moment generating function of Laplace random variables evaluated at $\lambda$, it is just $(1-\frac{\lambda^2}{\eta^2})^{-1}$ as long as $\lambda<\eta$. As it is always larger than $1$, we can directly bound
\begin{align*}
    \E\nolimits_{z}\left [\max_{x\in X}\sum_{t=1}^T \langle x^t,z^t\rangle\right ]
    &\le \frac 1\lambda \ln \left (\sum_{x\in X}\prod_{t=1}^T\sum_{i=1}^d \left (\E\nolimits_{z}\left [\exp \left (\lambda z_i^t\right )\right ]\right )^{x_i^t}\right )\\
    &\le \frac 1\lambda \ln \left (\sum_{x\in X}\prod_{t=1}^T\prod_{i=1}^d \left (\frac{1}{1-\frac{\lambda^2}{\eta^2}}\right )^{x_i^t}\right )
    \le \frac 1\lambda \ln \left (\sum_{x\in X}\prod_{t=1}^T\prod_{i=1}^d \frac{1}{1-\frac{\lambda^2}{\eta^2}}\right )\\
    &=\frac 1\lambda \ln \left (\lvert X\rvert \left (\frac{1}{1-\frac{\lambda^2}{\eta^2}}\right )^{dT}\right )=\frac 1\lambda \ln \lvert X\rvert +\frac{dT}{\lambda} \ln \left (\frac{1}{1-\frac{\lambda^2}{\eta^2}}\right ).
\end{align*}

By using the fact that $\frac{1}{1-x}\le \exp(2x)$ for all $x\le \frac 14$, as long as $\lambda\le \frac \eta 2$, we will have
\begin{equation*}
    \E\nolimits_{z}\left [\max_{x\in X}\sum_{t=1}^T \langle x^t,z^t\rangle\right ]
    \le \frac 1\lambda \ln \left (\lvert X\rvert \left (\frac{1}{1-\frac{\lambda^2}{\eta^2}}\right )^{dT}\right )=\frac 1\lambda \ln \lvert X\rvert +\frac{2dT}{\lambda} \frac{\lambda^2}{\eta^2}.
\end{equation*}

By picking $\lambda=\frac{\eta\sqrt{\ln \lvert X\rvert}}{2\sqrt{dT}}<\frac \eta 2$ (according to the assumption that $\ln \lvert X\rvert<dT$) gives the bound $\frac{5}{\eta}\sqrt{dT\ln \lvert X\rvert}$, which is what we want.
\end{proof}

\end{document}